\documentclass[3p]{elsarticle}

\usepackage{xspace}
\usepackage{amssymb}
\usepackage{amsmath}
\usepackage{amsthm}
\usepackage{multirow}

\usepackage{tikz}
\usetikzlibrary{calc}

\usepackage{hyperref}

\renewcommand{\epsilon}{\varepsilon}
\renewcommand{\phi}{\varphi}

\renewcommand{\vec}[1]{\mathbf{#1}}

\newcommand{\bigO}{\mathcal{O}}
\newcommand{\naturals}{\mathbb{N}}

\newcommand{\reals}{\mathbb{R}}
\newcommand{\posreals}{\reals_{\geq 0}}

\newcommand{\abs}[1]{\lvert #1 \rvert}
\newcommand{\norm}[2][]{\lVert #2 \rVert_{#1}}

\newcommand{\functionDot}{\,\cdot\,}

\newcommand{\upperboundFXBeta}{\mathit{UB}}

\newcommand{\interval}[2]{[#1,#2]}

\newcommand{\size}[1]{|#1|}

\newcommand{\setnocond}[1]{\{#1\}}
\newcommand{\setcond}[2]{\{\, #1 \,|\, #2 \,\}}

\newcommand{\mystackrelsingle}[2]{%
  \mathrel{\vbox{\offinterlineskip\ialign{%
    \hfil##\hfil\cr%
    $\scriptstyle#1$\cr%
    \noalign{\kern.3ex}%
    $#2$\cr%
}}}}

\newcommand{\AP}{\mathit{AP}}

\newcommand{\modelsymbol}[1]{\mathcal{#1}}
\newcommand{\mc}[1][D]{\modelsymbol{#1}}
\newcommand{\pmc}[1][D]{\mc[#1]_{\parametersSet}}

\newcommand{\mdp}[1][M]{\modelsymbol{#1}}
\newcommand{\pmdp}[1][M]{\mdp[#1]_{\parametersSet}}
\newcommand{\mdrm}[1][R]{\modelsymbol{#1}}
\newcommand{\pmdrm}[1][R]{\mdp[#1]_{\parametersSet}}

\newcommand{\applypolicy}[2][\policy]{#2^{#1}}

\newcommand{\pd}{\mu}
\newcommand{\Dist}[1]{\mathit{Dist}(#1)}
\newcommand{\Supp}[1]{\mathit{Supp}(#1)}

\newcommand{\mstates}{S}
\newcommand{\minit}[1][s]{\bar{#1}}
\newcommand{\mactions}{\mathit{Act}}
\newcommand{\mtransitions}{\mathbf{P}}
\newcommand{\mlabelling}{L}
\newcommand{\mreward}[1][r]{\mathfrak{#1}}

\newcommand{\mpath}{\sigma}
\newcommand{\mpaths}[1]{\mathit{Paths}(#1)}
\newcommand{\mpathsfin}[1]{\mathit{Paths}^{*}(#1)}
\newcommand{\last}[1]{\mathit{last}(#1)}
\newcommand{\mpathstate}[2][\mpath]{#1[#2]}
\newcommand{\mpathaction}[2][\mpath]{#1(#2)}

\newcommand{\policy}{\pi}
\newcommand{\policies}{\mathrm{Pol}}

\newcommand{\cylinder}[1]{\mathit{Cyl}(#1)}
\newcommand{\prob}{\mathit{Pr}}
\newcommand{\exprew}{\mathit{ExpRew}}
\DeclareMathOperator{\opt}{opt}

\newcommand{\parameter}[1][v]{#1}
\newcommand{\parameters}[1][\parameter]{\vec{#1}}
\newcommand{\parametersSet}{\mathrm{V}}
\newcommand{\parametersDomain}{\mathcal{V}}

\newcommand{\ratfunsSet}[1][\parametersSet]{\mathcal{F}_{#1}}
\newcommand{\polynomialsSet}[1][\parametersSet]{\mathcal{P}_{#1}}
\DeclareMathOperator{\range}{range}
\newcommand{\evaluation}{\nu}
\newcommand{\evaluate}[2][\evaluation]{#2\langle#1\rangle}

\newcommand{\ding}[1]{\fontfamily{pzd}\fontencoding{U}\fontseries{m}\fontshape{n}\selectfont\char#1}
\newcommand{\success}{\text{\ding{51}}}
\newcommand{\failure}{\text{\ding{55}}}

\newcommand{\lsf}{\phi}
\newcommand{\lpf}{\psi}
\newcommand{\lP}[2][\opt]{\mathtt{P}{#1}_{#2}}
\newcommand{\lER}[2][\opt]{\mathtt{R}{#1}_{#2}}
\newcommand{\lX}{\mathbf{X}}
\newcommand{\lU}{\mathbin{\mathbf{U}}}
\newcommand{\lBU}[1]{\mathbin{\mathbf{U}^{\leq #1}}}
\newcommand{\ltrue}{\mathtt{tt}}

\newcommand{\lF}{\mathbf{F}}
\newcommand{\lQ}[1][\opt]{\mathtt{Q}{#1}_{\mathord{=}?}}

\newcommand{\errorRate}{\epsilon}
\newcommand{\significanceLevel}{\eta}
\newcommand{\margin}{\lambda}

\newcommand{\nsamples}{\ell}
\newcommand{\nsims}{\mathfrak{N}}

\newcommand{\safetyLevel}{\zeta}
\newcommand{\threshold}{\tau}
\newcommand{\targetValue}{\beta}

\newcommand{\approxfun}[2][\margin]{\tilde{#2}_{#1}}
\newcommand{\ApproxFunOfProperty}[2][\margin]{\tilde{#2}_{\lsf, #1}} 
\newcommand{\ApproxFunOfDiff}[2][\margin]{\tilde{#2}_{\lsf, #1}}
\newcommand{\SampleDomain}[1]{\tilde{#1}} 

\newcommand{\smcEstimation}[1]{\hat{#1}}

\newcommand{\coefficients}[1][c]{\vec{#1}}

\newcommand{\lipcons}{L}
\newcommand{\hyperrectangle}{H}
\newcommand{\hypercube}{\Omega}
\newcommand{\Numdivide}{m}
\newcommand{\Numdimension}{n}
\newcommand{\lenside}{l}
\newcommand{\distance}{d}
\newcommand{\forsmall}{\epsilon}
\newcommand{\lPmc}[1]{\mathtt{P}_{#1}}

\newcommand{\pacpma}[1][]{\textsc{PacPMA}\ensuremath{^{#1}}\xspace}
\newcommand{\storm}{\textsc{Storm}\xspace}
\newcommand{\prism}{\textsc{PRISM}\xspace}

\newcommand{\matlab}{\textsc{MATLAB}\xspace}

\newcommand{\benchexec}{\textsc{BenchExec}\xspace}

\newdefinition{definition}{Definition}
\newtheorem{theorem}{Theorem}
\newtheorem{lemma}{Lemma}
\newtheorem{corollary}{Corollary}

\newtheorem{assumption}{Assumption}
\newtheorem{remark}{Remark}
\newtheorem{example}{Example}
\newenvironment{markedexample}{\begin{example}}{\qed\end{example}}

\iftrue
    \usepackage{todonotes}
	\newcommand{\at}[1]{\todo[inline,color=teal!10,caption={AT}]{\textbf{AT:} #1}}
	
	\newcommand{\czm}[1]{\todo[inline,color=orange!10,caption={CZm}]{\textbf{CZm:} #1}}
        \newcommand{\ly}[1]{\todo[inline,color=orange!10,caption={LY}]{\textbf{LY:} #1}}
        
	\newcommand{\zlj}[1]{\todo[inline,color=red!10,caption={ZLj}]{\textbf{ZLj:} #1}}
\else
    \newcommand{\at}[1]{}
    \newcommand{\czm}[1]{}
    \newcommand{\ly}[1]{}
    \newcommand{\zlj}[1]{}
    
\fi

\journal{Information and Computation}

\begin{document}

\begin{frontmatter}
    \title{PAC Approximation and DIRECT Optimization for Parametric Markov Models}
    
    \author[KLSS,ISCAS,UCAS]{Zhiming Chi}
    \author[IOECAS]{Ying Liu}
    \author[KLSS,ISCAS,IISG]{Andrea Turrini}
    \author[KLSS,ISCAS,UCAS,IISG]{Lijun Zhang\texorpdfstring{\corref{cor}}{}}
    \ead{zhanglj@ios.ac.cn}
    \author[KLSS,ISCAS,UCAS,IISG]{David N. Jansen}
    
    \cortext[cor]{Corresponding author}
    \affiliation[KLSS]{
        organization={Key Laboratory of System Software (Chinese Academy of Sciences)}, 
        city={Beijing},
        country={China}
    }
    \affiliation[ISCAS]{
        organization={Institute of Software, Chinese Academy of Sciences}, 
        city={Beijing},
        country={China}
    }
    \affiliation[UCAS]{
        organization={University of Chinese Academy of Sciences}, 
        city={Beijing},
        country={China}
    }
    \affiliation[IOECAS]{
        organization={Institute of Optics and Electronics, Chinese Academy of Sciences},
        city={Chengdu},
        country={China}
    }
    \affiliation[IISG]{
        organization={Institute of Intelligent Software Guangzhou}, 
        city={Guangzhou},
        country={China}
    }
    
    \begin{abstract}
        In this paper, we consider the parameter synthesis and optimization problem for parametric Markov decision processes (pMDPs), the extension of classical MDPs where exact probability values are replaced by parametric expressions.
        Computing the rational function $f_{\lsf}$ that maps parameter valuations to the satisfaction value of a PRCTL property $\lsf$ is a computationally expensive task, particularly for pMDPs where the optimal policy may vary across the parameter space.
        We adopt the \emph{scenario approach} to efficiently synthesize a probably approximately correct (PAC) approximation $\ApproxFunOfProperty{f}$ of $f_{\lsf}$: 
        by sampling parameter configurations and solving a linear program, we obtain a polynomial approximation whose error margin $\margin$ is guaranteed, with prescribed confidence, for all but an $\errorRate$-fraction of the parameter domain under the sampling distribution.
        We further show how this PAC framework can be combined with statistical model checking (SMC), enabling the analysis of black-box parametric models.
        Building on the PAC approximation, we integrate the DIRECT (DIviding RECTangles) algorithm for derivative-free global optimization over the parameter space.
        We establish conditional optimality-gap guarantees: under explicit Lipschitz and PAC-good-set assumptions, the difference between the true optimum $f_{\lsf}(\parameters^{*})$ and the value found by DIRECT is bounded by a partition-diameter term and, in the PAC case, an additional approximation-error term.
        An empirical evaluation on 2997 benchmarks focuses on the new DIRECT-based optimization component.
        The results show that DIRECT variants solve fewer instances than the scenario optimizer, but on their common successful instances they often return slightly better objective values and usually run faster, while remaining close to the scenario values within the PAC margin.
    \end{abstract}
\end{frontmatter}

\section{Introduction}
\label{sec:introduction}

Probabilistic model checking provides a rigorous framework for verifying quantitative properties of systems that exhibit both stochastic and nondeterministic behavior~\cite{DBLP:books/wi/Puterman94}.
Depending on the presence of nondeterminism, such systems are modeled as discrete-time Markov chains (MCs), where each state has a single probability distribution over successors, or as Markov decision processes (MDPs), where each state offers a nondeterministic choice among multiple transitions leading to potentially different distributions.
Properties of interest are typically specified in probabilistic logics such as PCTL~\cite{DBLP:journals/fac/HanssonJ94,DBLP:conf/fsttcs/BiancoA95} and its reward extension PRCTL~\cite{DBLP:conf/formats/AndovaHK03}, in forms such as $\lsf = \lP[\max]{\mathord{=}?}[\mathit{safe} \lU \success]$ (maximum probability of reaching $\success$ while staying $\mathit{safe}$) or $\lsf = \lER[\min]{\mathord{=}?}[\lF \mathit{goal}]$ (minimum expected reward to reach a $\mathit{goal}$).
Several mature tools support verification against these logics, including \prism~\cite{DBLP:conf/cav/KwiatkowskaNP11}, MRMC~\cite{DBLP:journals/pe/KatoenZHHJ11}, CADP~\cite{DBLP:journals/sttt/GaravelLMS13}, \textsc{IscasMc}~\cite{DBLP:conf/fm/HahnLSTZ14}, PROPhESY~\cite{DBLP:conf/cav/DehnertJJCVBKA15}, and \storm~\cite{DBLP:journals/sttt/HenselJKQV22}, which compute the numerical value assumed by such properties for the given input model.

A central challenge in applying probabilistic model checking to concrete systems is obtaining the precise probability values that govern the transitions of the model.
Since exact values are often unavailable, parametric Markov models have been introduced, where transition probabilities are expressed as functions of the parameters.
These models range from interval MDPs~\cite{DBLP:journals/ai/GivanLD00}, where probabilities lie in given intervals, to the more expressive parametric MDPs (pMDPs)~\cite{DBLP:conf/ictac/Daws04,DBLP:journals/sttt/HahnHZ11}, where transition probabilities are rational functions of the parameters; 
we consider the latter in this paper.

For a PRCTL formula $\lsf$ checked against a parametric Markov model, the satisfaction value is a function $f_{\lsf}$ over the parameters rather than a single number.
This function is typically a rational function, i.e., a fraction between two polynomials.
Computing $f_{\lsf}$ exactly is already challenging for parametric Markov chains (pMCs), as it involves manipulating polynomials with possibly very high degree~\cite{DBLP:journals/iandc/BaierHHJKK20,DBLP:journals/jcss/JungesK0W21}.
For pMDPs, the situation is further complicated by the fact that the policy optimizing the formula may change depending on the parameter values: 
$f_{\lsf}$ is then a piece-wise defined function, with each piece corresponding to a subregion of the parameter domain where a particular policy is optimal.
This necessitates computing multiple rational functions and identifying the subregions, making exact computation prohibitively expensive in practice.
For many downstream tasks, such as threshold checking over a parameter distribution or searching for high-value parameter valuations, a simpler approximation of $f_{\lsf}$ can be useful even when the exact expression is unavailable.
The price is that the guarantee is statistical rather than pointwise: our certificates hold with respect to the chosen sampling distribution over the parameter domain, not uniformly for every parameter valuation unless an additional verification step is performed.

\subsubsection*{Contribution of the paper.}
In our previous work~\cite{DBLP:conf/atva/LiuTHXZ23}, we exploited the \emph{scenario approach}~\cite{DBLP:journals/tac/CalafioreC06,DBLP:journals/arc/CampiGP09} to synthesize an approximation function $\ApproxFunOfProperty{f}$ with PAC (probably approximately correct) guarantees in the context of pMCs.
We then extended in~\cite{DBLP:conf/birthday/ChiLT0J25} this idea to pMDPs with rewards and showed how to combine the PAC approximation with statistical model checking.
The present journal version substantially extends these two works by adding a DIRECT-based global optimization, establishing formal optimality-gap bounds for exact and PAC objectives, presenting a discussion of Lipschitz bounds, providing the proofs for all the stated results, and conducting a larger experimental evaluation.
More precisely, in this paper we make the following contributions:
\begin{enumerate}
\item 
    We consider the PAC-based approximation framework applied to pMDPs decorated with rewards.
    To deal with the fact that for different choices of the parameters the policy optimizing the value of the PRCTL formula $\lsf = \lQ{}[\lpf]$ might change, we treat the MDP policy as in~\cite{DBLP:journals/sttt/BadingsCJJKT22}: each sampled instantiated MDP is evaluated under its own optimal policy, and the collected values are then used to synthesize one single, common approximation of the induced optimal-value function over the parameter domain.
    When enough samples of the parameters are considered, the synthesized function $\ApproxFunOfProperty{f}$ is guaranteed by construction to approximate the actual function $f_{\lsf}$ with $(\errorRate, \significanceLevel, \margin)$-PAC guarantee, that is, with confidence $1 - \significanceLevel$, the probability that the error due to the approximation is within the margin $\margin$ (obtained as a byproduct while synthesizing $\ApproxFunOfProperty{f}$) is at least $1 - \errorRate$.
    As in~\cite{DBLP:conf/atva/LiuTHXZ23}, the coefficients of $\ApproxFunOfProperty{f}$ and the margin $\margin$ are computed by solving a Linear Programming (LP) problem, whose number of constraints is linearly related to the number $\nsamples$ of samples, the number of parameters in the pMDP, and the degree of the polynomial template.
    We extend the set of analyzable properties to pMDP/pMDRM models and introduce new results specific for nondeterministic Markov models, comparing the effect of applying different policies $\policy$ on the induced rational functions $f_{\lsf}^{\policy}$ and the corresponding approximations~$\ApproxFunOfProperty{f}^{\policy}$.

\item 
    We show how scenario-based model checking for parametric Markov models can be combined with statistical model checking (SMC), so as to combine the advantages of both approaches while preserving statistical guarantees for the computed functions.

\item 
    As a contribution specific to this journal version, we integrate the DIRECT (DIviding RECTangles) algorithm~\cite{DIRECT} for derivative-free global optimization over the parameter space of parametric Markov models.
    DIRECT partitions the parameter domain into hyperrectangles, balancing global exploration with local exploitation, and does not require the knowledge of the Lipschitz constant of $f_{\lsf}$ to guide the partitioning.
    We establish conditional guarantees on the optimality gap between the true optimum and the value found by DIRECT, under an explicit Lipschitz assumption, 
    and extend this bound to the PAC setting where DIRECT operates on the approximation $\ApproxFunOfProperty{f}$ rather than on the exact function $f_{\lsf}$.

\item 
    We perform an extensive DIRECT-focused experimental evaluation on 2997 benchmarks comprising both pMDPs and pMCs with probability and reward properties, substantially extending the benchmark suite used in~\cite{DBLP:conf/birthday/ChiLT0J25}.
    The current experiments quantify both the agreement and the direction of improvement between the scenario-based optimizer and DIRECT on their common successful instances.
    They show that DIRECT variants solve fewer instances overall, but often return slightly better objective values and run faster on the subset of benchmarks they solve; most value differences remain within the PAC margin.
\end{enumerate}
We implemented our framework in our tool \pacpma\footnote{\url{https://github.com/iscas-tis/PacPMA/}}.

\subsubsection*{Related work.}
A substantial body of work addresses the efficient computation of the rational function $f_{\lsf}$ for parametric Markov models.
Daws~\cite{DBLP:conf/ictac/Daws04} proposed an automata-theoretic approach for pMC reachability, by computing $f_{\lsf}$ via state elimination and regular expression conversion.
This was extended by Hahn et~al.~\cite{DBLP:journals/sttt/HahnHZ11} to support bounded reachability and expected reward in the PARAM tool, improving the state elimination~\cite{DBLP:books/daglib/0016921} strategy.
Subsequent optimizations include: efficient state elimination ordering~\cite{DBLP:conf/qest/JansenCVWAKB14}, arithmetic circuit encodings~\cite{DBLP:conf/atva/GainerHS18}, fraction-free Gaussian elimination~\cite{DBLP:journals/iandc/BaierHHJKK20}, parameter lifting with monotonicity checking~\cite{DBLP:conf/tacas/SpelJK21}, and compositional fragment-based analysis~\cite{DBLP:conf/icse/FangCGA21}.

For parametric Markov models with nondeterminism, Hahn et~al.~\cite{DBLP:conf/nfm/HahnHZ11} handled nested PRCTL formulas by partitioning the parameter domain into hyperrectangles with uniform optimal policies, and Quatmann et~al.~\cite{DBLP:conf/atva/QuatmannD0JK16} established the first sound and feasible parameter synthesis procedure for pMDPs.
These approaches target exact symbolic or certified parameter synthesis, and therefore solve a different problem from ours.
In particular, when nondeterminism causes the optimal policy to change across the parameter domain, exact methods typically have to reason about policy-stable regions and the corresponding exact rational functions separately.
Our approach does not try to recover this exact partition; instead, it samples instantiated models, evaluates the induced optimal value, and synthesizes one compact approximation of the resulting optimal-value function over the whole domain, including regions where the optimizing policy changes.
The price is that the guarantee is distributional and statistical rather than an exact symbolic characterization of every policy region.
Our approach, building on the framework we developed in in~\cite{DBLP:conf/atva/LiuTHXZ23,DBLP:conf/birthday/ChiLT0J25}, provides such guarantees for the efficiently synthesized approximation function.
For a comprehensive survey, we refer the reader to~\cite{DBLP:conf/birthday/0001JK22}.

The most closely related work is that of Badings et~al.~\cite{DBLP:journals/sttt/BadingsCJJKT22}, which also applies the scenario approach~\cite{DBLP:journals/mp/CalafioreC05} to parametric Markov models.
The scenario approach relaxes an optimization problem by sampling a subset of constraints; if enough samples are drawn, the relaxed solution satisfies the original constraints with statistical guarantees~\cite{DBLP:journals/tac/CalafioreC06,DBLP:journals/arc/CampiGP09}.
It has been applied to robust optimization~\cite{DBLP:journals/tac/CalafioreC06}, safety of continuous-time dynamical systems~\cite{DBLP:journals/tcad/XueZEL20}, and neural network robustness~\cite{DBLP:conf/icse/LiYHS0Z22}.
The key difference between our work and the one of Badings et~al.\ is in how the scenario approach is used: 
they compute the probability that parameter instances satisfy a PRCTL formula $\lsf$ by comparing each instance's model checking result against a target value $\targetValue$, whereas we synthesize a low-degree polynomial approximation $\ApproxFunOfProperty{f}$ of the entire function $f_{\lsf}$ and then use $\ApproxFunOfProperty{f}(\parameters)$ for analysis.
This gives us a compact, differentiable representation that supports plotting, optimization, and threshold comparison with PAC guarantees under the chosen sampling measure, while avoiding the symbolic expression swell that often arises in exact rational function computation.
Statistical model checking (SMC)~\cite{DBLP:series/lncs/LegayLTYSG19} estimates the probability that a system satisfies a temporal property by simulating multiple execution traces, rather than by constructing and analyzing the full state space.
This simulation-based approach scales to complex systems and has been successfully applied to cyber-physical systems~\cite{DBLP:journals/scn/XieTFH21} and autonomous driving~\cite{DBLP:conf/ivs/BarbierRQRPLLI019,DBLP:conf/ivs/PaigwarBRLL20}, among other domains.
In our framework, SMC serves as an alternative back-end for evaluating the property satisfaction value at sampled parameter points, enabling the analysis of black-box parametric models for which the internal structure is not available.

\paragraph*{Organization of the paper}
Section~\ref{sec:preliminaries} introduces the models and logic we use.
Section~\ref{sec:synthesisPACfunctions} recalls the PAC-based model checking approach from~\cite{DBLP:conf/atva/LiuTHXZ23} and extends it to Markov decision processes.
Section~\ref{sec:combiningPACwithSMC} presents the combination of PAC-based model checking with statistical model checking.
Section~\ref{sec:DIRECT} introduces the DIRECT algorithm for global optimization and establishes theoretical guarantees on the optimality gap.
Section~\ref{sec:experiments} reports on the experimental evaluation, and Section~\ref{sec:conclusion} concludes with a discussion on possible extensions.

\section{Preliminaries}
\label{sec:preliminaries}

We begin by recalling the relevant background on probability measures, Markov models, reward structures, and the temporal logic PRCTL (see, e.g.,~\cite{DBLP:books/daglib/0020348,DBLP:books/wi/Puterman94} for more details).
We then introduce the parametric extensions of these models, which form the basis for the approximation framework presented in Section~\ref{sec:synthesisPACfunctions}.

A \emph{$\sigma$-field} over a set $X$ is a set $\mathcal{F} \subseteq 2^{X}$ that includes $X$ and is closed under complement and countable union. 
A \emph{measurable space} is a pair $(X, \mathcal{F})$ where $X$ is a set, also called the \emph{sample space}, and $\mathcal{F}$ is a $\sigma$-field over $X$. 
A measurable space $(X, \mathcal{F})$ is called \emph{discrete} if $\mathcal{F} = 2^{X}$. 
A \emph{measure} over a measurable space $(X, \mathcal{F})$ is a function $\pd \colon \mathcal{F} \to \posreals$ such that, for each countable collection $\setnocond{X_{i}}_{i \in I}$ of pairwise disjoint elements of $\mathcal{F}$, $\pd(\cup_{i \in I} X_{i}) = \sum_{i \in I} \pd(X_{i})$. 
A \emph{probability measure} over a measurable space $(X, \mathcal{F})$ is a measure $\pd$ over $(X, \mathcal{F})$ such that $\pd(X) = 1$. 
A measure over a discrete measurable space $(X, 2^{X})$ is called a \emph{discrete measure} over $X$. 
The \emph{support} of a discrete measure $\pd$ over $(X, \mathcal{F})$, denoted by $\Supp{\pd}$, is the set $\setcond{x \in X}{\pd(x) > 0}$.
When $X$ is finite, we usually call the discrete measure $\pd$ over $(X, 2^{X})$ a probability distribution and we denote by $\Dist{X}$ the set of all probability distributions over $X$.

\subsection{Probabilistic Models}
\label{ssec:probabilisticModels}

\begin{definition}
\label{def:mdp}
    Given a finite set of atomic propositions $\AP$, a \emph{(labelled) Markov decision process} (MDP) $\mdp$ is a tuple $\mdp = (\mstates, \minit, \mactions, \mtransitions, \mlabelling)$, where
    $\mstates$ is a finite set of \emph{states;}
    $\minit \in \mstates$ is the \emph{initial state;}
    $\mactions$ is a finite set of \emph{actions;}
    $\mtransitions \colon \mstates \times \mactions \times \mstates \to \posreals$ is a \emph{transition function} such that for each $s \in \mstates$ and $a \in \mactions$, $\sum_{s' \in \mstates} \mtransitions(s, a, s') \in \setnocond{0, 1}$; 
    and
    $\mlabelling \colon \mstates \to 2^{\AP}$ is a \emph{labelling function}.
\end{definition}
We denote by $\mactions(s) = \setcond{a \in \mactions}{\sum_{s' \in \mstates} \mtransitions(s, a, s') = 1}$ the set of actions enabled by the state $s$;
we require that $\size{\mactions(s)} \geq 1$ for all states $s \in \mstates$.

A discrete-time Markov chain can be seen as an MDP where $\size{\mactions(s)} = 1$ for all states $s \in \mstates$;
thus $\mactions$ can be omitted and we get the usual definition of Markov chain:
\begin{definition}
\label{def:mc}
    Given a finite set of atomic propositions $\AP$, a \emph{(labelled) (discrete-time) Markov chain} (MC) $\mc$ is a tuple $\mc = (\mstates, \minit, \mtransitions, \mlabelling)$ where
    $\mstates$ is a finite set of \emph{states};
    $\minit \in \mstates$ is the \emph{initial state};
    $\mtransitions \colon \mstates \times \mstates \to \posreals$ is a \emph{transition function} such that for each $s \in \mstates$, $\sum_{s' \in \mstates} \mtransitions(s, s') = 1$; 
    and
    $\mlabelling \colon \mstates \to 2^{\AP}$ is a \emph{labelling function}.
\end{definition}

The \emph{underlying graph} of an MDP $\mdp = (\mstates, \minit, \mactions, \mtransitions, \mlabelling)$ is a directed graph $\langle V, E \rangle$ with $V = \mstates$ as vertices and $E = \setcond{(s, s') \in \mstates \times \mstates}{\exists a \in \mactions. \mtransitions(s, a, s') > 0}$ as edges.


\begin{figure}[!htbp]
    \centering
        \includegraphics{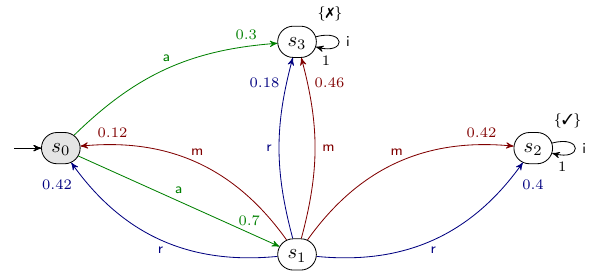}
    \caption{An example MDP modeling the nondeterministic choice in state $s_{2}$ between local \textsf{r}epair and \textsf{m}igration}
    \label{fig:mdp}
\end{figure}
\begin{markedexample}
\label{ex:mdp}
    As an example of MDP, consider the model shown in Figure~\ref{fig:mdp}, which represents a simplified cloud-service recovery procedure.
    It has four states: 
    $s_{0}$ is the initial admission state, $s_{1}$ is the recovery decision point, $s_{2}$ is the successful recovery state, and $s_{3}$ is the failed request state.
    The set of actions is $\setnocond{\textsf{a}, \textsf{r}, \textsf{m}, \textsf{i}}$, representing the actual \textsf{a}dmit, \textsf{r}epair, \textsf{m}igrate, and \textsf{i}dle actions, respectively, where \textsf{r} and \textsf{m} are the only nondeterministic alternatives in the MDP enabled by $s_{1}$.
    The non-zero probability values of the transitions are depicted on the directed arcs, together with the action enabling them.
    As labelling function $\mlabelling$, all states are mapped to $\emptyset$ except for $s_{2}$ and $s_{3}$, for which we have $\mlabelling(s_{2}) = \setnocond{\success}$ and $\mlabelling(s_{3}) = \setnocond{\failure}$.
\end{markedexample}

Let $\mdp$ be an MDP.
An \emph{infinite path} $\mpath$ of $\mdp$ is an alternating sequence of states and actions $\mpath = s_{0} a_{0} s_{1} a_{1} s_{2} \cdots$ such that $\mtransitions(s_{i}, a_{i}, s_{i+1}) > 0$ for each $i \in \naturals$;
a \emph{finite path} is just a finite prefix of an infinite path.
We let $\mpathsfin{\mdp}$ ($\mpaths{\mdp}$, resp.) denote the set of all finite (infinite, resp.) paths of $\mdp$.
Given a path $\mpath$, we write $\mpathstate{i}$ for the state $s_{i}$ and $\mpathaction{i}$ for the action $a_{i}$;
for a finite path $\mpath = s_{0} a_{1} s_{1} a_{1} s_{2} \cdots s_{n}$, we denote by $\size{\mpath} = n + 1$ its number of states and by $\last{\mpath} = s_{n}$ its last state.

Nondeterminism in MDPs is resolved by \emph{policies}, which select an action at each step based on the history of visited states.
Formally, a policy $\policy$ for an MDP $\mdp$ is a function $\policy \colon \mpathsfin{\mdp} \to \Dist{\mactions}$ such that $\Supp{\policy(\mpath)} \subseteq \mactions(\last{\mpath})$ for each finite path $\mpath$;
we denote by $\policies(\mdp)$ the set of all policies for $\mdp$ and we might simply write $\policies$ for $\policies(\mdp)$ when $\mdp$ is clear from the contest.
A policy $\policy$ is called \emph{memoryless} if it depends only on the last state of the path, i.e., $\policy(\mpath) = \policy(\mpath')$ whenever $\last{\mpath} = \last{\mpath'}$ for each $\mpath, \mpath' \in \mpathsfin{\mdp}$; 
it is called \emph{deterministic} if it always selects a single action with probability 1, i.e., $\size{\Supp{\policy(\mpath)}} = 1$ for each $\mpath \in \mpathsfin{\mdp}$.

Given an MDP $\mdp$, a policy $\policy \in \policies$, and a state $s$, $\policy$ induces a probability measure over the paths of $\mdp$ from the state $s$ as follows:
the basic measurable events are the cylinder sets of finite paths, where the \emph{cylinder set} of a finite path $\mpath$ is the set $\cylinder{\mpath} = \setcond{\mpath' \in \mpaths{\mdp}}{\text{$\mpath$ is a prefix of $\mpath'$}}$.
The probability $\prob_{s, \policy}$ of a cylinder set $\cylinder{\mpath}$ is defined inductively as follows:
\[
    \prob_{s, \policy}(\cylinder{\mpath}) = 
    \begin{cases}
        1 & \text{if $\mpath = s$,} \\
        0 & \text{if $\mpath = t \neq s$,} \\
        \prob_{s, \policy}(\cylinder{\mpath'}) \cdot \policy(\mpath')(a) \cdot \mtransitions(\last{\mpath'}, a, t) & \text{if $\mpath = \mpath' a t$.}
    \end{cases}
\]
Standard measure-theoretical arguments ensure that $\prob_{s, \policy}$ extends uniquely to the $\sigma$-field generated by cylinder sets (see, e.g.,~\cite{Billingsley95}).

Similarly to MCs, MDPs can be decorated with reward structures that assign values to states and transitions, modeling quantities such as energy consumption or time spent in particular configurations.
\begin{definition}
\label{def:mdrm}
    A \emph{Markov decision reward model} (MDRM) $\mdrm$ is a pair $\mdrm = (\mdp, \mreward)$ where $\mdp$ is an MDP and $\mreward = \mreward_{\mstates} \cup \mreward_{\mactions}$ is the \emph{reward structure}, where $\mreward_{\mstates} \colon \mstates \to \posreals$ and $\mreward_{\mactions} \colon \mstates \times \mactions \to \posreals$ are the state and action reward functions, respectively.
\end{definition}
\begin{markedexample}
\label{ex:rewardStructure}
    In our cloud-service example, we can define a reward structure that tracks the recovery cost: 
    a large penalty is obtained when the failed request state $s_{3}$ is reached (e.g., $\mreward_{\mstates}(s_{3}) = 1000$ and $\mreward_{\mstates}(s_{i}) = 0$ for $i \in \setnocond{0, 1, 2}$), and a higher action reward can be assigned to action $\mathsf{m}$ than to action $\mathsf{r}$ if moving the request to a backup instance consumes more resources than performing a local repair (e.g., $\mreward_{\mactions}(s_{1}, \mathsf{m}) = 10$ and $\mreward_{\mactions}(s_{1}, \mathsf{r}) = 3$).
\end{markedexample}

Given an MDRM $\mdrm = (\mdp, \mreward)$, a state $s$, and a policy $\policy$, we can define $\exprew_{s, \policy}^{\mdrm}$, the \emph{expected cumulative reward under $\policy$}, as follows (cf.~\cite{DBLP:books/daglib/0020348}):
given a set $T \subseteq \mstates$ of states, $\exprew_{s, \policy}^{\mdrm}(T)$ is the expectation with respect to the probability $\prob_{s, \policy}$ of the random variable $X^{T} \colon \mpaths{\mdp} \to \posreals$ defined as follows:
\[
    X^{T}(\mpath) = 
    \begin{cases}
        0 & \text{if $\mpath_{0} \in T$,} \\
        \infty & \text{if $\mpath_{i} \notin T$ for each $i \in \naturals$,} \\
        \sum\limits_{i = 0}^{\min \setcond{n \in \naturals}{\mpathstate{n} \in T} - 1} \mreward(\mpathstate{i}) + \mreward(\mpathstate{i}, \mpathaction{i}) & \text{otherwise.}
    \end{cases}
\]
We recall that if the policy $\policy$ does not make reaching $T$ almost sure, i.e., $\prob_{s, \policy}(\setcond{\mpath \in \mpathsfin{\mdp}}{\last{\mpath} \in T}) < 1$, then $\exprew_{s, \policy}^{\mdrm}(T) = \infty$ (cf.~\cite[Definition~10.71]{DBLP:books/daglib/0020348}).

\subsection{Probabilistic Reward Logic PRCTL}
\label{ssec:logicPRCTL}

To specify properties of MDRMs, we use PRCTL~\cite{DBLP:conf/formats/AndovaHK03}, which extends PCTL~\cite{DBLP:journals/fac/HanssonJ94} with reward operators.
PRCTL formulas are defined by the following grammar, where $\lsf$ is a \emph{state formula} and $\lpf$ is a \emph{path formula}:
\begin{align*}
    \lsf & ::= a \mid \lnot \lsf \mid \lsf \land \lsf \mid \lP{\mathord{\bowtie p}}(\lpf) \mid \lER{\mathord{\bowtie r}}(\lF \lsf) \\
    \lpf & ::= \lX \lsf \mid \lsf \lU \lsf \mid \lsf \lBU{k} \lsf
\end{align*}
where $a \in \AP$, 
$\opt \in \setnocond{\min, \max}$, 
$\mathord{\bowtie} \in \setnocond{\mathord{<}, \mathord{\leq}, \mathord{\geq}, \mathord{>}}$, 
$p \in \interval{0}{1}$, 
$r \in \posreals$,
and $k \in \naturals$.
We use freely the usually derived operators, like $\lsf_{1} \lor \lsf_{2} = \lnot(\lnot \lsf_{1} \land \lnot\lsf_{2})$, $\ltrue = a \lor \neg a$, and $\lF \lsf = \ltrue \lU \lsf$.
If we omit the $\lER{\mathord{\bowtie r}}(\lF \lsf)$ operator from PRCTL, we get the standard PCTL logic for MDPs.

The semantics of a state formula $\lsf$ and of a path formula $\lpf$ is given with respect to a state $s$ and a path $\mpath$ of a DMRM $\mdrm = (\mdp, \mreward)$, respectively.
The semantics is standard for all Boolean and temporal operators; 
for the $\lP{\mathord{\bowtie} p}$ operator, it is defined as 
\[
    \text{$s \models \lP{\mathord{\bowtie} p}(\lpf)$ iff $\opt_{\policy \in \policies} \setnocond{\prob_{s, \policy}(\setcond{\mpath \in \mpaths{\mdp}}{\mpath \models \lpf})}\bowtie p$}
\]
and, similarly, 
\[
    \text{$s \models \lER{\mathord{\bowtie} r}(\lF \lsf)$ iff $\opt_{\policy \in \policies} \setnocond{\exprew_{s, \policy}(\setcond{t \in \mstates}{t \models \lsf}}) \bowtie r$.}
\]

With some abuse of notation, we write $\mdrm \models \lsf$ if $\minit \models \lsf$; 
we also consider $\lP{\mathord{=}?}(\lpf)$ and $\lER{\mathord{=}?}(\lF \lsf)$ as PRCTL formulas, asking to compute the probability (resp.\@ expected reward) of satisfying $\lpf$ (resp.\@ $\lF \lsf$) in the initial state $\minit$ of $\mdrm$, i.e., to compute the value $\opt_{\policy \in \policies} \setnocond{\prob_{\minit, \policy}(\setcond{\mpath \in \mpaths{\mdp}}{\mpath \models \lpf})}$ (resp.\@ $\opt_{\policy \in \policies} \setnocond{\exprew_{\minit, \policy}(\setcond{t \in \mstates}{t \models \lsf})}$).
We just write $\lQ$ to denote either $\lP{\mathord{=}?}(\lpf)$ or $\lER{\mathord{=}?}(\lF \lsf)$, i.e., when we are just interested in computing an expected value about the MDP.

\begin{markedexample}
\label{ex:prctlFormula}
    As a concrete illustration of PRCTL formulas, consider the MDP from Figure~\ref{fig:mdp}.
    For the property $\lP{\mathord{=}?}(\lF \success)$, choosing \textsf{r} in $s_{1}$ gives reachability probability $\frac{0.7 \cdot 0.4}{1 - 0.7 \cdot 0.42} \approx 0.397$, while choosing \textsf{m} gives $\frac{0.7 \cdot 0.42}{1 - 0.7 \cdot 0.12} \approx 0.321$.
    Thus $\lP[\max]{\mathord{=}?}(\lF \success) \approx 0.397$ and $\lP[\min]{\mathord{=}?}(\lF \success) \approx 0.321$.
    Both the $\min$ and $\max$ expected reward accumulated before reaching $s_{2}$ are infinite, given that the absorbing failed request state $s_{3}$ can be reached with positive probability.
\end{markedexample}

It is well known that deterministic memoryless policies suffice to compute $\lQ$ and that applying a policy $\policy$ to a MDRM $\mdrm = (\mdp, \mreward)$ induces a MC $\applypolicy{\mdp}$ and the corresponding reward structure $\applypolicy{\mreward}$, where for instance the probability of going from state $s$ to state $s'$ in the MC $\applypolicy{\mdp}$ induced by a memoryless policy $\policy$ is $\sum_{a \in \mactions(s)} \policy(s)(a) \cdot \mtransitions(s, a, s')$.
For more details, we refer the interested reader to, e.g.,~\cite{DBLP:books/daglib/0020348,DBLP:conf/sfm/ForejtKNP11}.

\subsection{Parametric Models}
\label{ssec:parametricModels}

We now present the parametric extensions of MDPs and MDRMs, following~\cite{DBLP:conf/nfm/HahnHZ11,DBLP:journals/sttt/HahnHZ11,DBLP:conf/atva/LiuTHXZ23}.
Let $\parametersSet = \setnocond{\parameter_{1}, \dotsc, \parameter_{n}}$ be a finite set of \emph{parameters}, with each parameter $\parameter_{i}$ being associated by $\range \colon \parametersSet \to \reals$ to a closed interval $\range(\parameter_{i}) = \interval{L_{\parameter_{i}}}{U_{\parameter_{i}}} \subseteq \reals$ of its valid values.
We consider the parameters in a fixed order, writing them as a vector $\parameters = (\parameter_{1}, \dotsc, \parameter_{n})$; 
this induces the domain of the parameters $\parametersDomain = \prod_{i=1}^{n} \range(\parameter_{i})$ with $\parametersDomain \subseteq \reals^{n}$.
An \emph{evaluation} $\evaluation$ is a function $\evaluation \colon \parametersSet \to \reals$ assigning to each parameter one specific value in its range, i.e., for each $\parameter \in \parametersSet$, $\evaluation(\parameter) \in \range(\parameter)$;
we write $\evaluation(\parameters)$ to denote the vector of real numbers induced by $\evaluation$ on $\parameters$, i.e., $\evaluation(\parameters) = (\evaluation(\parameter_{1}), \cdots, \evaluation(\parameter_{n}))$.

Let $\polynomialsSet$ denote the ring of the polynomials with variables $\parametersSet$ over the field $\reals$ of real numbers; 
a \emph{rational function} $f$ is the ratio between two polynomials $g_{1}, g_{2} \in \polynomialsSet$, that is, $f(\parameters) = \frac{g_{1}(\parameters)}{g_{2}(\parameters)}$;
let $\ratfunsSet$ be the set of all rational functions.
Given $f = \frac{g_{1}}{g_{2}} \in \ratfunsSet$ and an evaluation $\evaluation$, we denote by $\evaluate{f}$ the rational number $f(\evaluation(\parameters))$;
we assume that $\evaluate{f}$ is well defined for each evaluation $\evaluation$, that is, $\evaluate{g_{2}} \neq 0$.

\begin{definition}
\label{def:pmc}
    Given a finite set of parameters $\parametersSet$, a \emph{parametric Markov decision process} (pMDP) $\pmdp$ with parameters $\parametersSet$ is a tuple $\pmdp = (\mstates, \minit, \mactions, \mtransitions, \mlabelling)$ where $\mstates$, $\minit$, $\mactions$ and $\mlabelling$ are as in Definition~\ref{def:mdp}, while $\mtransitions \colon \mstates \times \mactions \times \mstates \to \ratfunsSet$.
\end{definition}

\begin{definition}
\label{def:pmcInstantiation}
    Given a pMDP $\pmdp$ and an evaluation $\evaluation$, the evaluation $\evaluation$ \emph{induces} the MDP $\evaluate{\mdp} = (\mstates, \minit, \mactions, \mtransitions_{\evaluation}, \mlabelling)$ where $\mtransitions_{\evaluation}$ is defined, for each $s, s' \in \mstates$ and $a \in \mactions$, as $\mtransitions_{\evaluation}(s, a, s') = \evaluate{\mtransitions(s, a, s')}$.
\end{definition}

The extension of MDRMs to parametric MDRMs (pMDRMs) is natural:
a pMDRM $\pmdrm$ is just a pair $\pmdrm = (\pmdp, \mreward)$ where $\pmdp$ is a pMDP and $\mreward$ is a reward function.
In this paper, parameters occur affect the transition probabilities of the underlying pMDP; 
the reward structure $\mreward$ is non-parametric.

Throughout the paper, we consider only well-posed parameter domains ensuring that the induced MDP is indeed a MDP according to Definition~\ref{def:mdp};
this can be achieved by e.g.\@ restricting $\parametersDomain$ to only those points inducing a valid MDP.
The following standing assumption fixes this convention.
\begin{assumption}
\label{asmt:wellPosedDomain}
    Given a pMDP $\pmdp$, for each evaluation $\evaluation$ such with $\evaluation(\parameters) \in \parametersDomain$, we assume that $\evaluate{\mdp}$ is a valid MDP satisfying Definition~\ref{def:mdp}.
\end{assumption}

To simplify the presentation and ensure that the underlying graph of $\pmdp$ does not depend on the actual evaluation, we make the following assumption:
\begin{assumption}
\label{asmt:sameGraph}
    Given a pMDP $\pmdp$ and any pair of evaluations $\evaluation_{1}$ and $\evaluation_{2}$, for the induced MDPs $\evaluate[\evaluation_{1}]{\pmdp}$ and $\evaluate[\evaluation_{2}]{\pmdp}$ we require that for each $s, s' \in \mstates$ and $a \in \mactions$, it holds that $\mtransitions_{\evaluation_{1}}(s, a, s') = 0$ if and only if $\mtransitions_{\evaluation_{2}}(s, a, s') = 0$.
\end{assumption}
By this assumption, different evaluations can only change the strictly positive probability value of the transition from $s$ to $s'$ by action $a$; 
they cannot, however, change whether the probability of the transition from $s$ to $s'$ by action $a$ is $0$.

\begin{figure}[tb]
    \centering
    \includegraphics{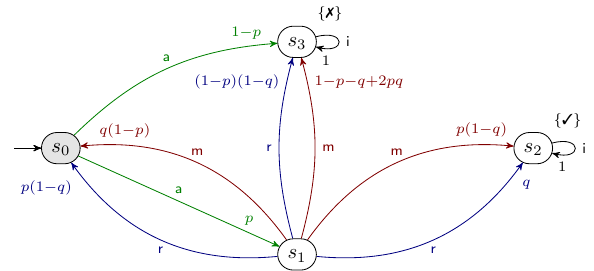}
    \caption{The parametric counterpart of the MDP shown in Figure~\ref{fig:mdp}}
    \label{fig:pmdp}
\end{figure}
\begin{markedexample}
\label{ex:pmdp}
    As an example of pMDP, consider the parametric recovery model shown in Figure~\ref{fig:pmdp}.
    It is the parametric counterpart of the MDP in Figure~\ref{fig:mdp}, where the numerical probabilities have been replaced by expressions over the parameters $p$ and $q$. 
    The parameter $p$ is the reliability of admitting a request, while $q$ is the probability that local repair succeeds.
    If local repair does not immediately succeed, the request is retried from the admission state with probability $p(1-q)$ and fails with probability $(1-p)(1-q)$.
    Migration is favorable when admission reliability is high and local repair is less favorable: it succeeds with probability $p(1-q)$, retries admission with probability $q(1-p)$, and fails with the remaining probability $1-p-q+2pq$.
    We use $\range(p) = \interval{0.2}{0.9}$ and $\range(q) = \interval{0.2}{0.8}$.
    Under evaluation $\evaluation(p) = 0.7$ and $\evaluation(q) = 0.4$, we recover the concrete MDP of Figure~\ref{fig:mdp}.
    
    The rational function for $\lP{\mathord{=}?}(\lF \success)$ under the policy $\policy^{\textsf{r}}$ choosing \textsf{r} in $s_{1}$ is $f_{\lsf}^{\textsf{r}}(p,q) = \frac{pq}{1-p^{2}+p^{2}q}$, while the function under the policy $\policy^{\textsf{m}}$ choosing \textsf{m} is $f_{\lsf}^{\textsf{m}}(p, q) = \frac{p^{2}(1-q)}{1-pq+p^{2}q}$.
    The two policy functions and the regions where each policy gives a higher value are shown in the left plot of Figure~\ref{fig:pmdp1MultipleFunctions}.
    These functions follow directly from the Bellman equations
    \[
        x_{s_{0}} = p \cdot x_{s_{1}} \qquad x_{s_{1}}^{\textsf{r}} = q \cdot 1 + p(1-q) \cdot x_{s_{0}} \qquad x_{s_{1}}^{\textsf{m}} = p(1-q) \cdot 1 + q(1-p) \cdot x_{s_{0}}
    \]
    that, when solved with $x_{s_{1}} = x_{s_{1}}^{\textsf{r}}$ and $x_{s_{1}} = x_{s_{1}}^{\textsf{m}}$ due to the policy choosing \textsf{r} and \textsf{m} in $s_{1}$, respectively, give
    \[
        x_{s_{0}}^{\textsf{r}} = \frac{pq}{1-p^{2}+p^{2}q}
        \qquad
        x_{s_{0}}^{\textsf{m}} = \frac{p^{2}(1-q)}{1-pq+p^{2}q},
    \]
    respectively, and hence the two functions $f_{\lsf}^{\textsf{r}}(p,q)$ and $f_{\lsf}^{\textsf{m}}(p, q)$.
    When evaluated at $(p, q) = (0.7,0.4)$, we get the values seen in Example~\ref{ex:prctlFormula}.
    If we change only the value of $q$, we get that e.g.\@ at $q = 0.2$, migration is better and gives value $\approx 0.409$ while at $q = 0.6$, local repair is better and gives value $\approx 0.522$.
\end{markedexample}

\begin{figure}[tb]
    \centering
    \resizebox{\linewidth}{!}{
        \begin{tikzpicture}
            \node (rat) at (0,0) {\includegraphics{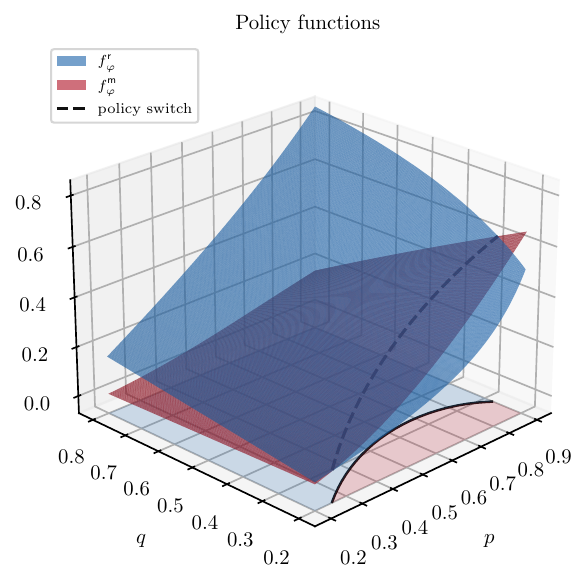}};
            \node[anchor=west] (apprx) at (rat.east) {\includegraphics{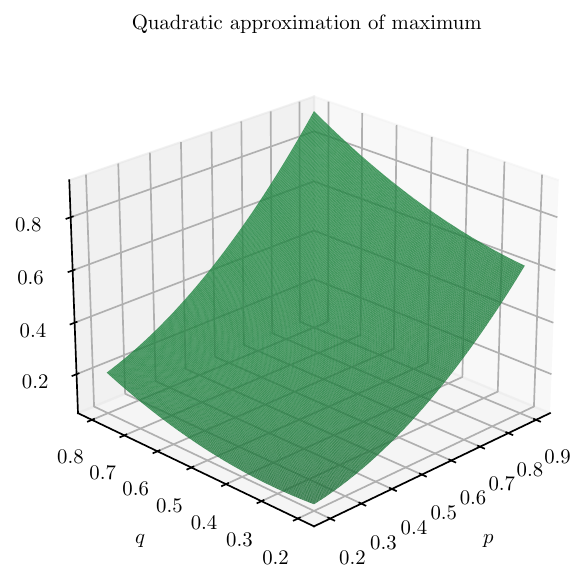}};
        \end{tikzpicture}
    }
    \caption{On the left, the rational functions $f_{\lsf}^{\textsf{r}}$ and $f_{\lsf}^{\textsf{m}}$, with floor shading indicating where it is preferable to choose the repair (blue) or migration (red) action, so to get a higher value; the dashed curve marks the policy switch. On the right, the quadratic approximation of $\max \setnocond{f_{\lsf}^{\textsf{r}}, f_{\lsf}^{\textsf{m}}}$}
    \label{fig:pmdp1MultipleFunctions}
\end{figure}
The \emph{policy switch} is the set of parameter values where the two deterministic memoryless policies give the same reachability probability, that is, $f_{\lsf}^{\textsf{r}}(p,q) = f_{\lsf}^{\textsf{m}}(p,q)$; 
it is shown on the plot on the left side of Figure~\ref{fig:pmdp1MultipleFunctions} as the black dashed curve and its straight curve projection.
On this curve, both repair and migration are equally good for $\lP(\lF\success)$.
On one side of the curve, in light blue, we have $f_{\lsf}^{\textsf{r}}(p,q) > f_{\lsf}^{\textsf{m}}(p,q)$, so the maximizing scheduler chooses local repair in state $s_{1}$; 
on the other side, in light red, we have $f_{\lsf}^{\textsf{r}}(p,q) < f_{\lsf}^{\textsf{m}}(p,q)$, so it chooses migration.
Thus the dashed curve is the boundary of the two policy regions. 
The plot on the right of Figure~\ref{fig:pmdp1MultipleFunctions} approximates via a quadratic polynomial the pointwise maximum obtained by taking the larger of the two policy-specific functions at each parameter point.
The quadratic approximation shown on the right is synthesized from randomly sampled scenario points, as described later in Section~\ref{ssec:PACmodels}; the regularly spaced black points that appeared in an earlier version were only a fixed visualization grid and were not used to compute the approximation.

\section{Probably Approximately Correct Function Synthesis}
\label{sec:synthesisPACfunctions}

In this section, we briefly recall from~\cite{DBLP:conf/atva/LiuTHXZ23} how to get low-degree polynomials that approximate the exact rational functions corresponding to the PRCTL properties of a discrete-time Markov chain, while providing a statistical PAC guarantee on the closeness of the approximating polynomial with the approximated function.
We then recall and adapt the pMDP extension developed in~\cite{DBLP:conf/birthday/ChiLT0J25} to introduce the notation we use and make this work self-contained before introducing our novel DIRECT development and expanded experiments.

\subsection{Probably Approximately Correct Models}
\label{ssec:PACmodels}

As mentioned in the introduction and as shown in~\cite{DBLP:conf/atva/LiuTHXZ23}, a polynomial for the rational function relative to a PRCTL property against a pMC can be synthesized by sampling values for the parameters, evaluate the value of the property in the instantiated MC, and then finding the coefficients of the polynomial that minimize the error between the polynomial and the computed property value on the sampled parameters.
The synthesized polynomial comes with a probably approximately correctness guarantee with respect to the given significance level $\significanceLevel$ and error rate $\errorRate$, provided that a sufficiently large number of samples is used.

The synthesis of the approximating polynomial is based on the scenario approach~\cite{DBLP:journals/tac/CalafioreC06,DBLP:journals/arc/CampiGP09} that was originally applied in the context of robust convex programming problems~\cite{DBLP:journals/mor/Ben-TalN98, DBLP:journals/orl/Ben-TalN99, DBLP:journals/siamjo/GhaouiOL98}. 
These problems are in general described by the following uncertain convex optimization problem
\begin{equation}
\label{eq:optProblemExact}
    \begin{split}
        \min \quad & \vec{a}^{T} \theta \\
        \mathrm{s.t.} \quad & f_{\omega}(\theta) \leq 0 \qquad  \forall \omega \in \Omega \\
        & \theta \in \Theta
    \end{split}
\end{equation}
In the problem above, $\Theta \subseteq \reals^{m}$ is a convex and closed domain of the optimization variables $\theta$ so that the objective function $\vec{a}^{T} \theta$ is minimized.
$\Omega$ is the domain of the uncertain parameters $\omega$ and for each of these parameters $\omega$, the constraint $f_{\omega}(\theta) \leq 0$ is based on the function $f_{\omega} \colon \Theta \to \reals$ that is assumed to be a convex function of $\theta \in \Theta$.

In general, the size of the set of parameters $\Omega$ is infinite, and this makes the optimization problem~\eqref{eq:optProblemExact} have infinitely many constraints, thus making it difficult to be solved. 
To make~\eqref{eq:optProblemExact} solvable, Calafiore and Campi~\cite{DBLP:journals/tac/CalafioreC06} proposed to transform~\eqref{eq:optProblemExact} into a scenario optimization problem, by relying on the Helly theorem~\cite{rockafellar1997convex}.
In the scenario optimization problem, formalized below in Definition~\ref{def:scenarioApproach}, only finitely many constraints are used, with their number being chosen so to provide statistical guarantees on the error made with respect to the exact solution of~\eqref{eq:optProblemExact}.
\begin{definition}
\label{def:scenarioApproach}
    Given the set $\Omega$, let $\Omega_{\nsamples} = \setnocond{\omega_{1}, \dotsc, \omega_{\nsamples}}$ be $\nsamples$ independent identically distributed samples taken from $\Omega$ according to a given probability measure $P$ over $\Omega$.
    The \emph{scenario optimization problem} corresponding to the problem~\eqref{eq:optProblemExact} is defined as
    \begin{equation}
    \label{eq:optProblemPAC}
        \begin{split}
            \min \quad & \vec{a}^{T} \theta \\
            \mathrm{s.t.} \quad & f_{\omega_{i}}(\theta) \leq 0 \qquad \forall \omega_{i} \in \Omega_{\nsamples} \\
            & \theta \in \Theta
        \end{split}
    \end{equation}
\end{definition}
It is easy to observe that the optimization problem~\eqref{eq:optProblemPAC} is in practice a relaxation of the problem~\eqref{eq:optProblemExact}, since we consider fewer constraints for the optimization variables $\theta$, so the optimal solution $\theta^{*}_{l}$ of the problem~\eqref{eq:optProblemPAC} satisfies all constraints $f_{\omega_{i}}(\theta^{*}_{\nsamples}) \leq 0$ for each sample $\omega_{i} \in \Omega_{\nsamples}$, but no requirement is imposed on $\theta^{*}_{\nsamples}$ regarding the constraints $f_{\omega}(\theta^{*}_{\nsamples}) \leq 0$ for $\omega \in \Omega \setminus \Omega_{\nsamples}$. 
The issue now is how to provide a strong enough guarantee that the optimal solution $\theta^{*}_{\nsamples}$ of~\eqref{eq:optProblemPAC} also satisfies the other constraints $f_{\omega}(\theta) \leq 0$ occurring in~\eqref{eq:optProblemExact} corresponding to the choices of $\omega \in \Omega \setminus \Omega_{\nsamples}$ that we have not considered in~\eqref{eq:optProblemPAC}.

Statistics theory ensures that if enough samples $\nsamples$ are taken, then the probability of $\theta^{*}_{\nsamples}$ violating the omitted constraints is smaller than a given threshold: 
the following theorem establishes how many samples need to be taken according to $P$ to format $\Omega_{\nsamples} = \setnocond{\omega_{1}, \dotsc, \omega_{\nsamples}}$ at least to ensure that $\theta^{*}_{\nsamples}$ ``works well'' also for the constraints $f_{\omega}(\theta) \leq 0$ with $\omega \in \Omega \setminus \Omega_{\nsamples}$ of the problem~\eqref{eq:optProblemExact} we have not considered.
The fact that $\theta^{*}_{\nsamples}$ ``works well'' is formalized by using an \emph{error rate} $\errorRate$ to bound the probability that the solution $\theta^{*}_{\nsamples}$ violates the constraints of problem~\eqref{eq:optProblemExact} and by a \emph{significance level} $\significanceLevel$ with respect to the random sampling solution algorithm.
The minimal number of sampled points $\nsamples$ is related to the error rate $\errorRate \in (0, 1]$ and significance level $\significanceLevel \in (0, 1]$ by: 
\begin{theorem}
\label{thm:PACnumberOfSamples}
    If the optimization problem~\eqref{eq:optProblemPAC} is feasible and it has a unique optimal solution $\theta^{*}_{\nsamples}$, then $P(f_{\omega}(\theta^{*}_{\nsamples}) > 0) < \errorRate$ holds with confidence at least $1 - \significanceLevel$, provided that the number of constraints $\nsamples$ satisfies
    \[
        \nsamples \geq \frac{2}{\errorRate} \cdot \Big(\ln \frac{1}{\significanceLevel} + m \Big),
    \] 
    where $m$ is the dimension of $\theta$, that is, $\theta \in \Theta \subseteq \reals^{m}$, and $\errorRate$ and $\significanceLevel$ are the given error rate and significance level, respectively.
\end{theorem}
The uniqueness assumption about $\theta^{*}_{\nsamples}$ in the theorem above is not restrictive, since for multiple optimal solutions we can just use the Tie-break
rule~\cite{DBLP:journals/tac/CalafioreC06} to get a unique optimal solution.

Statistical guarantees like the one established in Theorem~\ref{thm:PACnumberOfSamples} can be obtained also when comparing how close two functions are on their common domain. 
This, together with the results of Theorem~\ref{thm:PACnumberOfSamples} applied to the optimization problem~\eqref{eq:optProblemPAC}, will be used later for the synthesis of approximation functions in model checking parametric MDPs.

\begin{definition}
\label{def:PACmodel}
    Given a set of $n$ parameters $\parametersSet = \setnocond{\parameter_{1}, \dotsc, \parameter_{n}}$, their valuation domain $\parametersDomain = \prod_{i = 1}^{n} \range(\parameter_{i})$, and a function $f \colon \parametersDomain \to \reals$, let $P$ be a probability measure over $\parametersDomain$ and $\errorRate, \significanceLevel \in (0,1]$ be an error rate and a significance level, respectively.
    
    We say that the polynomial $\approxfun{f} \in \polynomialsSet$ is a $\margin$-PAC approximation of $f$ with $(\errorRate, \significanceLevel)$-guarantee if the inequality $P(\setcond{\parameters \in \parametersDomain}{\abs{\approxfun{f}(\parameters) - f(\parameters)} \leq \margin}) \geq 1 - \errorRate$ holds with confidence $1 - \significanceLevel$, where $\margin \in \posreals$ is used as a margin to bound the PAC approximation error. 
\end{definition}

\begin{markedexample}
\label{ex:marginPACapproximation}
    As an example of $\margin$-PAC approximation, for the cloud-service pMDP shown in Figure~\ref{fig:pmdp} and the PRCTL property $\lsf = \lP[\max]{\mathord{=}?}(\lF \success)$, the function $f_{\lsf}$ representing $\lsf$ is the pointwise maximum of the two rational functions $f_{\lsf}^{\textsf{r}} = \frac{pq}{1-p^{2}+p^{2}q}$ and $f_{\lsf}^{\textsf{m}} = \frac{p^{2}(1-q)}{1-pq+p^{2}q}$ established in Example~\ref{ex:pmdp}.
    A possible linear approximation of this maximum is $\ApproxFunOfProperty{f} = -0.323 + 0.912p + 0.352q$ (where for ease of notation we write the coefficients approximated to the third digit);
    we obtain such a linear function by applying the method we present in Section~\ref{ssec:ourMethod}, which gives us the corresponding margin $\margin = 0.082$ when we assign $\errorRate$ and $\significanceLevel$ to be $\errorRate = \significanceLevel = 0.05$. 
    We can check that the approximation function $\ApproxFunOfProperty{f}$ satisfies the condition of Definition~\ref{def:PACmodel} when we consider $\margin = 0.082$, that is, $\ApproxFunOfProperty{f}$ is a $0.082$-PAC approximation of $f_{\lsf}$ with $(0.05, 0.05)$-guarantee.
\end{markedexample}

In Definition~\ref{def:PACmodel}, for $f_{\lsf}$ and $\ApproxFunOfProperty{f}$ given in Example~\ref{ex:marginPACapproximation}, the margin $\margin$ can be chosen freely in $\posreals$, as long as $\margin \geq 0.082$: 
it is immediate to note that, if $P(\abs{\ApproxFunOfProperty{f}(\parameters) - f_{\lsf}(\parameters)} \leq \margin) \geq 1 - \errorRate$ holds with confidence $1 - \significanceLevel$ for $\margin = 0.082$, then the same holds true for all values of $\margin$ larger than $0.082$.
This means, for instance, that for the same choice of $\errorRate$ and $\significanceLevel$, we can have that $\ApproxFunOfProperty{f}$ is e.g.\@ a $0.1$-PAC approximation of $f_{\lsf}$ with $(\errorRate, \significanceLevel)$-guarantee while it may not be a $0.01$-PAC approximation of $f_{\lsf}$ with $(\errorRate, \significanceLevel)$-guarantee;
this can happen because $\ApproxFunOfProperty{f}$ can make $P(\abs{\ApproxFunOfProperty{f}(\parameters) - f_{\lsf}(\parameters)} \leq 0.01) \geq 1 - \errorRate$ fail, with confidence $1 - \significanceLevel$.
So this means that for the same $f_{\lsf}$, $\ApproxFunOfProperty{f}$, $\errorRate$, and $\significanceLevel$, Definition~\ref{def:PACmodel} can be satisfied or not depending on the value of $\margin$ we are interested in. 
As we will see below in Section~\ref{ssec:ourMethod}, when we synthesize the approximation polynomial for the given function $f_{\lsf}$, we obtain a polynomial $\ApproxFunOfProperty{f}$ together with a value for the margin $\margin$;
both of them are the result of solving an optimization problem (cf.\@ problem~\eqref{eq:PACLPpolynomialApproximation}) determined by the parameters $\errorRate$ and $\significanceLevel$, and a large enough set of points randomly sampled in $\parametersDomain$.
Therefore, the actual value of $\margin$ can be different for different sets of sampled points, as does $\ApproxFunOfProperty{f}$;
still, we can ensure by construction that the obtained polynomial $\ApproxFunOfProperty{f}$ and the associated margin $\margin$ jointly satisfy Definition~\ref{def:PACmodel}, that is, the approximation function $\ApproxFunOfProperty{f}$ obtained by solving problem~\eqref{eq:PACLPpolynomialApproximation} we present below is a $\margin$-PAC approximation of $f_{\lsf}$ with $(\errorRate, \significanceLevel)$-guarantee.

As usual in the PAC approach, we assume that $P$ is the uniform distribution on the domain $\parametersDomain = \prod_{i = 1}^{n} \range(\parameter_{i})$ unless otherwise specified. 
The intuition underlying Definition~\ref{def:PACmodel} is that we want to have statistical guarantees about the fact that the $\margin$-PAC approximation $\approxfun{f}$ is closer to $f$ than the margin $\margin$, on $\parametersDomain$. 
Thus the two statistical parameters $\significanceLevel$ and $\errorRate$ are introduced to quantify how often the difference between $\approxfun{f}$ and $f$ respects the threshold $\margin$. 
Specifically, the error rate $\errorRate$ describes the probability that the difference between the value $f(\parameters)$ and $\approxfun{f}(\parameters)$ is larger than $\margin$ on the sampled points, while the significance level $\significanceLevel$ bounds the risk of wrongly accepting this event; 
these two parameters can be changed to improve the quality of the approximation, that is, to get $\approxfun{f}$ closer to $f$.

\subsection{Synthesizing Parametric Functions}
\label{ssec:ourMethod}

We now recall and adapt the scenario approach for the synthesis of parametric functions to parametric Markov decision processes.
The basic construction was presented in~\cite{DBLP:conf/atva/LiuTHXZ23} for parametric discrete-time Markov chains and then extended to pMDPs in~\cite{DBLP:conf/birthday/ChiLT0J25};
we include the formulation here to fix the notation for later use in the newly introduced DIRECT optimization and the extended experiments.

Let $\pmdrm = (\pmdp, \mreward)$ be a pMDRM whose underlying pMDP is $\pmdp = (\mstates, \minit,  \mactions, \mtransitions, \mlabelling)$; 
let $\parameters$ denote the vector of parameters $(\parameter_{1}, \dotsc, \parameter_{n})$ of $\pmdp$. 
Given the PRCTL state formula $\lsf = \lQ(\lpf)$, the analytic function $f_{\lsf}(\parameters)$, representing the $\opt$imal value of satisfying $\lpf$ in the pMDRM $\pmdrm$, can be a rational function with a very complicated form already in the context of parametric MCs~\cite{DBLP:journals/sttt/HahnHZ11} since the polynomials in these rational functions may have exponentially many terms. 
This means that analyzing them can be time consuming and subject to numerical errors; 
thus, as suggested in~\cite{DBLP:conf/atva/LiuTHXZ23}, we approximate them with some low-degree polynomials while giving statistical guarantees on the approximation error we introduce.

Consider the $\margin$-PAC approximation given in Definition~\ref{def:PACmodel}. 
The approximation procedure developed in~\cite{DBLP:conf/atva/LiuTHXZ23} computes the coefficients $\coefficients$ of the polynomial $\ApproxFunOfProperty{f}(\parameters)$ by solving a scenario optimization problem as defined in Definition~\ref{def:scenarioApproach}. 
The scenario problem~\eqref{eq:optProblemPAC} is obtained by first taking $\nsamples$ samples of the parameters to form the set $\SampleDomain{\parametersDomain}$, which correspond to $\nsamples$ evaluations $\evaluation_{i}$ for $i=1$ to $\nsamples$; 
then the value $\evaluate[\evaluation_{i}]{f_{\lsf}}$ is obtained by solving a standard model checking problem for $\lsf$ on the instance $\evaluate[\evaluation_{i}]{\pmdp}$.
Lastly, the corresponding constrains are produced, with the form $\abs{\evaluate[\evaluation_{i}]{f_{\lsf}} - \coefficients \cdot \evaluate[\evaluation_{i}]{(1, \parameters, \dotsc, \parameters^{d})^{T}}} \leq \margin$, where $\parameters^{j}$ denotes the monomials of degree $j$ of the polynomial with variables $\parameters$, and then the optimization problem is solved with respect to the variables $\coefficients$ and $\margin$ with the aim of minimizing $\margin$.
With this approach, it does not matter how complicated the function $f_{\lsf}$ is: 
in order to construct the approximating polynomial $\ApproxFunOfProperty{f}$ of $f_{\lsf}$ by solving the optimization problem, we do not need the analytical form of $f_{\lsf}$ to compute $\evaluate[\evaluation_{i}]{f_{\lsf}}$; 
we get the corresponding value by applying standard model checking techniques on the instance $\evaluate[\evaluation_{i}]{\pmdp}$.
Since $\ApproxFunOfProperty{f}$ comes with statistical guarantees, provided that the number of samples $\nsamples$ is large enough, we can use it to analyze various properties the original function $f_{\lsf}$ may satisfy.

Given the vector of parameters $\parameters$ and a degree $d \in \naturals$, we denote by $\parameters^{d}$ the vector of monomials $\parameters^{d} = (\parameters^{\vec{\alpha}})_{\norm[1]{\vec{\alpha}} = d}$, where each monomial $\parameters^{\vec{\alpha}}$ is defined as $\parameters^{\vec{\alpha}} = \parameter_{1}^{\alpha_{1}} \parameter_{2}^{\alpha_{2}} \cdots \parameter_{n}^{\alpha_{n}}$, with $\vec{\alpha} = (\alpha_{1}, \dotsc, \alpha_{n}) \in \naturals^{n}$ and $\norm[1]{\vec{\alpha}} = \sum_{i = 1}^{n} \alpha_{i}$.
Then, we associate a vector $\coefficients_{i}$ of coefficients to each of the monomials in the vector $(\parameters^{i})_{i = 0}^{d}$, obtaining the PAC approximation schema $\approxfun{f}(\parameters) = \sum_{i = 0}^{d} \coefficients_{i} \cdot \parameters^{i}$.
In general, for $n$ parameters and a polynomial of degree $d$, we need $\binom{n + d}{n}$ coefficients.

\begin{markedexample}
\label{ex:polynomialTemplate}
    As an example of the polynomial template, consider again the cloud-service recovery pMDP shown in Figure~\ref{fig:pmdp}; 
    it has two parameters $p$ and $q$, so $n = 2$.
    If we use a quadratic polynomial, we have $d = 2$, so we obtain a template with $\binom{n + d}{n} = \binom{2 + 2}{2} = 6$ coefficients, namely, $\approxfun{f}(p, q) = \coefficients_{0} \cdot 1 + \coefficients_{1} \cdot (p, q) + \coefficients_{2} \cdot (p, q)^{2} = c_{0} \cdot 1 + (c_{11} \cdot p + c_{12} \cdot q) + (c_{21} \cdot p^{2} + c_{22} \cdot p \cdot q + c_{23} \cdot q^{2})$.
\end{markedexample}

Given the error rate $\errorRate$ and the significance level $\significanceLevel$, by Theorem~\ref{thm:PACnumberOfSamples} we need only to independently and identically sample at least $\nsamples \geq \frac{2}{\errorRate} \big(\ln \frac{1}{\significanceLevel} + m) = \frac{2}{\errorRate} \big(\ln \frac{1}{\significanceLevel} + \binom{n + d}{n} + 1 \big)$ points $\SampleDomain{\parametersDomain} = \setnocond{\parameters_{i}}_{i=1}^{\nsamples}$, where $m = \binom{n + d}{n} + 1$ is due to the number $\binom{n + d}{n}$ of coefficients in the polynomial, plus $1$ for $\margin$. 
We use these points to generate the constraints of the following relaxed LP problem, which corresponds to the problem~\eqref{eq:optProblemPAC}, where $\coefficients \in \reals^{\binom{n + d}{n}}$:
\begin{equation}
\label{eq:PACLPpolynomialApproximation}
\begin{split}
    \min\limits_{\coefficients, \margin} \quad & 1 \cdot \margin + 0 \cdot \coefficients \\
	\mathrm{s.t.} \quad &  -\margin \leq f_{\lsf}(\parameters_{i}) - \coefficients \cdot (1, \parameters_{i}, \dotsc, \parameters^{d}_{i})^{T} \leq \margin \qquad \forall \parameters_{i} \in \SampleDomain{\parametersDomain}  \\
        & \margin \geq 0
\end{split}
\end{equation}
We then solve the optimization problem~\eqref{eq:PACLPpolynomialApproximation} to get the coefficients $\coefficients$ that allow $\margin$ to be minimal with respect to the constrains given by the sampled points $\SampleDomain{\parametersDomain}$, hence the $\margin$-PAC approximation $\ApproxFunOfProperty{f}$ of the original function $f_{\lsf}$, with the statistical guarantees given by Theorem~\ref{thm:PACnumberOfSamples} and Definition~\ref{def:PACmodel}.

\begin{markedexample}
\label{ex:lpConstraint}
    The constraints of the problem~\eqref{eq:PACLPpolynomialApproximation} relative to the cloud-service recovery pMDP shown in Figure~\ref{fig:pmdp}, $\lsf = \lP[\max]{\mathord{=}?}[\lF \success]$ (maximum probability of successful recovery), and the quadratic polynomial template shown in Example~\ref{ex:polynomialTemplate} are as follows. 
    For the sampled point/evaluation $(p, q) = (0.7, 0.4)$, the constraint $-\margin \leq f_{\lsf}(\parameters_{i}) - \coefficients \cdot (1, \parameters_{i}, \dotsc, \parameters^{d}_{i})^{T} \leq \margin$ gets instantiated to
    \[
        -\margin \leq 0.397 - (c_{0} + 0.7 \cdot c_{11} + 0.4 \cdot c_{12} + 0.49 \cdot c_{21} + 0.28 \cdot c_{22} + 0.16 \cdot c_{23}) \leq \margin,
    \]
    where we use only three digits for the value $0.397$ of $\lsf$ on the MDP induced by $(p, q) = (0.7, 0.4)$.
    All other constrains are similar, with the only difference being the actual numbers appearing in them, such as the value of $f_{\lsf}(\parameters_{i})$ and the multipliers of the LP variables $c_{i}$. 
\end{markedexample}

\subsubsection*{Discussion on $\margin$ and $\ApproxFunOfProperty{f}$ as solution of the problem~\eqref{eq:PACLPpolynomialApproximation}.}
When we instantiate the problem~\eqref{eq:PACLPpolynomialApproximation} by sampling randomly $\nsamples$ points from $\parametersDomain$ and then solve it, we get as solution a pair $(\margin, \coefficients)$, so we can synthesize the approximation function $\ApproxFunOfProperty{f}$ by instantiating the chosen template with $\coefficients$.
If we instantiate another time the problem~\eqref{eq:PACLPpolynomialApproximation} by sampling randomly again $\nsamples$ points, we get another pair $(\margin', \coefficients')$ that is likely different from $(\margin, \coefficients)$, and so the corresponding function $\ApproxFunOfProperty[\margin']{f}'$ is also different from $\ApproxFunOfProperty{f}$.
This means that by instantiating and solving the problem~\eqref{eq:PACLPpolynomialApproximation} multiple times, we might get different outcomes and one may wonder whether this may cause some problem and whether there are differences in the properties one can derive by using e.g.\@ $\ApproxFunOfProperty[\margin']{f}'$ instead of $\ApproxFunOfProperty{f}$.
As stated in~\cite{DBLP:journals/siamjo/CampiG08,DBLP:journals/arc/CampiGP09}, the margin $\margin$ is actually a random variable, since its value depends on the actual points that have been sampled to construct the problem~\eqref{eq:PACLPpolynomialApproximation}; 
by a similar reasoning, also $\ApproxFunOfProperty{f}$ can be seen as a random variable.
This means that, for $\margin$ and $\ApproxFunOfProperty{f}$ as computed in the first instance of the problem~\eqref{eq:PACLPpolynomialApproximation}, for a point $\parameters \in \parametersDomain \setminus \SampleDomain{\parametersDomain}$, we can have that $\abs{f_{\lsf}(\parameters) - \ApproxFunOfProperty{f}(\parameters)} \leq \margin$ holds, but we can also have the opposite, that is, $\abs{f_{\lsf}(\parameters) - \ApproxFunOfProperty{f}(\parameters)} > \margin$.
However, Theorem~\ref{thm:PACnumberOfSamples} ensures that, if the number of samples $\nsamples$ is large enough, so to respect $\nsamples \geq \frac{2}{\errorRate} \big(\ln \frac{1}{\significanceLevel} + \binom{n + d}{n} + 1 \big)$, then the latter situation occurs with probability at most $\errorRate$, with confidence $1 - \significanceLevel$.
The same happens when we consider $\margin'$ and $\ApproxFunOfProperty[\margin']{f}'$ from the second instance of the problem~\eqref{eq:PACLPpolynomialApproximation}.
As suggested in~\cite{DBLP:journals/arc/CampiGP09}, we can take a very small value for $\significanceLevel$, like $10^{-10}$ or $10^{-20}$: 
this makes the probability of violating $\abs{f_{\lsf}(\parameters) - \ApproxFunOfProperty{f}(\parameters)} \leq \margin$ by an $\errorRate$-fraction of $\parametersDomain$ negligible while only slightly increasing $\nsamples$, since by Theorem~\ref{thm:PACnumberOfSamples} the significance $\significanceLevel$ only provides a logarithmic contribution to the growth of $\nsamples$.

\begin{markedexample}
    For the repair-policy function $f_{\lsf}^{\textsf{r}}(p,q)$ on $\parametersDomain = \interval{0.2}{0.9} \times \interval{0.2}{0.8}$, we instantiated the LP problem~\eqref{eq:PACLPpolynomialApproximation} twice with different random seeds.
    With the linear template and 280 samples, the instance based on the first seed gives
    \[
        \ApproxFunOfProperty{f}(p,q)
        =
        -0.375 + 0.821p + 0.599q,
        \qquad
        \margin = 0.090,
    \]
    while with the second seed we obtain
    \[
        \ApproxFunOfProperty[\margin']{f}'(p,q)
        =
        -0.325 + 0.788p + 0.540q,
        \qquad
        \margin' = 0.098.
    \]
    With the quadratic template and 400 samples, the first seed induces
    \[
        \ApproxFunOfProperty{f}(p,q) = -0.077 -0.380p +0.634q +0.831p^{2}+0.655pq-0.508q^{2},
        \qquad
        \margin = 0.034,
    \]
    while with the second seed we get
    \[
        \ApproxFunOfProperty[\margin']{f}'(p,q)
        = -0.124 -0.311p +0.807q +0.773p^{2}+0.664pq-0.698q^{2},
        \qquad
        \margin' = 0.038.
    \]
    These concrete runs illustrate that both the coefficients and the optimized margin are random outcomes of the sampled scenario problem.
    However, each synthesized pair still carries its own PAC guarantee: 
    outside an $\errorRate$-fraction of $\parametersDomain$, with confidence $1-\significanceLevel$, the corresponding polynomial is within its own computed margin of $f_{\lsf}^{\textsf{r}}$.
\end{markedexample}

\subsection{PRCTL Property Analysis for pMDRM}
\label{subsec:reachabilityPMRMs}

Given the probabilistic formula $\lsf = \lP{=?}(\lpf)$ with path formula $\lpf$,  we can obviously use the PAC approximation $\ApproxFunOfProperty{f}$ to check whether the domain of parameters $\parametersDomain$ is safe, with PAC guarantee.
In this section, we show how a direct PAC based approach can be used to check the safety of the domain, without having to learn the approximations first. 
Then, we consider linear approximations and discuss how counterexamples can be generated in this case before showing how the polynomial PAC approximation $\ApproxFunOfProperty{f}$ can be used to analyze global properties of $f_{\lsf}$ over the whole parameter space $\parametersDomain$.
Lastly, we present how to extend the approach to the reward formula $\lsf = \lER{=?}(\lF \lsf')$.

We use the following definition of safe regions.
\begin{definition}
\label{def:safeRegion}
    Given the domain of parameters $\parametersDomain$, a function $f \colon \parametersDomain \to \posreals$, a safety level $\safetyLevel \in \posreals$, and a comparison $\mathord{\bowtie} \in \setnocond{<, >}$, we say that the point $\parameters \in \parametersDomain$ is \emph{safe} if and only if $f(\parameters) \bowtie \safetyLevel$; 
    we call $\parametersDomain$ safe if and only if each $\parameters \in \parametersDomain$ is safe.
\end{definition}

If we want to analyze whether a pMDRM $\pmdrm$ is safe, i.e., whether it can never make a state property $\lsf = \lQ$ exceed a given safety level $\safetyLevel$, we can either compute $f_{\lsf}$ and then compare it with $\safetyLevel$ on $\parametersDomain$, or we can use the following optimization problem that gives us the answer with a $(\errorRate, \significanceLevel)$-statistical guarantee, provided that we take at least $\nsamples \geq \frac{2}{\errorRate} \cdot (\ln \frac{1}{\significanceLevel} + 2)$ samples to form $\SampleDomain{\parametersDomain} \subseteq \parametersDomain$:
This direct safety check is the degree-zero instance of the scenario-based verification approach for uncertain parametric MDPs by Badings et~al.~\cite{DBLP:journals/sttt/BadingsCJJKT22}: 
it samples parameter valuations and optimizes a constant bound on the sampled property values.
We recall it here to make the later comparison with learned linear and polynomial approximations self-contained.
\begin{equation}
\label{eq:safeRegion}
\begin{split}
    \opt_{\bowtie} \quad & \margin \\
    \mathrm{s.t.} \quad & f_{\lsf}(\parameters) \bowtie \margin \qquad \forall \parameters \in \SampleDomain{\parametersDomain},
\end{split}
\end{equation}
where $\opt_{\bowtie} = \min$ and $\mathord{\bowtie} = \mathord{<}$ if $\lsf = \lQ[\max]$ and $\opt_{\bowtie} = \max$ and $\mathord{\bowtie} = \mathord{>}$ if $\lsf = \lQ[\min]$.
Here, the idea is that $\pmdrm$ is considered safe if, for instance, its maximum probability of reaching some bad state is at most $\margin$, or the minimum expected reward before reaching failure is at least $\margin$. 
Since we use polynomials of degree $0$ in~\eqref{eq:safeRegion}, the expression $\nsamples \geq \frac{2}{\errorRate} \cdot (\ln \frac{1}{\significanceLevel} + m )$ in Theorem~\ref{thm:PACnumberOfSamples} becomes $\nsamples \geq \frac{2}{\errorRate} \cdot (\ln \frac{1}{\significanceLevel} + 2)$ since $m = \binom{n + d}{n} + 1 = \binom{n + 0}{n} + 1 = 2$.

The optimization problem~\eqref{eq:safeRegion} can be solved in time $\bigO(\size{\SampleDomain{\parametersDomain}})$, since it only needs to compute the maximum/minimum value of $f_{\lsf}(\parameters)$, depending on $\mathord{\bowtie} \in \setnocond{<, >}$, for $\parameters \in \SampleDomain{\parametersDomain}$ as the optimal solution $\margin^{*}$.  
Although the calculation is very simple, polynomials with degree 0, i.e., constants, also have good probability and statistical meaning, so we have the following result as a direct consequence of the definitions:
\begin{lemma}
\label{lem:zeroApproximation}
    Given the safety level $\safetyLevel$, if the optimal solution $\margin^{*}$ of the problem~\eqref{eq:safeRegion} satisfies $\margin^{*} \bowtie \safetyLevel$, then the domain $\parametersDomain$ is safe with $(\errorRate, \significanceLevel)$-guarantee.
    Otherwise, if $\margin^{*} \not\bowtie \safetyLevel$, then the parameter point $\parameters^{*} \in \SampleDomain{\parametersDomain}$ corresponding to $\margin^{*}$ is unsafe.
\end{lemma}

By Lemma~\ref{lem:zeroApproximation}, we can analyze with $(\errorRate, \significanceLevel)$-guarantee whether the parameter space is safe or not. 
For example, consider the cloud-service pMDP shown in Figure~\ref{fig:pmdp} and the safety property $\lP[\max]{< 0.9}(\lF \success)$, asking whether the maximum probability of successful recovery stays below the safety level $0.9$ over the parameter domain.
If we set $\errorRate = \significanceLevel = 0.05$, by sampling in the region $\parametersDomain = \interval{0.2}{0.9} \times \interval{0.2}{0.8}$ at least 200 points and solving the resulting optimization problem~\eqref{eq:safeRegion}, a representative sampling gives the optimum value $\margin^{*} \approx 0.856$, attained at about $(p,q) = (0.899,0.789)$, which is below the safety level.
Since $\margin^{*} < 0.9$, by Lemma~\ref{lem:zeroApproximation}, the region $\parametersDomain$ is safe with $(0.05, 0.05)$-guarantee.
Since we have the functions representing $\lP[\max]{=?}(\lF \success)$ from Example~\ref{ex:pmdp}, and they are not too complex, we can also directly inspect them:
the maximum of $\max\setnocond{f_{\lsf}^{\textsf{r}}, f_{\lsf}^{\textsf{m}}}$ over $\parametersDomain$ is attained at $(p,q) = (0.9,0.8)$ and is roughly $0.859$.
Thus the domain is in fact safe for the threshold $0.9$; 
the scenario argument above gives the corresponding PAC-style certificate without relying on this hand calculation.

\subsubsection{Linear PAC Approximation and Counterexamples.}
\label{sssec:1degree}

Since constants can approximate the maximum value of the function $f_{\lsf}$ with the given $(\errorRate, \significanceLevel)$-PAC guarantee, linear functions can also be used to approximate $f_{\lsf}$, which are more precise than constants.
In other words, by using a linear approximation $\ApproxFunOfProperty{f}$ instead of a constant one, as done above, by taking $\nsamples \geq \frac{2}{\errorRate} \cdot (\ln \frac{1}{\significanceLevel} + n + 2)$ samples, we can verify with $(\errorRate, \significanceLevel)$-statistical guarantee whether the function $f_{\lsf}$ is respecting the safety level $\safetyLevel$ on the whole domain $\parametersDomain$, with additional information about where possible counterexamples can be found.
More precisely, since it is easy to find the extreme values for a linear function, by e.g.\@ solving a linear programming problem, we can inspect them to evaluate whether the function $f_{\lsf}$ exceeds $\safetyLevel$, by computing the value $f_{\lsf}(\parameters^{*})$ on the parameters $\parameters^{*}$ corresponding to the extreme values.
Then, we know whether $\parameters^{*}$ is a real or a spurious counterexample; 
in the latter case, we can just add it to the set $\SampleDomain{\parametersDomain}$ and improve the computation of $\ApproxFunOfProperty{f}$.
This approach gives us the following result:
\begin{lemma}
\label{lem:linearApproximation}
    Given a pMDRM $\pmdrm$ and a PRCTL property $\lsf = \lQ$, let $P$ be a probability measure over $\parametersDomain$ and $\ApproxFunOfProperty{f}$ be a linear $\margin$-PAC approximation of $f_{\lsf}$ with $(\errorRate, \significanceLevel)$-guarantee.
    Let $\mathord{\pm} = \mathord{+}$ and $\mathord{\lessgtr} = \mathord{<}$ if $\lsf = \lQ[\max]$, and symmetrically $\mathord{\pm} = \mathord{-}$ and $\mathord{\lessgtr} = \mathord{>}$ if $\lsf = \lQ[\min]$; 
    let $\mp$ and $\gtrless$ denote the opposite of $\pm$ and $\lessgtr$, respectively.
    
    Given the safety level $\safetyLevel \in \posreals$, if for each $\parameters \in \parametersDomain$ we have $\ApproxFunOfProperty{f}(\parameters) \pm \margin \lessgtr \safetyLevel$, then $P(f_{\lsf}(\parameters) \lessgtr \safetyLevel) \geq 1 - \errorRate$ holds with confidence $1 - \significanceLevel$. 
    On the other hand, if $P(\ApproxFunOfProperty{f}(\parameters) \mp \margin \gtrless \safetyLevel) > \errorRate$, then there exists $\parameters \in \parametersDomain$ such that $f_{\lsf}(\parameters) \gtrless \safetyLevel$ holds with confidence $1 - \significanceLevel$.
\end{lemma}
\begin{proof}
    In the proof below, we consider the case $\lsf = \lQ[\max]$, so $\mathord{\pm} = \mathord{+}$ and $\mathord{\lessgtr} = \mathord{<}$.
    The other case $\lsf = \lQ[\min]$ is just symmetrical.
    
    To prove the statement of the lemma, consider on the one hand that, if the condition $\ApproxFunOfProperty{f}(\parameters) + \margin < \safetyLevel$ holds for each $\parameters \in \parametersDomain$, then we have 
    \begin{align*}
        P(f_{\lsf}(\parameters) < \safetyLevel) 
        & \geq P(f_{\lsf}(\parameters) < \ApproxFunOfProperty{f}(\parameters) + \margin)\\
        & {} \geq P(f_{\lsf}(\parameters) - \ApproxFunOfProperty{f}(\parameters) < \margin) \\
        & {} \geq P(\abs{f(\parameters) - \ApproxFunOfProperty{f}(\parameters)} < \margin).
    \end{align*}
    By the definition of PAC approximation, it follows that $P(f_{\lsf}(\parameters) < \safetyLevel) \geq 1 - \errorRate$ so the parameters space $\parametersDomain$ is safe with confidence $1 - \significanceLevel$. 
    
    On the other hand, we first assume that for each $\parameters \in \parametersDomain$, the condition $f_{\lsf}(\parameters) < \safetyLevel$ holds. 
    Since $\ApproxFunOfProperty{f}$ is a PAC approximation of $f_{\lsf}$ with $(\errorRate, \significanceLevel)$-guarantee, this implies that $P(\abs{\ApproxFunOfProperty{f}(\parameters) - f_{\lsf}(\parameters)} \leq \margin) \geq 1 - \errorRate$ according to Definition~\ref{def:PACmodel}. 
    Moreover, we have $P(f_{\lsf}(\parameters) - \margin \leq \ApproxFunOfProperty{f}(\parameters) \leq f_{\lsf}(\parameters) + \margin) \geq 1 - \errorRate$.
    Therefore, $P(\ApproxFunOfProperty{f}(\parameters) \leq f_{\lsf}(\parameters) + \margin) \geq 1 - \errorRate$ holds. 
    Since $f_{\lsf}(\parameters) < \safetyLevel$, this implies that $P(\ApproxFunOfProperty{f}(\parameters) \leq \safetyLevel + \margin) \geq 1 - \errorRate$, which is equivalent to the inequality $P(\ApproxFunOfProperty{f}(\parameters) > \safetyLevel + \margin) < \errorRate$.
    This contradicts the assumption ``$P(\ApproxFunOfProperty{f}(\parameters) -\margin > \safetyLevel) > \errorRate$'' in the statement of the lemma, thus the condition we assumed ``$\forall \parameters \in \parametersDomain$, the condition $f_{\lsf}(\parameters) < \safetyLevel$ holds'' cannot be true. 
    From this we derive that there exists a point $\parameters \in \parametersDomain$ such that $f_{\lsf}(\parameters) > \safetyLevel$ with confidence $1 - \significanceLevel$, i.e., the domain of parameters $\parametersDomain$ is unsafe, as desired.
\end{proof}

\begin{figure}[tb]
    \centering
    \resizebox{\linewidth}{!}{
        \includegraphics{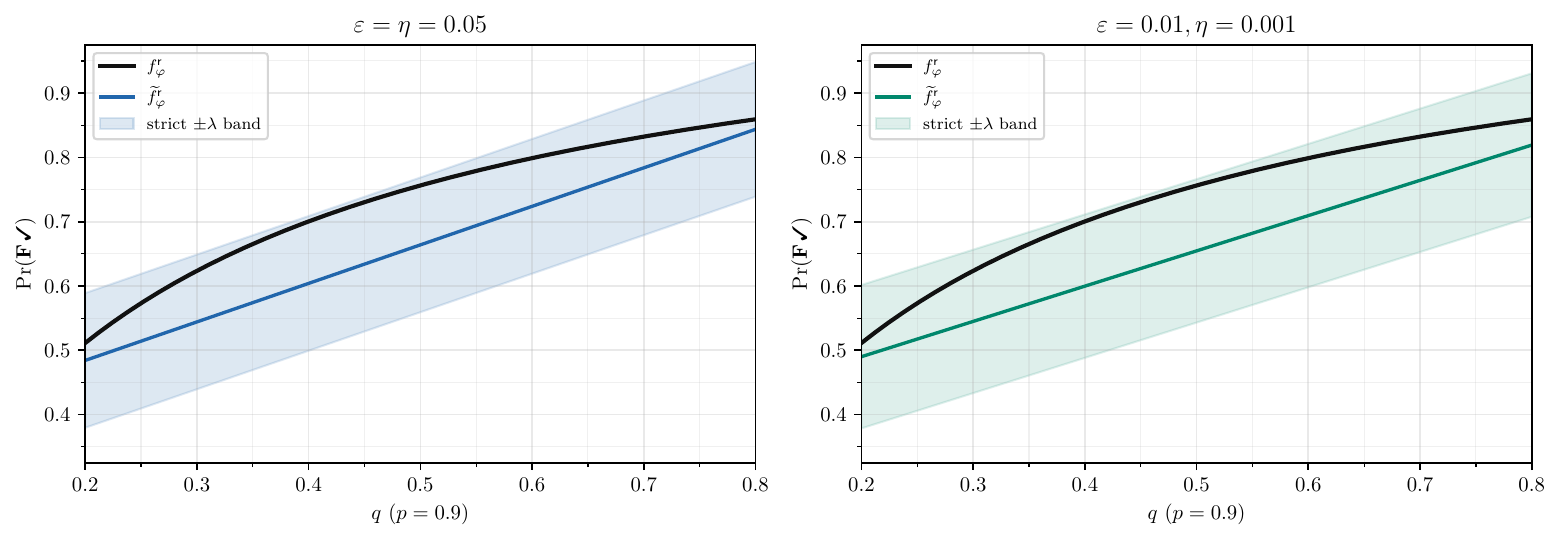}
    }
    \caption{The rational function $f_{\lsf}^{\textsf{r}}(p, q) = \frac{pq}{1-p^{2}+p^{2}q}$ and its linear approximations $\ApproxFunOfProperty{f}(p, q)$ for two choices of $\errorRate$ and $\significanceLevel$; the plots are shown for $p = 0.9$, to achieve the maximum value of $f_{\lsf}^{\textsf{r}}(p, q)$ (cf.~Figure~\ref{fig:pmdp1MultipleFunctions})}
    \label{fig:linearApproximation}
\end{figure}

\begin{markedexample}
\label{ex:linearApproximation}
    The plots in Figure~\ref{fig:linearApproximation} show the results of applying a linear PAC approximation on the policy-specific function $f_{\lsf}^{\textsf{r}}(p, q) = \frac{pq}{1-p^{2}+p^{2}q}$ (cf.~Example~\ref{ex:pmdp}), where $\lsf = \lP{=?}(\lF \success)$, induced by choosing the repair action in the cloud-service pMDP shown in Figure~\ref{fig:pmdp}.
    We sampled 280 points for $\errorRate = \significanceLevel = 0.05$ and 2182 points for $\errorRate = 0.01$ and $\significanceLevel = 0.001$, respectively, according to Theorem~\ref{thm:PACnumberOfSamples}. 
    The plot on the left uses the first statistical choice, while the plot on the right uses the stronger statistical choice; 
    in both plots, we fix the parameter $p = 0.9$ and show the approximation band determined by the computed margin $\margin$.
    For the case $\errorRate = \significanceLevel = 0.05$, the linear approximation is $\ApproxFunOfProperty{f}(p, q) = -0.375 + 0.821p + 0.599q$ with $\margin = 0.090$ by rounding the coefficients to three decimals. 
    We can easily check that for each $(p,q) \in \parametersDomain$ we have $\ApproxFunOfProperty{f}(p,q) + \margin < 0.94$ by linear programming, so $\parametersDomain = \interval{0.2}{0.9} \times \interval{0.2}{0.8}$ is a $0.94$-safe region with respect to $f_{\lsf}^{\textsf{r}}(p,q)$ with $(0.05,0.05)$-guarantee. 
    On the other hand, if we set $\safetyLevel = 0.60$, we can prove that $P(\ApproxFunOfProperty{f} - \margin > \safetyLevel) \approx 0.057 > \errorRate = 0.05$, so by Lemma~\ref{lem:linearApproximation} we get that there exists an unsafe region such that $f_{\lsf}(p,q) > \safetyLevel$, with confidence $95\%$.
    
    We can take advantage of the easy computation of linear programming with linear functions to further search for potential counterexamples that may exist.
    The maximum value of $\ApproxFunOfProperty{f}$ can be found at $(0.9, 0.8)$, according to the linearity of $\ApproxFunOfProperty{f}$, so we can instantiate the cloud-service pMDP in Figure~\ref{fig:pmdp} with the parameter point $(0.9, 0.8)$ to get that $f_{\lsf}^{\textsf{r}}(p,q) \approx 0.859$. 
    Since $f_{\lsf}^{\textsf{r}}(p,q) > 0.60$ for the safety level $\safetyLevel = 0.60$, we can claim that the \emph{real counterexample} $(0.9, 0.8)$ is found.
    In the case that the parameter point $\parameters_{0} = (p,q)$ corresponding to maximum value of $\ApproxFunOfProperty{f}$ is a spurious counterexample for the pMDP with respect to $\lsf$, we can learn a more precise approximation by adding $\parameters_{0}$ to $\SampleDomain{\parametersDomain}$. 
    One may also divide the domain $\parametersDomain$ into several subdomains and analyze each of them separately. 
\end{markedexample}

As for the computational complexity, it is easy to find the maximum value of a linear function by linear programming; 
on the other hand, computing the maximum value of polynomials and rational functions is rather difficult if their degree is very high or the dimension of the parameter space is too large.  
So a linear function is a good alternative to compute the maximum value of $f_{\lsf}$ with PAC guarantee, while polynomials are suitable for analyzing more complicated properties, such as the global ones considered below.

\subsubsection{Polynomial PAC Approximation.}
\label{sssec:polynomialApproximation}

So far, we have seen how to use an approximate polynomial $\ApproxFunOfProperty{f}$ of degree $0$ and $1$ in analyzing properties of the rational function $f_{\lsf}$ for a pMDRM with respect to a state property $\lsf = \lQ$.
In general, one advantage of polynomials over rational functions is that they make it easy to compute complex operations such as inner product and integral~\cite{rudin1976principles}, as needed to evaluate e.g.\@ the $L_{p}$ norm $\norm[p]{g} = \sqrt[p]{\int_{Z} \abs{g(z)}^{p}\,dz}$ of a function $g \colon Z \to \reals$, with $p \geq 1$.
The following lemma links the value of the $L_{p}$ norm of $g$ and the one of its approximation $\approxfun{g}$:
\begin{lemma}
\label{lem:PACnormApproximation}
    Let $Z \subseteq \reals^{n}$ have finite measure, $g \colon Z \to \reals$ be a measurable function, and $\approxfun{g}$ be a polynomial $\margin$-PAC approximation of $g$ with $(\errorRate, \significanceLevel)$-guarantee with respect to the uniform distribution on $Z$.
    Assume moreover that there is a computable global bound $\Gamma \geq \margin$ such that $\abs{g(z) - \approxfun{g}(z)} \leq \Gamma$ for every $z \in Z$.
    Then, for each $p \geq 1$, with confidence $1-\significanceLevel$,
    $\abs{\norm[p]{g} - \norm[p]{\approxfun{g}}}
        \leq
        \sqrt[p]{\size{Z} \big((1-\errorRate) \margin^{p} + \errorRate \Gamma^{p} \big)}$.
    In particular, if $\Gamma = c\margin$ for some $c \geq 1$, then
        $\abs{\norm[p]{g} - \norm[p]{\approxfun{g}}}
        \leq
        \margin \sqrt[p]{\size{Z} \big((1-\errorRate) + \errorRate c^{p}\big)}$.
\end{lemma}
\begin{proof}
    By the reverse triangle inequality for the $L_{p}$ norm, we have that $\abs{\norm[p]{g} - \norm[p]{\approxfun{g}}} \leq \norm[p]{g-\approxfun{g}}$.
    With confidence $1-\significanceLevel$, we can partition the set $Z$ into two sets $G$ and $B$ of good and bad points, respectively, such that $G = \setcond{z \in Z}{\abs{g(z)-\approxfun{g}(z)} \leq \margin}$ and $B = \setcond{z \in Z}{\abs{g(z)-\approxfun{g}(z)} > \margin}$;
    the PAC guarantee ensures that the measure of $B$ is at most $\errorRate \size{Z}$.
    We now split the $L_{p}$ error integral according to $Z = G \cup B$:
    \[
        \norm[p]{g - \approxfun{g}}^{p}
        =
        \int_{G} \abs{g(z) - \approxfun{g}(z)}^{p}\,dz
        +
        \int_{B} \abs{g(z) - \approxfun{g}(z)}^{p}\,dz.
    \]
    On the good set $G$, the PAC error bound gives $\abs{g(z) - \approxfun{g}(z)}^{p} \leq \margin^{p}$.
    On the bad set $B$, the PAC error bound may fail, so we use the known global bound $\Gamma^{p}$ instead.
    Since $\size{B} \leq \errorRate \size{Z}$ and $\Gamma \geq \margin$, the right-hand side is the largest when $\size{B} = \errorRate\size{Z}$, hence
        $\norm[p]{g - \approxfun{g}}^{p}
        \leq
        (1-\errorRate) \size{Z} \margin^{p}
        +
        \errorRate \size{Z} \Gamma^{p}$.
    Taking the $p$-th root establishes the first inequality in the statement of the lemma, and substituting $\Gamma = c\margin$ gives the second one.
\end{proof}
We can thus use higher-degree polynomials to better approximate $f_{\lsf}$ and get more information about its behavior.
For instance, we can check whether $f_{\lsf}$ is within distance $\safetyLevel$ from a desired target value $\targetValue$ with respect to a given norm $\norm[p]{\functionDot}$.
This is useful, for instance, to evaluate how much the behavior of $\pmdrm$ with respect to the property $\lsf$ is affected by the variations of the parameters.
We can model this situation as follows:
\begin{definition}
\label{def:fluctuations}
    Given the domain $\parametersDomain$ of a set of parameters, a function $f_{\lsf} \colon \parametersDomain \to \posreals$, a safety level $\safetyLevel$, and a target value $\targetValue \in \posreals$, we say that \emph{$f$ is near $\targetValue$ within the safety level $\safetyLevel$} on $\parametersDomain$ with respect to the $L_{p}$ norm, if $\norm[p]{f - \targetValue} < \safetyLevel$.
\end{definition}
To verify the above property, we can rely on the following result:
\begin{lemma}
\label{lem:nearMargin}
    Given the domain $\parametersDomain$ of a set of parameters, a function $f_{\lsf} \colon \parametersDomain \to \posreals$, a safety level $\safetyLevel$, and a target value $\targetValue \in \posreals$, let $\size{\parametersDomain} = \int_{\parametersDomain} 1\,d\parameters$, $M$ be an upper bound of $f_{\lsf}(\parametersDomain)$, and $\ApproxFunOfProperty{f}$ be a $\margin$-PAC approximation of $f_{\lsf}$ with $(\errorRate, \significanceLevel)$-guarantee.
    For each $p \geq 1$, if $\ApproxFunOfProperty{f}$ satisfies the condition 
    \begin{equation}
    \label{eq:polyUpperbound}
        \sqrt[p]{ \left( \margin \sqrt[p]{(1 - \errorRate) \cdot \size{\parametersDomain} } + \norm[p]{\ApproxFunOfProperty{f} - \targetValue} \right)^{p} + \errorRate \cdot \size{\parametersDomain} \cdot \max(\abs{M - \targetValue}^{p}, \targetValue^{p})} < \safetyLevel
    \end{equation}
    then $\norm[p]{f_{\lsf} - \targetValue} < \safetyLevel$ holds with confidence $1 - \significanceLevel$.
\end{lemma}
\begin{proof}
    Similarly to the proof of Lemma~\ref{lem:PACnormApproximation}, with confidence $1-\significanceLevel$, the PAC guarantee gives a measurable bad set $B_{0}$ of measure at most $\errorRate \size{\parametersDomain}$ on which the error may exceed $\margin$.
    Since $\parametersDomain$ is a box in our applications, we can enlarge this bad set, if needed, to a measurable set $B \supseteq B_{0}$ with $\size{B} = \errorRate \size{\parametersDomain}$ and define $G = \parametersDomain \setminus B$.
    Then $\size{G} = (1-\errorRate) \size{\parametersDomain}$ and $\abs{f_{\lsf}(\parameters) - \ApproxFunOfProperty{f}(\parameters)} \leq \margin$ for every $\parameters \in G$.
    In the following sequence of (in)equalities, we motivate between them how to obtain the next term in the sequence.
    \begin{align*}
        & \norm[p]{f_{\lsf} - \targetValue} \\
        \intertext{by definition of the $L_{p}$ norm we get}
        = & \sqrt[p]{\int_{\parametersDomain} \abs{f_{\lsf}(\parameters) - \targetValue}^{p}\,d\parameters}\\
        \intertext{and by splitting the integral region into two parts we obtain}
        = & \sqrt[p]{\int_{G} \abs{f_{\lsf}(\parameters) - \targetValue}^{p}\,d\parameters + \int_{B} \abs{f_{\lsf}(\parameters) - \targetValue}^{p}\,d\parameters}  \\
        \intertext{where $\parametersDomain = G \uplus B$ and $G$ is chosen as above, so that $\size{G}=(1-\errorRate)\size{\parametersDomain}$ and for each $\parameters \in G$, we have $\abs{f_{\lsf}(\parameters) - \ApproxFunOfProperty{f}(\parameters)} \leq \margin$. 
        Then by the known triangular inequality of the $L_{p}$ norm, we derive}
        \leq & \sqrt[p]{\left( \sqrt[p]{\int_{G} \abs{f_{\lsf}(\parameters) - \ApproxFunOfProperty{f}(\parameters)}^{p}\,d\parameters} + \sqrt[p]{\int_{G} \abs{\ApproxFunOfProperty{f}(\parameters) - \targetValue}^{p}\,d\parameters} \right)^{p} + \int_{B} \abs{f_{\lsf}(\parameters) - \targetValue}^{p}\,d\parameters}\\
        \intertext{and the construction of $G$ from the PAC-good set implies}
        \leq & \sqrt[p]{\left(\sqrt[p]{(1 - \errorRate) \size{\parametersDomain} \margin^{p}} + \sqrt[p]{\int_{\parametersDomain} \abs{\ApproxFunOfProperty{f}(\parameters) - \targetValue}^{p}\,d\parameters} \right)^{p} + \int_{B} \abs{f_{\lsf}(\parameters) - \targetValue}^{p}\,d\parameters}\\
        \intertext{The first term is obtained by taking $\margin^{p}$ out of the $p$-th root, the second integral over $G$ is upper-bounded by the integral over all of $\parametersDomain$, and the last integral is bounded using $\size{B} = \errorRate\size{\parametersDomain}$ and $0 \leq f_{\lsf}(\parameters) \leq M$.}
        \leq & \sqrt[p]{\left(\margin \sqrt[p]{(1 - \errorRate) \cdot \size{\parametersDomain} } + \norm[p]{\ApproxFunOfProperty{f} - \targetValue} \right)^{p} + \errorRate \cdot \size{\parametersDomain} \cdot \max(\abs{M - \targetValue}^{p}, \targetValue^{p})}. 
    \end{align*}
    Since by the lemma assumption we have that 
    \[
        \sqrt[p]{\left(\margin \sqrt[p]{(1 - \errorRate) \cdot \size{\parametersDomain} } + \norm[p]{\ApproxFunOfProperty{f} - \targetValue} \right)^{p} + \errorRate \cdot \size{\parametersDomain} \cdot \max(\abs{M - \targetValue}^{p}, \targetValue^{p})} < \safetyLevel
    \] 
    holds, it follows that the property $\norm[p]{f_{\lsf} - \targetValue} < \safetyLevel$ is satisfied as well, with confidence $1-\significanceLevel$.
\end{proof}

\begin{markedexample}
\label{ex:probabilityLimited}
    Consider again the repair policy in the cloud-service pMDP shown in Figure~\ref{fig:pmdp} and $\lsf = \lP{=?}(\lF \success)$; 
    since $f_{\lsf}$ represents probabilities, we have the well-known upper bound $M = 1$.
    Here we consider the $L_{2}$ norm, which is widely used in describing the error between functions in the signal processing field (see, e.g.,~\cite{boggess2015first,conway2019course}), as it can reflect the global approximation properties and is easy to compute. 
    To simplify the notation, let $\upperboundFXBeta$ denote the complex expression occurring in the formula~\eqref{eq:polyUpperbound}, that is:
    \[
        \upperboundFXBeta(\ApproxFunOfProperty{f}, \parametersDomain, \targetValue) = \sqrt{ \left(\margin \sqrt{(1 - \errorRate) \cdot \size{\parametersDomain}} + \norm[2]{\ApproxFunOfProperty{f} - \targetValue}\right)^{2} + \errorRate \cdot \size{\parametersDomain} \cdot \max(\abs{1-\targetValue}^{2}, \targetValue^{2})}.
    \] 
    We want to know whether $f_{\lsf}^{\textsf{r}}(p, q) = \frac{pq}{1-p^{2}+p^{2}q}$ is near $0.45$ within $0.20$, i.e., given the safety level $\safetyLevel = 0.20$, we want to check $\norm[2]{f_{\lsf}^{\textsf{r}} - 0.45} < 0.20$.
    According to Lemma~\ref{lem:nearMargin}, we first compute a PAC approximation $\ApproxFunOfProperty{f}$ of $f_{\lsf}$. 
    By setting $\errorRate = \significanceLevel = 0.05$, we get the quadratic polynomial 
    $\ApproxFunOfProperty{f}(p, q) = -0.073 - 0.405p + 0.623q + 0.873p^{2} + 0.603pq - 0.463q^{2}$, by rounding to three decimals. 
    In this case, we get $\upperboundFXBeta(\ApproxFunOfProperty{f}, \parametersDomain, \targetValue) = 0.183 < \safetyLevel = 0.20$, so Lemma~\ref{lem:nearMargin} applies. 
    If, instead, we would have chosen $\safetyLevel' = 0.18$, then we cannot prove $\norm[2]{f_{\lsf}^{\textsf{r}} - 0.45} < 0.18$ by relying on Lemma~\ref{lem:nearMargin}. 
    To do so, we need to consider more samples or a higher-degree approximation in order to decrease the approximation margin and the norm term in~\eqref{eq:polyUpperbound}.
\end{markedexample}

The condition about $M$ in the statement of Lemma~\ref{lem:nearMargin} makes the lemma always applicable to the rational function $f_{\lsf}$ for the property $\lsf = \lP{\mathord{=}?}[\lpf]$, since such a function represents the probability of satisfying $\lpf$, so $M = 1$ is a proper upper bound for $f_{\lsf}$ that can be used independently of the underlying pMDP we are analyzing.
Instead, when we consider the property $\lsf = \lER{\mathord{=}?}[\lF \lsf']$, we have no suitable value for the upper bound $M$ that we can use in all cases.
In fact, by definition $\lER{\mathord{=}?}[\lF \lsf'] = \infty$ if $\lsf'$ cannot be satisfied almost surely under the $\opt$imizing policy.
Moreover, even if we assume that $\lP{\mathord{=}1}[\lF \lsf'] = 1$, then we have no value of $M$ that we can use in all cases.

\begin{markedexample}
\label{ex:rewardNotLimited}
    Consider again the pMDP shown in Figure~\ref{fig:pmdp} and the reward structure $\mreward = (\mreward_{\mstates}, \mreward_{\mactions})$ defined as
    \[
        \mreward_{\mstates}(s) = 
        \begin{cases}
            1000 & \text{if $s = s_{3}$} \\
            0 & \text{otherwise}
        \end{cases}
        \quad \text{and} \quad
        \mreward_{\mactions}(s, a) = 
        \begin{cases}
            10 & \text{if $(s,a) = (s_{1}, \mathsf{m})$} \\
            3 & \text{if $(s,a) = (s_{1}, \mathsf{r})$} \\
            0 & \text{otherwise}
        \end{cases}
    \]
    already seen in Example~\ref{ex:rewardStructure}. 
    For the reward property $\lsf = \lER{\mathord{=}?}[\lF (\success \lor \failure)]$, we have that it holds almost surely and the two rational functions corresponding to choosing \textsf{r} and \textsf{m} in $s_{1}$ are $f_{\lsf}^{\textsf{r}} = \frac{1000(1 - pq - p^{2} + p^{2}q) + 3p}{1 - p^{2} + p^{2}q}$ and $f_{\lsf}^{\textsf{m}} = \frac{1000(1 - pq - p^{2} + 2p^{2}q)+10p}{1 - pq + p^{2}q}$, respectively.
    On the domain $\interval{0.2}{0.9} \times \interval{0.2}{0.8}$, these functions have $M_{\textsf{r}} \approx 959$, attained at $(p,q) = (0.2,0.2)$, and $M_{\textsf{m}} \approx 993$, attained at $(p,q) = (0.2,0.8)$, respectively, as upper bounds. 
    If we would change the reward structure, for instance by setting $\mreward_{\mstates}(s_{3}) = 2000$, then the upper bounds would change to $M'_{\textsf{r}} \approx 1918$, attained at $(p,q) = (0.2,0.2)$, and $M'_{\textsf{m}} \approx 1984$, attained at $(p,q) = (0.2,0.8)$, respectively, showing that the upper bound $M$ used in Lemma~\ref{lem:nearMargin} strictly depends on the actual pMDRM $\pmdrm$ we are analyzing.
\end{markedexample}

\subsubsection{Generalization for Reward Properties.}

The constructions given above can be applied to both the functions corresponding to probability and reward properties:
for instance, we can approximate the rational function representing the state property $\lsf = \lER{\mathord{= ?}}(\lF \lsf')$, the reward counterpart of $\lP{\mathord{=}?}(\lpf')$, by instantiating $f_{\lsf}(\parameters_{i})$ in Problem~\eqref{eq:PACLPpolynomialApproximation} with the expected reward value computed on the pMDP instantiated with $\parameters_{i}$.
Similarly, we can compute linear and polynomial PAC approximations for safe regions, with the latter defined in terms of the value of the reward instead of the probability.
The only result that does not apply easily to rewards is Lemma~\ref{lem:nearMargin}, as it needs an upper bound of the reward property that is not known in advance.

To generalize this result, we can consider also the following case:
given a pMDRM $\pmdrm$, we want to verify whether the expected value of $\lsf = \lER{=?}(\lF \lsf')$ over the parameters $\parameters$, denoted $f_{\lsf}(\parameters)$, can reach a given threshold $\threshold$. 
This model the situations where, to make a decision, we need to know whether the expectation of the rewards for a certain decision satisfies the given conditions.
We formalize this case as follows:
\begin{definition}
\label{def:integralAboveThreshold}
    Given the domain $\parametersDomain$ of a set of parameters, a function $f_{\lsf} \colon \parametersDomain \to \posreals$, a threshold $\threshold \in \posreals$, and a probability measure $P$ over $\parametersDomain$, we say that \emph{the expectation of $f_{\lsf}$} on $\parametersDomain$ with respect to $P$ can reach the threshold $\threshold$, if 
    \begin{equation}
    \label{eq:integralThreshold}
        \int_{\parametersDomain} f_{\lsf}(\parameters)\,dP(\parameters) > \threshold.
    \end{equation}
\end{definition}

We can resort to the following lemma to check condition~\eqref{eq:integralThreshold}:
\begin{lemma}
\label{lem:integralThreshold}
    Given the domain $\parametersDomain$ of a set of parameters, a probability measure $P$ over $\parametersDomain$, a function $f_{\lsf} \colon \parametersDomain \to \posreals$, and a threshold $\threshold \in \posreals$, let $\ApproxFunOfProperty{f}$ be a $\margin$-PAC approximation of $f$ with $(\errorRate, \significanceLevel)$-guarantee. 
    If $\ApproxFunOfProperty{f}$ satisfies the condition 
    \begin{equation}
    \label{eq:lowerboundofrewardmodel}
        \int_{\parametersDomain}(\ApproxFunOfProperty{f}(\parameters) - \margin)\,dP(\parameters) - \errorRate \cdot \size{\parametersDomain} \cdot \max_{\parameters \in \parametersDomain}(\ApproxFunOfProperty{f}(\parameters) - \margin) > \threshold,
    \end{equation}
    then $\int_{\parametersDomain} f_{\lsf}(\parameters)\,dP(\parameters) > \threshold$ holds with confidence $1 - \significanceLevel$.
\end{lemma}
\begin{proof}
    By splitting the integral region $\parametersDomain$ into two parts $\parametersDomain = \parametersDomain_{1} \uplus \parametersDomain_{2}$ where $\parametersDomain_{1}$ is such that $P(\parametersDomain_{1}) \geq 1 - \errorRate$ and for each $\parameters \in \parametersDomain_{1}$ we have $\abs{f_{\lsf}(\parameters) - \ApproxFunOfProperty{f}(\parameters)} \leq \margin$, as in the proof of Lemma~\ref{lem:nearMargin}, we have
    \[
        \int_{\parametersDomain} f_{\lsf}(\parameters)\,dP(\parameters) = \int_{\parametersDomain_{1}} f_{\lsf}(\parameters)\,dP(\parameters) + \int_{\parametersDomain_{2}} f_{\lsf}(\parameters)\,dP(\parameters).
    \]
    Since we have $f_{\lsf}(\parameters) \geq 0$ for each $\parameters \in \parametersDomain$ by definition of $f_{\lsf}$, it follows that $\int_{\parametersDomain_{2}} f(\parameters)\,dP(\parameters) \geq 0$. 
    This implies that
    \begin{align*}
        \int_{\parametersDomain} f_{\lsf}(\parameters)\,dP(\parameters) & \geq \int_{\parametersDomain_{1}} f_{\lsf}(\parameters)\,dP(\parameters) \\
        & \geq \int_{\parametersDomain_{1}} (\ApproxFunOfProperty{f}(\parameters) - \margin)\,dP(\parameters) \\
        & = \int_{\parametersDomain} (\ApproxFunOfProperty{f}(\parameters) - \margin)\,dP(\parameters) - \int_{\parametersDomain_{2}} (\ApproxFunOfProperty{f}(\parameters) - \margin)\,dP(\parameters) \\
        & \geq \int_{\parametersDomain} (\ApproxFunOfProperty{f}(\parameters) - \margin)\,dP(\parameters) - \errorRate \cdot \size{\parametersDomain} \cdot \max_{\parameters \in \parametersDomain_{2}}(\ApproxFunOfProperty{f}(\parameters) - \margin) \\ 
        & \geq \int_{\parametersDomain} (\ApproxFunOfProperty{f}(\parameters) - \margin)\,dP(\parameters) - \errorRate \cdot \size{\parametersDomain} \cdot \max_{\parameters \in \parametersDomain}(\ApproxFunOfProperty{f}(\parameters) - \margin).
    \end{align*}
    This means that if the approximation polynomial $\ApproxFunOfProperty{f}$ satisfies 
    \[
        \int_{\parametersDomain} (\ApproxFunOfProperty{f}(\parameters) - \margin)\,dP(\parameters) - \errorRate \cdot \size{\parametersDomain} \cdot \max_{\parameters \in \parametersDomain}(\ApproxFunOfProperty{f}(\parameters) - \margin) \geq \threshold,
    \]
    then 
    \[
        \int_{\parametersDomain} f_{\lsf}(\parameters)\,dP(\parameters) \geq \int_{\parametersDomain} (\ApproxFunOfProperty{f}(\parameters) - \margin)\,dP(\parameters) - \errorRate \cdot \size{\parametersDomain} \cdot \max_{\parameters \in \parametersDomain}(\ApproxFunOfProperty{f}(\parameters) - \margin) \geq \threshold
    \]
    holds with confidence $1 - \significanceLevel$, as required.
\end{proof}

\subsection{pMDRM Specific PRCTL Analysis}
\label{ssed:pMDRMspecificPRCTLanalysis}

One main difference between pMCs and pMDRMs is that for a given PRCTL property $\lsf = \lQ$, for the former the function $f_{\lsf}$ has a single expression while for the latter the function $f_{\lsf}$ is likely to be defined piece-wise, since for different subregions of $\parametersDomain$ the policy $\opt$imizing $\lQ$ can be different, so is the corresponding expression of $f_{\lsf}$.
Although each fixed-policy branch is a rational function on a well-posed parameter domain, the optimized value function for a pMDRM may fail to be smooth at the boundaries where the optimizing policy changes.
To see this, consider again the two functions $f_{\lsf}^{\textsf{r}}$ and $f_{\lsf}^{\textsf{m}}$ shown in Figure~\ref{fig:pmdp1MultipleFunctions}:
the value of $\max\setnocond{f_{\lsf}^{\textsf{r}}, f_{\lsf}^{\textsf{m}}}$ bounces at the policy switch points, as clearly visible along $p = 0.9$. 
The following results therefore do not base the choice of the policy on the smoothness of the induced functions; 
instead, they compare fixed-policy branches directly and use the PAC approximation of their differences to identify which policy gives the highest value for $\lsf$ on a given parameter region.

\subsubsection*{Locally better policy.}
By fixing a specific policy, the pMDP degenerates into a pMC, so we can easily calculate the PAC approximation of the induced rational function for $\lsf$, 
which represents the property of interest.
In general, for a pMDP, there are multiple policies we can use to resolve nondeterminism; 
among a pair of them, we want to choose the one providing the higher reachability probability under some parameter conditions. 
This can be achieved by means of the following lemma:
\begin{lemma}
\label{lem:betterPolicy}
    Given a pMDRM $\pmdrm$, let $\parametersDomain$ be the domain of its parameters and $P$ be a probability measure over $\parametersDomain$. 
    Given a PRCTL state formula $\lsf$ and two policies $\policy$ and $\policy'$, let $f_{\lsf}$ and $f'_{\lsf}$ be the two rationals functions corresponding to $\lsf$ when computed on the induced parametric Markov chains $\applypolicy{\pmdrm}$ and $\applypolicy[\policy']{\pmdrm}$, respectively.
    Let $\ApproxFunOfDiff{g}$ be a $\margin$-PAC approximation of $f_{\lsf} - f'_{\lsf}$ with $(\errorRate, \significanceLevel)$-guarantee.

    \begin{enumerate}
    \item
        If for each parameter $\parameters \in \parametersDomain$, we have that
        \begin{equation}
        \label{eq:pmdppolicy}
            \ApproxFunOfDiff{g}(\parameters) - \margin > 0
        \end{equation}
        holds, then the functions $f$ and $f'$ satisfy the condition $f_{\lsf}(\parameters) > f'_{\lsf}(\parameters)$ with $(\errorRate, \significanceLevel)$-guarantee, that is, $P(f_{\lsf}(\parameters) > f'_{\lsf}(\parameters)) \geq 1 - \errorRate$ holds with confidence $1 - \significanceLevel$.
    \item
        In turn, if $P(\ApproxFunOfDiff{g}(\parameters) + \margin \leq 0) > \errorRate$, then there exists at least a choice of parameters $\parameters^{*}$ in the parameter space such that $f(\parameters^{*}) \leq f'(\parameters^{*})$ with confidence $1 - \significanceLevel$, that is, the policy $\policy'$ is better than policy $\policy$ at the parameter $\parameters^{*}$ with confidence $1 - \significanceLevel$.
    \end{enumerate}
\end{lemma}
\begin{proof}
    We prove the first statement of the lemma as follows. 
    Since $\ApproxFunOfDiff{g}$ is a $\margin$-PAC approximations of $f_{\lsf} - f'_{\lsf}$ with $(\errorRate, \significanceLevel)$-guarantee, namely, $P(\ApproxFunOfDiff{g}(\parameters) - \margin \leq f_{\lsf}(\parameters) - f'_{\lsf}(\parameters) \leq \ApproxFunOfDiff{g}(\parameters) + \margin) \geq 1 - \errorRate$ holds with confidence $1 - \significanceLevel$, so it can be easily deduced that $P(f_{\lsf}(\parameters) - f'_{\lsf}(\parameters) \geq \ApproxFunOfDiff{g}(\parameters) - \margin) \geq 1 - \errorRate$.
    Since the condition $\ApproxFunOfDiff{g}(\parameters) - \margin > 0$ is satisfied, the following inequalities can be easily derived with confidence $1 - \significanceLevel$:
    \[
        P(f_{\lsf}(\parameters) - f'_{\lsf}(\parameters) > 0) \geq P(f_{\lsf}(\parameters) - f'_{\lsf}(\parameters) \geq \ApproxFunOfDiff{g}(\parameters) - \margin) \geq 1 - \errorRate.
    \]

    Consider now the second statement of the lemma; 
    we prove it by contradiction: 
    under the hypothesis that $P(\ApproxFunOfDiff{g}(\parameters) + \margin \leq 0) > \errorRate$, suppose that for all parameters $\parameters \in \parametersDomain$, we have that $f(\parameters) > f'(\parameters)$ holds with the confidence $1 - \significanceLevel$.
    Since $\ApproxFunOfDiff{g}$ is a $\margin$-PAC approximation of $f_{\lsf} - f'_{\lsf}$ with $(\errorRate, \significanceLevel)$-guarantee, we have that $P(\abs{f_{\lsf}(\parameters) - f'_{\lsf}(\parameters) -\ApproxFunOfDiff{g}(\parameters)} \leq \margin) \geq 1 - \errorRate$ holds; 
    this implies that 
    \[
        P(f_{\lsf}(\parameters) - f'_{\lsf}(\parameters) - \margin \leq \ApproxFunOfDiff{g}(\parameters) \leq f_{\lsf}(\parameters) - f'_{\lsf}(\parameters) + \margin) \geq 1 - \errorRate,
    \]
    hence $P(\ApproxFunOfDiff{g}(\parameters) + \margin \geq f_{\lsf}(\parameters) - f'_{\lsf}(\parameters)) \geq 1 - \errorRate$ holds as well. 
    Since by assumption for all parameters $\parameters \in \parametersDomain$, $f_{\lsf}(\parameters) > f'_{\lsf}(\parameters)$ holds with confidence $1 - \significanceLevel$,
    we have that
    \[
        P(\ApproxFunOfDiff{g}(\parameters) + \margin > 0) > P(\ApproxFunOfDiff{g}(\parameters) + \margin > f_{\lsf}(\parameters) - f'_{\lsf}(\parameters)) \geq 1 - \errorRate,
    \]
    which is equivalent to $P(\ApproxFunOfDiff{g}(\parameters) + \margin \leq 0) \leq \errorRate$;
    this however contradicts the hypothesis that ``$P(\ApproxFunOfDiff{g}(\parameters) + \margin \leq 0) > \errorRate$'', so the assumption ``for all parameters $\parameters \in \parametersDomain$, we have that $f(\parameters) > f'(\parameters)$ holds with the confidence $1 - \significanceLevel$'' is false, thus there exists at least a choice of parameters $\parameters^{*}$ in the parameter space such that $f(\parameters^{*}) \leq f'(\parameters^{*})$ holds with confidence $1 - \significanceLevel$, as desired.
\end{proof}

\begin{figure}[tb]
    \centering
    \resizebox{\linewidth}{!}{
        \begin{tikzpicture}
            \node (rat) at (0,0) {\includegraphics{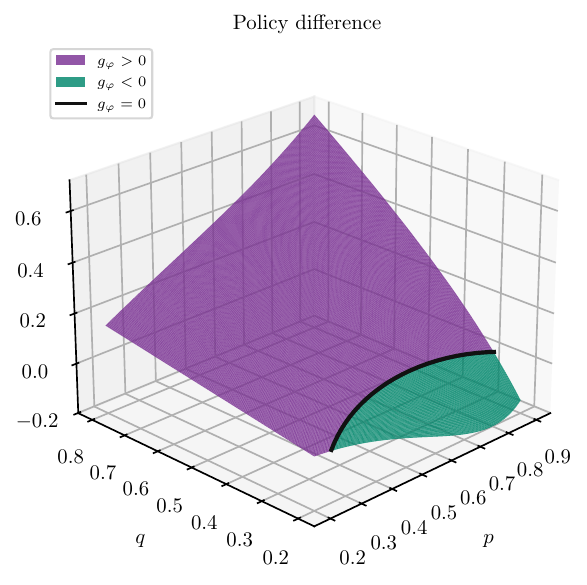}};
            \node[anchor=west] (apprx) at (rat.east) {\includegraphics{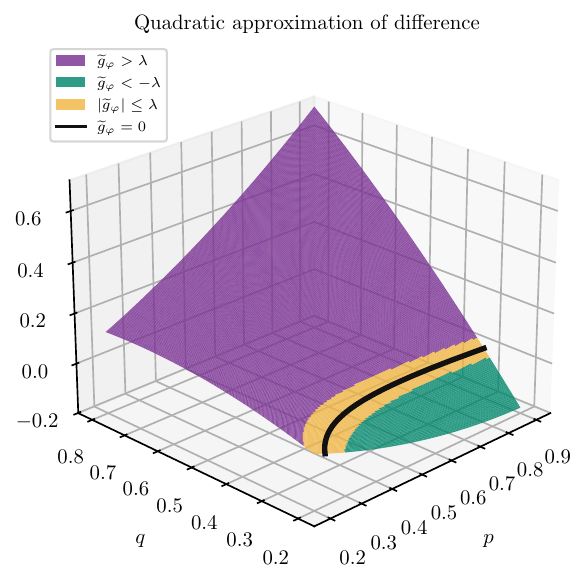}};
        \end{tikzpicture}
    }
    \caption{The difference $g_{\lsf} = f_{\lsf}^{\textsf{r}} - f_{\lsf}^{\textsf{m}}$ (on the left) and its quadratic approximation (on the right). The purple regions have positive difference and therefore favor repair, the teal regions have negative difference and therefore favor migration, and the amber band in the approximation marks the area where $\abs{\ApproxFunOfDiff{g}(\parameters)} \leq \margin$}
    \label{fig:pmdp1MultipleFunctionsDifference}
\end{figure}
\begin{markedexample}
\label{ex:betterPolicy}
    From Lemma~\ref{lem:betterPolicy}, we can know with a high confidence and low error rate that if the inequality~\eqref{eq:pmdppolicy} is satisfied, then the policy $\policy$ gives a higher value for the satisfaction of $\lsf$ than the policy $\policy'$.
    When we apply the lemma to the pMDP shown in Figure~\ref{fig:pmdp}, on the domain of parameters $\range(p) = \interval{0.2}{0.9}$ and $\range(q) = \interval{0.2}{0.8}$ we know from Example~\ref{ex:pmdp} that the two rational functions for $\lsf = \lP[\max]{\mathord{=}?}(\lF \success)$ are 
    \[
        f_{\lsf}^{\textsf{r}} = \frac{pq}{1-p^{2}+p^{2}q} 
        \qquad\qquad
        f_{\lsf}^{\textsf{m}} = \frac{p^{2}(1-q)}{1-pq+p^{2}q}
    \]
    as induced by the policies $\policy^{\textsf{r}}$ and $\policy^{\textsf{m}}$ choosing \textsf{r} and \textsf{m} in $s_{1}$, respectively, that we computed in Example~\ref{ex:pmdp} (depicted in Figure~\ref{fig:pmdp1MultipleFunctions}).
    The left plot of Figure~\ref{fig:pmdp1MultipleFunctionsDifference} shows the difference $g_{\lsf}(p,q) = f_{\lsf}^{\textsf{r}}(p,q) - f_{\lsf}^{\textsf{m}}(p,q)$ between the two functions; 
    in the plot we also mark where their difference crosses the $0$ level by a solid black line, which corresponds to the policy switch shown in Figure~\ref{fig:pmdp1MultipleFunctions} from the difference-function perspective.
    Purple positive values of $g_{\lsf}$ indicate the repair-dominant region, teal negative values indicate the migration-dominant region, and values close to zero are the parameter valuations where the two policies are hard to distinguish.
    The plot on the right shows a quadratic PAC approximation of $g_{\lsf}$; 
    because it is an approximation with margin $\margin$, conclusions about which policy is better are robust away from the zero level set but become uncertain in the narrow amber band where $\abs{\ApproxFunOfDiff{g}}$ is below $\margin$.
    The analytical expression of the difference between the two functions is
    \[
        g_{\lsf}(p,q) = f_{\lsf}^{\textsf{r}} - f_{\lsf}^{\textsf{m}}
        =
        \frac{p(p^{3}q^{2}-2p^{3}q+p^{3}+p^{2}q^{2}-pq^{2}+pq-p+q)}{(1-p^{2}+p^{2}q)(1-pq+p^{2}q)},
    \]
    so the better policy switches along the zero set of the numerator.
    If we take $g_{\lsf}(p,q)$ and compute its approximation $\ApproxFunOfDiff{g}$, we can easily see for the computed margin $\margin \approx 0.029$ that the condition~\eqref{eq:pmdppolicy} is satisfied for instance at the point $(p, q) = (0.7, 0.4)$, for which we have $\ApproxFunOfDiff{g}(0.7, 0.4) - 0.029 \approx 0.073$.
    On the other hand, the condition is violated for e.g.\@ the point $(p, q) = (0.7, 0.2)$, for which we have $\ApproxFunOfDiff{g}(0.7, 0.2) - 0.029 \approx -0.186$.
    If we consider instead the condition $P(\ApproxFunOfDiff{g}(\parameters) + \margin \leq 0) > \errorRate$ in the second point of the lemma, then we get that $P(\ApproxFunOfDiff{g}(\parameters) + 0.029 \leq 0) \approx 0.11 > 0.05$, so there exists at least a choice of parameters $\parameters^{*}$ such that $f_{\lsf}^{\textsf{r}}(\parameters^{*}) \leq f_{\lsf}^{\textsf{m}}(\parameters^{*})$ with confidence $1 - \significanceLevel$: 
    this happens on the migration-dominant side of the policy-switch curve, i.e., the teal parts of the plots in Figure~\ref{fig:pmdp1MultipleFunctionsDifference}.
\end{markedexample}

\subsubsection*{Overall better policy.}
As we have seen from the previous example, different policies can give very different outcomes on the behavior of the pMDP.
If we have to choose only one policy to be applied to the pMDP to solve nondeterminism independently on the actual value of the parameters, we need to understand the overall performance of different policies over the entire parameter space to help us choose the policy that on average performs better overall.
Consider again the previous example, where for the cloud-service pMDP shown in Figure~\ref{fig:pmdp} there are two possible policies, namely $\policy^{\textsf{r}}$ and $\policy^{\textsf{m}}$, whose corresponding reachability probability functions of the goal state are $f_{\lsf}^{\textsf{r}}$ and $f_{\lsf}^{\textsf{m}}$, respectively.
We say that the policy $\policy^{\textsf{r}}$ is better than $\policy^{\textsf{m}}$ with respect to a given probability measure $P$ over $\parametersDomain$, if the inequality $\int_{\parametersDomain} f_{\lsf}^{\textsf{r}}(\parameters)\,dP(\parameters) \geq \int_{\parametersDomain} f_{\lsf}^{\textsf{m}}(\parameters)\,dP(\parameters)$ holds, that is, the expected value of $f_{\lsf}^{\textsf{r}}$ over the whole domain $\parametersDomain$ is larger than the expected value of $f_{\lsf}^{\textsf{m}}$.

Since computing and evaluating the rational functions $f_{\lsf}^{\textsf{r}}$ and $f_{\lsf}^{\textsf{m}}$ can be difficult, we consider instead their PAC approximations $\ApproxFunOfProperty{f}^{\textsf{r}}$ and $\ApproxFunOfProperty{f}^{\textsf{m}}$; 
the following lemma allows us to determine the better policy over the whole parameter space, with statistical guarantees.

\begin{lemma}
\label{lem:betterpolicyoverall}
    Given a pMDRM with domain of parameters $\parametersDomain$, a PRCTL property $\lsf = \lQ$, two policies $\policy$ and $\policy'$, and their corresponding functions $f_{\lsf} \colon \parametersDomain \to \posreals$ and $f'_{\lsf} \colon \parametersDomain \to \posreals$,
    let $\ApproxFunOfDiff{g}$ be a $\margin$-PAC approximation of $g_{\lsf} = f_{\lsf} - f'_{\lsf}$ with $(\errorRate, \significanceLevel)$-guarantee.
    If the inequality
    \begin{equation}
    \label{eq:pmdpMargin}
        \int_{\parametersDomain} \ApproxFunOfDiff{g}(\parameters)\,dP(\parameters) \geq (\max_{\parameters \in \parametersDomain} \ApproxFunOfDiff{g}(\parameters) - \inf_{\parameters \in \parametersDomain} g_{\lsf}(\parameters)) \cdot \errorRate \cdot \size{\parametersDomain} + \margin  \cdot \size{\parametersDomain}
    \end{equation}
    holds, where $\size{\parametersDomain} = \int_{\parametersDomain} 1\,d\parameters$, then the inequality
    \[
        \int_{\parametersDomain} f_{\lsf}(\parameters)\,dP(\parameters) \geq \int_{\parametersDomain} f'_{\lsf}(\parameters)\,dP(\parameters)
    \]
    holds with confidence $1 - \significanceLevel$.
\end{lemma}
\begin{proof}
    We partition the parameter space $\parametersDomain$ into two disjoint parts $\parametersDomain = \parametersDomain_{\leq \margin} \cup \parametersDomain_{> \margin}$ where $\parametersDomain_{\leq \margin} = \setcond{\parameters \in \parametersDomain}{\abs{g_{\lsf}(\parameters) - \ApproxFunOfDiff{g}(\parameters)} \leq \margin}$ and $\parametersDomain_{> \margin} = \setcond{\parameters \in \parametersDomain}{\abs{g_{\lsf}(\parameters) - \ApproxFunOfDiff{g}(\parameters)} > \margin}$, which satisfies $P(\parametersDomain_{\leq \margin}) \geq 1 - \errorRate$ and $P(\parametersDomain_{> \margin}) < \errorRate$ with confidence $1 - \significanceLevel$.
    Then we have that
    \[
        \int_{\parametersDomain_{\leq \margin}} g_{\lsf}(\parameters)\,dP(\parameters) = \int_{\parametersDomain} g_{\lsf}(\parameters)\,dP(\parameters) - \int_{\parametersDomain_{> \margin}} g_{\lsf}(\parameters)\,dP(\parameters);
    \]
    a similar equality holds when $\ApproxFunOfDiff{g}$ is used instead of $g_{\lsf}$. 
    This implies that
    \begin{align*}
        \int_{\parametersDomain_{\leq \margin}} g_{\lsf}(\parameters)\,dP(\parameters) \geq {} & \int_{\parametersDomain_{\leq \margin}} (\ApproxFunOfDiff{g}(\parameters) - \margin)\,dP(\parameters) \\
        {} = {} & \int_{\parametersDomain_{\leq \margin}} \ApproxFunOfDiff{g}(\parameters)\,dP(\parameters) - \int_{\parametersDomain_{\leq \margin}} \margin \,dP(\parameters) \\
        {} = {} & \int_{\parametersDomain_{\leq \margin}} \ApproxFunOfDiff{g}(\parameters)\,dP(\parameters) - \margin \cdot \size{\parametersDomain_{\leq \margin}} \\
        {} \geq {} & \int_{\parametersDomain_{\leq \margin}} \ApproxFunOfDiff{g}(\parameters)\,dP(\parameters) - \margin \cdot \size{\parametersDomain} \\
        {} = {} & \int_{\parametersDomain} \ApproxFunOfDiff{g}(\parameters)\,dP(\parameters) - \int_{\parametersDomain_{> \margin}} \ApproxFunOfDiff{g}(\parameters)\,dP(\parameters) - \margin \cdot \size{\parametersDomain} \\
        {} \geq {} & \int_{\parametersDomain} \ApproxFunOfDiff{g}(\parameters)\,dP(\parameters) - (\max_{\parameters \in \parametersDomain} \ApproxFunOfDiff{g}(\parameters)) \cdot \errorRate \cdot \size{\parametersDomain} - \margin \cdot \size{\parametersDomain}. 
    \end{align*}
    with confidence $1-\significanceLevel$. Moreover, we have that 
    \[
        \int_{\parametersDomain_{> \margin}} g_{\lsf}(\parameters)\,dP(\parameters) \geq (\inf_{\parameters \in \parametersDomain} g_{\lsf}(\parameters)) \cdot \errorRate \cdot \size{\parametersDomain}
    \]
    holds with confidence $1-\significanceLevel$. 
    By combining all inequalities together, we have that
    \begin{align*}
        \int_{\parametersDomain} g_{\lsf}(\parameters)\,dP(\parameters) \geq {} & \int_{\parametersDomain} \ApproxFunOfDiff{g}(\parameters)\,dP(\parameters) - \margin \cdot \size{\parametersDomain} - (\max_{\parameters \in \parametersDomain} \ApproxFunOfDiff{g}(\parameters)) \cdot \errorRate \cdot \size{\parametersDomain} + (\inf_{\parameters \in \parametersDomain} g_{\lsf}(\parameters) \cdot \errorRate \cdot \size{\parametersDomain} \\
        {} = {} & \int_{\parametersDomain} \ApproxFunOfDiff{g}(\parameters)\,dP(\parameters) - \margin \cdot \size{\parametersDomain} - (\max_{\parameters \in \parametersDomain} \ApproxFunOfDiff{g}(\parameters) - \inf_{\parameters \in \parametersDomain} g_{\lsf}(\parameters)) \cdot \errorRate \cdot \size{\parametersDomain}. 
    \end{align*}
    Therefore, under the condition~\eqref{eq:pmdpMargin}, the inequality 
    \[
        \int_{\parametersDomain} f_{\lsf}(\parameters)\,dP(\parameters) \geq \int_{\parametersDomain} f'_{\lsf}(\parameters)\,dP(\parameters)
    \]
    holds with confidence $1 - \significanceLevel$.
\end{proof}
In Lemma~\ref{lem:betterpolicyoverall}, we have to deal with $\inf_{\parameters \in \parametersDomain} g_{\lsf}(\parameters)$, which depends on the actual rational functions corresponding to $\lsf$.
For reward properties, we cannot know in advance its value; 
for probabilistic properties, instead, we have that the value of $\inf_{\parameters \in \parametersDomain} g_{\lsf}(\parameters)$ is at least -1 since the upper bounds and lower bounds of $f_{\lsf}(\parameters)$ and $f'_{\lsf}(\parameters)$ are known to be 1 and 0, respectively.

\section{Combining PAC Approximations with Statistical Model Checking}
\label{sec:combiningPACwithSMC}

When solving the optimization problem~\eqref{eq:PACLPpolynomialApproximation}, computing $f_{\lsf}(\parameters_{i})$ usually requires the use of model checking software such as \storm or \prism, which can be very costly for larger models. 
Therefore, attempting to estimate $f_{\lsf}(\parameters_{i})$ using Statistical Model Checking (SMC) methods~\cite{DBLP:series/lncs/LegayLTYSG19} is highly valuable, as SMC can quickly provide a value for $f_{\lsf}(\parameters_{i})$ with statistical guarantees. 
In this section, we show how we can integrate the statistical guarantees provided by SMC with the ones given by the scenario approach, to establish global guarantees.
Following the standard SMC estimation view~\cite{DBLP:series/lncs/LegayLTYSG19}, we use the following definition.

\begin{definition}
\label{def:smcestimation}
    Given a domain of parameters $\parametersDomain$, a probability measure $P$ over $\parametersDomain$, and a function $f_{\lsf} \colon \parametersDomain \to \posreals$, let $\margin \in \posreals$ be a margin to measure the estimation error and $\errorRate \in \posreals$ be an error rate.

    We say that the SMC estimation $\smcEstimation{f}_{\lsf} \colon \parametersDomain \to \posreals$ of $f_{\lsf}$ with respect to probability measure $P$ is a $(\errorRate, \margin)$-estimation if 
    we have $P(\abs{f_{\lsf}(\parameters) - \smcEstimation{f}_{\lsf}(\parameters)} \leq \margin) \geq 1 - \errorRate$.
    $\errorRate$ is related to the number of simulations $\nsims$ used in the SMC algorithm by 
    \[
        \nsims \geq \frac{\ln 2 - \ln \errorRate}{2 \margin^{2}}.
    \]
\end{definition}

The following theorem combines the guarantees provided by SMC with the PAC guarantees from the scenario approach, to establish global guarantees for the obtained approximation.
\begin{theorem}
\label{thm:combiningPACwithSMC}
    Given a domain of parameters $\parametersDomain$, a probability measure $P$ over $\parametersDomain$, and a function $f_{\lsf} \colon \parametersDomain \to \posreals$, let $\smcEstimation{f}_{\lsf} \colon \parametersDomain \to \posreals$ be an SMC $(\errorRate_{s}, \margin_{s})$-estimation of $f_{\lsf}$
    and $\ApproxFunOfProperty[\margin_{p}]{f}$ be a $\margin_{p}$-PAC approximation of $\smcEstimation{f}_{\lsf}$ with $(\errorRate_{p}, \significanceLevel_{p})$-guarantee.
    Then $\ApproxFunOfProperty[\margin_{p}]{f}$ is a $(\margin_{s} + \margin_{p})$-PAC approximation of $f_{\lsf}$ with $(\errorRate, \significanceLevel_{p})$-guarantee, that is,
    \[
        P(\abs{f_{\lsf}(\parameters) - \ApproxFunOfProperty[\margin_{p}]{f}(\parameters)} \leq \margin_{s} + \margin_{p}) \geq 1 - \errorRate
    \]
    holds with confidence $1 - \significanceLevel_{p}$, where $\errorRate = \min\setnocond{1, \errorRate_{s} + \errorRate_{p}}$.
    Moreover, $\errorRate$ can be replaced by $\errorRate = \errorRate_{\mathrm{ind}} = \errorRate_{s} + \errorRate_{p} - \errorRate_{s} \errorRate_{p}$ if the product-good-set condition $P(G_{s} \cap G_{p}) \geq (1 - \errorRate_{s})(1 - \errorRate_{p})$ holds, where 
    $G_{s} =
        \setcond{\parameters \in \parametersDomain}{\abs{f_{\lsf}(\parameters)-\smcEstimation{f}_{\lsf}(\parameters)} \leq \margin_{s}}$
    and 
    $G_{p} =
        \setcond{\parameters \in \parametersDomain}{\abs{\smcEstimation{f}_{\lsf}(\parameters)-\ApproxFunOfProperty[\margin_{p}]{f}(\parameters)} \leq \margin_{p}}$.
    
    The number of simulations $\nsims$ used in the SMC algorithm and the number of sampling points $\nsamples$ used in the PAC algorithm are related to the error rates $\errorRate_{s}$ and $\errorRate_{p}$, significance level $\significanceLevel_{p}$, and margin $\margin_{s}$ by
    \[
        \nsims \geq \frac{\ln 2 - \ln \errorRate_{s}}{2\margin_{s}^{2}},
        \qquad \nsamples \geq \frac{2}{\errorRate_{p}} \cdot \left(\ln \frac{1}{\significanceLevel_{p}} + m \right).
    \]
\end{theorem}
\begin{proof}
    For any parameter $\parameters \in \parametersDomain$, we have that
    \[  
        \abs{f_{\lsf}(\parameters) - \ApproxFunOfProperty[\margin_{p}]{f}(\parameters)} 
        \leq \abs{f_{\lsf}(\parameters) - \smcEstimation{f}_{\lsf}(\parameters)} + \abs{\smcEstimation{f}_{\lsf}(\parameters) -\ApproxFunOfProperty[\margin_{p}]{f}(\parameters)}.
    \]
    This implies the following set inclusion relation on the parameter space $\parametersDomain$:
    \begin{align*}
        & \setcond{\parameters \in \parametersDomain}{\abs{f_{\lsf}(\parameters) - \ApproxFunOfProperty[\margin_{p}]{f}(\parameters)} \leq \margin_{s} + \margin_{p}} \\
        {} \supseteq {} &
        \setcond{\parameters \in \parametersDomain}{\abs{f_{\lsf}(\parameters) - \smcEstimation{f}_{\lsf}(\parameters)} + \abs{\smcEstimation{f}_{\lsf}(\parameters) - \ApproxFunOfProperty[\margin_{p}]{f}(\parameters)} \leq \margin_{s} + \margin_{p}}\\
        {} \supseteq {} & 
        \setcond{\parameters \in \parametersDomain}{\text{$\abs{f_{\lsf}(\parameters) - \smcEstimation{f}_{\lsf}(\parameters)} \leq \margin_{s}$ and $\abs{\smcEstimation{f}_{\lsf}(\parameters) - \ApproxFunOfProperty[\margin_{p}]{f}(\parameters)} \leq \margin_{p}$}}.
    \end{align*}
    Therefore, under the probability measure $P$, we have
    \begin{align*}
        & P(\abs{f_{\lsf}(\parameters) - \ApproxFunOfProperty[\margin_{p}]{f}(\parameters)} \leq \margin_{s} + \margin_{p}) \\
        {} \geq {} & P(\abs{f_{\lsf}(\parameters) - \smcEstimation{f}_{\lsf}(\parameters)} + \abs{\smcEstimation{f}_{\lsf}(\parameters) -\ApproxFunOfProperty[\margin_{p}]{f}(\parameters)} \leq \margin_{s} + \margin_{p})\\
        {} \geq {} & P\big(\text{$\abs{f_{\lsf}(\parameters) - \smcEstimation{f}_{\lsf}(\parameters)} \leq \margin_{s}$ and $\abs{\smcEstimation{f}_{\lsf}(\parameters) - \ApproxFunOfProperty[\margin_{p}]{f}(\parameters)} \leq \margin_{p}$} \big) \\
        {} \geq {} & 1 - P(\abs{f_{\lsf}(\parameters) - \smcEstimation{f}_{\lsf}(\parameters)} > \margin_{s}) \\
        & {} - P(\abs{\smcEstimation{f}_{\lsf}(\parameters) - \ApproxFunOfProperty[\margin_{p}]{f}(\parameters)} > \margin_{p}) \\
        {} \mystackrelsingle{*}{\geq} {} & 1 - \errorRate_{s} - \errorRate_{p}.
    \end{align*}
    The inequality $\mystackrelsingle{*}{\geq}$ is justified by the conclusion of Definition~\ref{def:smcestimation} for the first bad event and by the $\margin_{p}$-PAC $(\errorRate_{p}, \significanceLevel_{p})$-guarantee on $\ApproxFunOfProperty[\margin_{p}]{f}$ for the second bad event.
    This argument uses the union bound and therefore does not require independence between the SMC estimation error and the PAC approximation error.
    If the product-good-set condition $P(G_{s} \cap G_{p}) \geq (1-\errorRate_{s})(1-\errorRate_{p})$ holds, then the third line of the display can instead be bounded from below by $(1-\errorRate_{s})(1-\errorRate_{p}) = 1 - \errorRate_{s} - \errorRate_{p} + \errorRate_{s}\errorRate_{p}$, which yields the sharper error rate $\errorRate_{\mathrm{ind}} = \errorRate_{s} + \errorRate_{p} - \errorRate_{s}\errorRate_{p}$.
    Thus we have that if the number of simulations $\nsims$ and the number of sampled points $\nsamples$ satisfy the requirements of the SMC and PAC algorithms 
    \[
        \nsims \geq \frac{\ln 2 - \ln \errorRate_{s}}{2\margin_{s}^{2}},
        \qquad \nsamples \geq \frac{2}{\errorRate_{p}} \cdot \left(\ln \frac{1}{\significanceLevel_{p}} + m \right),
    \]
    respectively, then we have that $P(\abs{f_{\lsf}(\parameters) - \ApproxFunOfProperty[\margin_{p}]{f}(\parameters)} \leq \margin_{s} + \margin_{p}) \geq 1 - \errorRate$ holds with confidence $1 - \significanceLevel_{p}$, where $\errorRate = \min\setnocond{1,\errorRate_{s}+\errorRate_{p}}$.
    Under the product-good-set condition, the same conclusion holds with $\errorRate_{\mathrm{ind}}$ in place of $\errorRate$.
\end{proof}

\begin{remark}
\label{rem:productGoodSet}
    The product-good-set condition in Theorem~\ref{thm:combiningPACwithSMC} is a condition on the two induced subsets $G_{s}$ and $G_{p}$ of the parameter domain, not merely on the pseudo-random seeds used by an implementation.
    Using independent random generators for SMC simulations and for PAC scenario sampling is a useful implementation practice, but it does not by itself prove that the events $\parameters \in G_{s}$ and $\parameters \in G_{p}$ are independent under the parameter distribution $P$, because the PAC approximation is fitted to SMC-produced values and the two error events may still be coupled through the same parameter-dependent function landscape.
    The sharper rate $\errorRate_{\mathrm{ind}}$ is therefore justified when the product-good-set condition can be established analytically, by the design of an experiment in which the two error mechanisms are independent conditional on $\parameters$, or by an additional validation step estimating $P(G_{s} \cap G_{p})$.
    Without such evidence, the union-bound rate $\errorRate = \min\setnocond{1,\errorRate_{s}+\errorRate_{p}}$ is the sound default.
\end{remark}

\begin{figure}[tb]
    \resizebox{\linewidth}{!}{
    \begin{tikzpicture}
        \node (exact) at (0,0) {\includegraphics{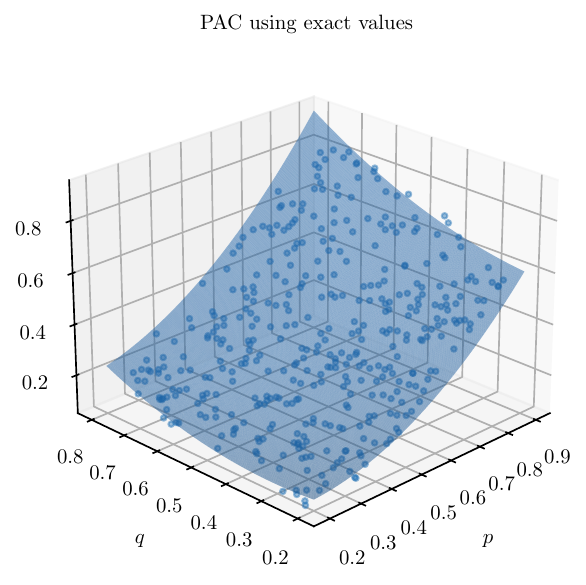}};
        \node[anchor=west] at (exact.east) {\includegraphics{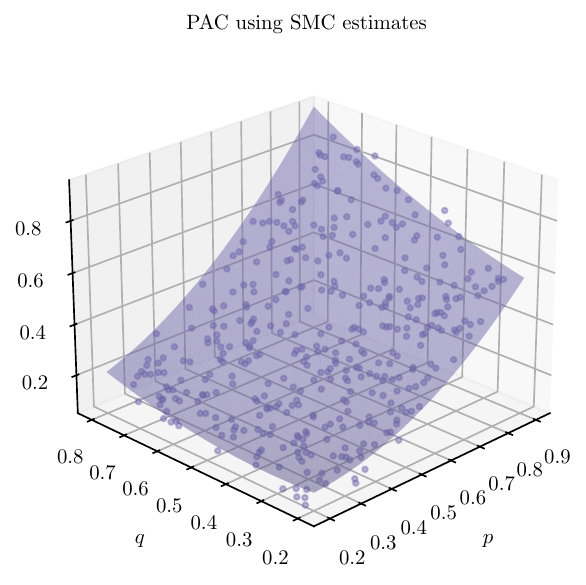}};
    \end{tikzpicture}
    }
    \caption{PAC approximation on exact values (on the left) vs. SMC (on the right)}
    \label{fig:PACcombinedSMC}
\end{figure}
\begin{markedexample}
\label{ex:PACcombinedSMC}
    As an example of combining statistical model checking with approximation based on PAC scenarios, consider again the cloud-service pMDP shown in Figure~\ref{fig:pmdp}.
    From Example~\ref{ex:pmdp} we know that the rational function for $\lsf = \lP[\max]{=?}(\lF \success)$ is the pointwise maximum of $f_{\lsf}^{\textsf{r}} = \frac{pq}{1-p^{2}+p^{2}q}$ and $f_{\lsf}^{\textsf{m}} = \frac{p^{2}(1-q)}{1-pq+p^{2}q}$, plotted in Figure~\ref{fig:pmdp1MultipleFunctions}.
    In Figure~\ref{fig:PACcombinedSMC} we show on the left the quadratic polynomial $\ApproxFunOfProperty{f}$ superimposed with the blue cloud of 400 points computed by the ordinary model checker and used by the PAC scenario to synthesize it;
    on the right of the figure we show the quadratic polynomial $\smcEstimation{f}_{\lsf}$ and in purple the cloud of points for the same choice of sampled parameters computed by a statistical model checker.
    As we can see from the plots, the blue cloud whose $z$ coordinates lie in $\interval{0.048}{0.825}$ is closer to the synthesized function than the purple cloud whose $z$ coordinates lie in $\interval{0.049}{0.871}$; 
    this is reflected by the computed margins, that are $\approxfun[]{\margin} = 0.040$ for $\ApproxFunOfProperty{f}$ and $\smcEstimation{\margin} = 0.076$ for $\smcEstimation{f}_{\lsf}$.
    This is in line with the statement of Theorem~\ref{thm:combiningPACwithSMC}, where using SMC with PAC scenario is affected by an increase of the margin $\margin_{s} + \margin_{p}$ and of the error rate $\min\setnocond{1,\errorRate_{s}+\errorRate_{p}}$. 
\end{markedexample}

Theorem~\ref{thm:combiningPACwithSMC} says that an approximation learned from SMC estimates can be treated as a PAC approximation of the exact satisfaction function, provided that we use the combined margin $\margin_{s} + \margin_{p}$ and the combined error rate $\min\setnocond{1,\errorRate_{s}+\errorRate_{p}}$.
Consequently, the same threshold reasoning used in Lemma~\ref{lem:linearApproximation} can be applied also in the SMC-backed setting.
The following corollary makes this explicit: 
it is the SMC counterpart of Lemma~\ref{lem:linearApproximation}, with $\margin$ replaced by the total SMC-plus-PAC margin and $\errorRate$ replaced by the total error rate.
\begin{corollary}
\label{cor:SMCPACthreshold}
    Under the assumptions of Theorem~\ref{thm:combiningPACwithSMC}, let $\margin_{\mathrm{tot}} = \margin_{s} + \margin_{p}$ and $\errorRate_{\mathrm{tot}} = \min\setnocond{1,\errorRate_{s}+\errorRate_{p}}$.
    For a safety level $\safetyLevel$, the following statements hold with confidence $1-\significanceLevel_{p}$:
    \begin{itemize}
    \item
        if $\ApproxFunOfProperty[\margin_{p}]{f}(\parameters) + \margin_{\mathrm{tot}} < \safetyLevel$ for all $\parameters \in \parametersDomain$, then $P(f_{\lsf}(\parameters) < \safetyLevel) \geq 1-\errorRate_{\mathrm{tot}}$
        and
    \item
        if $P(\ApproxFunOfProperty[\margin_{p}]{f}(\parameters) -\margin_{\mathrm{tot}} > \safetyLevel) > \errorRate_{\mathrm{tot}}$, then $P(f_{\lsf}(\parameters)>\safetyLevel) > 0$.
    \end{itemize}
\end{corollary}
\begin{proof}
    By Theorem~\ref{thm:combiningPACwithSMC}, with confidence $1-\significanceLevel_{p}$ the set $G = \setcond{\parameters \in \parametersDomain}{\abs{f_{\lsf}(\parameters)-\ApproxFunOfProperty[\margin_{p}]{f}(\parameters)} \leq \margin_{\mathrm{tot}}}$ has probability at least $1-\errorRate_{\mathrm{tot}}$.
    If $\ApproxFunOfProperty[\margin_{p}]{f}(\parameters)+\margin_{\mathrm{tot}} < \safetyLevel$ for all $\parameters \in \parametersDomain$, then $f_{\lsf}(\parameters)<\safetyLevel$ for every $\parameters\in G$, hence the first claim follows.
    For the second claim, let $A = \setcond{\parameters \in \parametersDomain}{\ApproxFunOfProperty[\margin_{p}]{f}(\parameters) - \margin_{\mathrm{tot}} > \safetyLevel}$.
    If $P(A) > \errorRate_{\mathrm{tot}}$, then $P(A \cap G) \geq P(A) + P(G) - 1 > 0$.
    For every $\parameters \in A\cap G$, we have $f_{\lsf}(\parameters) \geq \ApproxFunOfProperty[\margin_{p}]{f}(\parameters) - \margin_{\mathrm{tot}} > \safetyLevel$, and therefore $P(f_{\lsf}(\parameters) > \safetyLevel) > 0$.
\end{proof}

\section{DIRECT-Based Optimization for Parametric Markov Models}
\label{sec:DIRECT}

In the previous sections, we have presented the PAC scenario approach for synthesizing an approximation function $\ApproxFunOfProperty{f}$ of the rational function $f_{\lsf}$ associated with a PRCTL property $\lsf$ on a parametric Markov model.
A natural follow-up question is: given the approximation function $\ApproxFunOfProperty{f}$, how can we efficiently locate the parameters $\parameters^{*}$ that optimize $f_{\lsf}$ over the parameter domain $\parametersDomain$?
This is particularly important in the context of parametric MDPs, where one is often interested in finding the parameter configuration that maximizes or minimizes the satisfaction value of a given property, for instance to identify the worst-case or best-case behavior of the system.

To this end, we adopt the DIRECT (DIviding RECTangles) algorithm~\cite{DIRECT}, a derivative-free global optimization method.
DIRECT is well suited for our setting because: 
(i) the function $f_{\lsf}$ (or its approximation $\ApproxFunOfProperty{f}$) is treated as a black box, since we can only evaluate it at sampled parameter points via model checking, and no gradient information about the objective function is required; 
(ii) DIRECT does not require the knowledge of the Lipschitz constant of $f_{\lsf}$, which may be hard to estimate for complex rational functions; 
and (iii) the algorithm naturally balances global exploration and local exploitation, which is desirable when the function landscape contains multiple subregions with different optimal policies.

In this section, we first recall the DIRECT algorithm (Section~\ref{ssec:DIRECT}), then present how it can be applied to the optimization of property satisfaction values in parametric Markov models (Section~\ref{ssec:DIRECTforPMDP}), and finally provide conditional guarantees on the quality of the obtained solutions by relating the optimality gap to the Lipschitz constant and the partition diameter (Section~\ref{ssec:theoreticalGuarantees}).

\subsection{The DIRECT Algorithm}
\label{ssec:DIRECT}
The DIRECT algorithm~\cite{DIRECT} is a global optimization method that operates without the need for a Lipschitz constant. 
Its core principle involves dynamically partitioning the search space into hyperrectangles and conducting searches across all possible Lipschitz constants, effectively balancing global and local exploration. 
The algorithm requires only the continuity of the objective function and samples points within the domain to guide its search process. 
By leveraging the information gathered, it strategically determines the next areas to explore, ensuring a balance between efficiency and effectiveness. 
This makes DIRECT particularly useful for optimizing ``black-box'' functions or simulations, where the underlying structure of the objective function is poorly understood.

Now we describe the DIRECT algorithm in detail. 
To maintain generality, the parameter space $\parametersDomain = \prod_{i = 1}^{n} \range(\parameter_{i})$ is first normalized to the unit hypercube $\hypercube \subseteq \reals^{\Numdimension}$. 
The algorithm initially obtains the function value at the centroid of the normalized domain and then iteratively selects and divides the hyperrectangles that potentially contain the optimal points.    
Let us assume that $\hypercube$ has been partitioned into $\Numdivide$ hyperrectangles $\hypercube_\Numdivide = \setnocond{\hyperrectangle_{1}, \dots, \hyperrectangle_{\Numdivide}}$.
Let the side lengths of each $\hyperrectangle_{k}$ be $\lenside_{k} = (\lenside_{k, 1}, \dots, \lenside_{k, n})$ 
and $c_{k} = (c_{k, 1}, \dots, c_{k, n})$ be the centroid, i.e., the geometric center, of $\hyperrectangle_{k}$; 
then $\distance_{k} = \frac{1}{2}\norm[2]{\lenside_{k}}$ is the
\emph{distance} from the centroid of the $k$-th hyperrectangle to its 
vertices. 
We present the standard DIRECT selection rule for minimization. 
If the original goal is to maximize a function $f$, we apply DIRECT to $h=-f$ and then translate the returned minimizer of $h$ back into a maximizer of $f$.
In DIRECT, the hyperrectangle to be partitioned in each iteration is referred to as a \emph{potentially optimal hyperrectangle}, defined as follows:
\begin{definition}
\label{def:potentiallyOptimal}
    Given $\Numdivide$ hyperrectangles $\hyperrectangle_{1}, \dots, \hyperrectangle_{\Numdivide}$, let $\forsmall > 0$ and $f_{\opt} = \min_{1 \leq k \leq \Numdivide} f(c_{k})$ be the current best sampled function value for a minimization problem. 
    A candidate hyperrectangle $\hyperrectangle_{r} \in \setnocond{\hyperrectangle_{1}, \dots, \hyperrectangle_{\Numdivide}}$ is \emph{potentially optimal}
    if there exists $\tilde{\lipcons} > 0$ such that
    \begin{equation}
    \label{eq:potentiallyOptimal}
        \begin{cases}
        f(c_{r}) - \tilde{\lipcons} \distance_{r} \leq f(c_{k}) - \tilde{\lipcons} \distance_{k}, & \forall k = 1, \dots, \Numdivide,\\
        f(c_{r}) - \tilde{\lipcons} \distance_{r} \leq f_{\opt} - \forsmall \abs{f_{\opt}}.
        \end{cases}
    \end{equation}
\end{definition}
The first condition in Eq.~\eqref{eq:potentiallyOptimal} ensures that $\hyperrectangle_{r}$ could contain the global optimum if the Lipschitz constant were $\tilde{\lipcons}$: 
a hyperrectangle with a small function value at its centroid and a large distance (hence a large unexplored volume) is more likely to be potentially optimal.
The second condition prevents the algorithm from exploring regions where the potential improvement is negligible compared to the current best value $f_{\opt}$.
In each iteration, all potentially optimal hyperrectangles are identified from the current set $\setnocond{\hyperrectangle_{1}, \dots, \hyperrectangle_{\Numdivide}}$ and partitioned.
This selection does not depend on the order in which the hyperrectangles are listed; 
only implementation-level ties, such as equal function values in the later splitting rule, require a deterministic tie-breaking convention.

Once the potentially optimal hyperrectangles are identified, it is necessary to determine the rules for partitioning them. 
For one-dimensional problems, each potentially optimal hyperrectangle is simply divided into three equal parts. 
However, for $n$-dimensional problems, there are $n$ possible directions for partitioning. 
The DIRECT algorithm chooses to partition along the longest edge, as this ensures that the hyperrectangles eventually shrink in every dimension. 
When there are multiple longest edges, the partitioning process is as follows:
\begin{itemize}
\item 
    Let $\hyperrectangle_{k}$ be the potentially optimal hyperrectangle in $\reals^{n}$. 
    Define the set $I = \setcond{j}{\lenside_{k, j} = \max \setcond{\lenside_{k, i}}{i = 1, \dots, n}}$.
\item 
    For all $j \in I$, let $\Delta = \frac{1}{3} \lenside_{k, j}$ and 
    evaluate $f$ at $c_{k} \pm \Delta e_{j}$, where $e_{j}$ is the standard basis vector in the $j$-th direction.
\item 
    Partition $\hyperrectangle_{k}$ along each dimension $j \in I$ into three equal sub-rectangles. 
    The order of partitioning is determined by the values of
    \[
        u_{j} = \min \setnocond{f(c_{k} + \Delta e_{j}), f(c_{k} - \Delta e_{j})}.
    \]
    The dimensions are processed in ascending order of $u_{j}$, with ties broken arbitrarily, so that the largest sub-rectangles contain the centroids with the best function values.
\end{itemize}

\begin{figure}[tb]
    \centering
    \resizebox{0.82\linewidth}{!}{
        \includegraphics{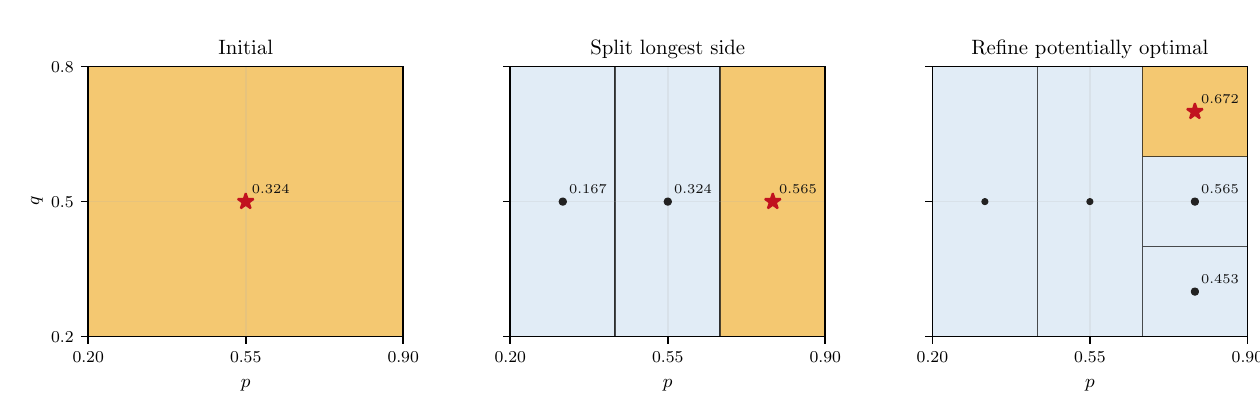}
    }
    \caption{Schematic DIRECT hyperrectangle refinement on the cloud-service parameter domain. 
    The first panel shows the initial hyperrectangle $\parametersDomain = \interval{0.2}{0.9} \times \interval{0.2}{0.8}$, the second panel shows a split along the longest side, and the third panel shows a further refinement of one potentially optimal hyperrectangle.}
    \label{fig:directCloudServiceBlocks}
\end{figure}

\begin{markedexample}
\label{ex:directHyperrectangleSplit}
    Consider again the cloud-service pMDP with parameter domain $\parametersDomain = \interval{0.2}{0.9} \times \interval{0.2}{0.8}$.
    DIRECT first evaluates the objective at the centroid $(p,q) = (0.55,0.5)$ of the initial hyperrectangle $\hyperrectangle = \parametersDomain$, as shown on the left side of Figure~\ref{fig:directCloudServiceBlocks}.
    This is the only hyperrectangle, so it is the only potentially optimal one.
    Since the $p$-side has length $0.7$ and the $q$-side has length $0.6$, the first split is along the $p$ direction, producing three hyperrectangles with centroids approximately $(0.317,0.5)$, $(0.55,0.5)$, and $(0.783,0.5)$, shown in Figure~\ref{fig:directCloudServiceBlocks}, central plot.
    Their reachability values are approximately $0.167$, $0.324$, and $0.565$, respectively.
    For maximizing $f_{\lsf} = \lP[\max]{\mathord{=}?}(\lF\success)$, DIRECT can equivalently minimize the negated value $-f_{\lsf}$; 
    hence the orange hyperrectangle with centroid $(0.783,0.5)$ is the best among the three newly evaluated hyperrectangles, and it is selected as the potentially optimal hyperrectangle in the schematic refinement.
    Another iteration of DIRECT is performed, leading to the hyperrectangles shown on the right side of Figure~\ref{fig:directCloudServiceBlocks}.
\end{markedexample}

\subsection{DIRECT-Based Exploration-Exploitation for Property Optimization}
\label{ssec:DIRECTforPMDP}

We now describe how the DIRECT algorithm can be applied to find the parameter configuration that optimizes the satisfaction value $f_{\lsf}$ of a PRCTL property $\lsf$ on a parametric Markov model.
The key idea is to use DIRECT to search the parameter domain $\parametersDomain$ by treating the evaluation of the property as a black-box function:
given a parameter point $\parameters \in \parametersDomain$, we obtain the value $f_{\lsf}(\parameters)$ by instantiating the parametric model with $\parameters$ and running a model checker on the resulting instantiated Markov model.
More precisely, let $\pmdp$ be a pMDP with parameter domain $\parametersDomain$.
Given a PRCTL state formula $\lsf = \lQ(\lpf)$, we define the \emph{property optimization problem} as:
\begin{equation}
\label{eq:propertyOptimization}
    \parameters^{*} = \arg\opt_{\parameters \in \parametersDomain} f_{\lsf}(\parameters),
\end{equation}
where $f_{\lsf}(\parameters)$ is the optimal value of $\lsf$ on the instantiated MDP $\evaluate[\parameters]{\pmdp}$ induced by the parameters $\parameters$.

To obtain $\parameters^{*}$ satisfying Eq.~\ref{eq:propertyOptimization}, we propose the following procedure, which integrates DIRECT with the PAC scenario framework:
\begin{enumerate}
    \item 
        \textbf{Normalization:} 
        normalize the parameter domain $\parametersDomain$ to the unit hypercube $\hypercube \subseteq \reals^{n}$.

    \item 
        \textbf{DIRECT iterations:} 
        apply the DIRECT algorithm on $\hypercube$:
        \begin{itemize}
        \item 
            in each iteration, identify the set of potentially optimal hyperrectangles according to Definition~\ref{def:potentiallyOptimal};
        \item 
            for each potentially optimal hyperrectangle, evaluate the chosen objective at the newly generated centroid points: either evaluate $f_{\lsf}$ by instantiating the pMDP and invoking a model checker (or a statistical model checker, as described in Section~\ref{sec:combiningPACwithSMC}), or evaluate the already synthesized PAC approximation $\ApproxFunOfProperty{f}$; 
            and
        \item 
            partition the potentially optimal hyperrectangles and update the partition $\hypercube_{\Numdivide}$.
        \end{itemize}
    
    \item 
        \textbf{Termination:} 
        stop when a convergence criterion is met, e.g., when the maximum distance $\max_{k} \distance_{k}$ falls below a prescribed tolerance $\delta > 0$, or when the number of function evaluations exceeds a budget $N_{\max}$.
    
    \item 
        \textbf{Output:} 
        return the best parameter point found during the search with respect to the chosen objective, i.e., $\hat{\parameters} = \arg\opt_{k} f_{\lsf}(c_{k})$ for exact or statistical evaluations of $f_{\lsf}$, or $\hat{\parameters} = \arg\opt_{k} \ApproxFunOfProperty{f}(c_{k})$ when DIRECT is run on the PAC approximation, together with the hyperrectangle $\hyperrectangle_{k^{*}}$ containing $\hat{\parameters}$.
\end{enumerate}

\begin{figure}[tb]
    \centering
    \resizebox{0.68\linewidth}{!}{
        \includegraphics{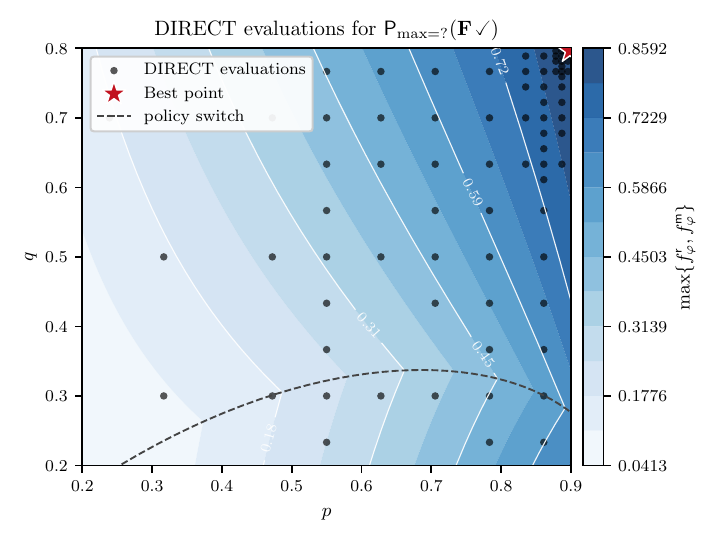}
    }
    \caption{DIRECT applied to the cloud-service pMDP from Section~\ref{sec:synthesisPACfunctions}. 
    The contour background shows the closed-form value of $\max\setnocond{f_{\lsf}^{\textsf{r}}, f_{\lsf}^{\textsf{m}}}$ only for visualization;
    each black point is a DIRECT evaluation of the instantiated model. 
    The dashed curve marks the policy switch, and the red star marks the best point found by DIRECT.}
    \label{fig:directCloudService}
\end{figure}
\begin{markedexample}
\label{ex:directExecution}
    To see how DIRECT works, consider again the cloud-service pMDP $\pmdp$ from Example~\ref{ex:pmdp} and the property $\lsf = \lP[\max]{\mathord{=}?}(\lF \success)$ over the parameter domain $\parametersDomain = \interval{0.2}{0.9} \times \interval{0.2}{0.8}$;
    Figure~\ref{fig:directCloudService} shows the points chosen by DIRECT that are then used to evaluate $f_{\lsf}$ on in a black-box manner through a call to the model checker:
    for each parameter point $(p,q)$ selected by DIRECT, we instantiate the pMDP with the corresponding values of $p$ and $q$ and call the model checker to compute the value of $f_{\lsf}(p,q)$.
    Besides the actual points chosen by DIRECT, shown as black dots, in Figure~\ref{fig:directCloudService} we also show the contour background that has been computed from the closed-form rational functions $f_{\lsf}^{\textsf{r}}$ and $f_{\lsf}^{\textsf{m}}$; 
    the background is provided only to make the landscape visible and it is not used by DIRECT.
    The sampled points illustrate the exploration--exploitation behavior of DIRECT: 
    the algorithm first covers the domain and then refines regions with high sampled values; 
    in this instance, the best sampled points move toward larger values of $p$ and $q$.
\end{markedexample}

The procedure given above naturally balances exploration and exploitation in the parameter space: 
DIRECT explores large hyperrectangles with potentially good function values (exploration), while it also refines hyperrectangles around the current best value (exploitation).
In the context of parametric Markov models, this is particularly beneficial because the function $f_{\lsf}$ for pMDPs is typically piece-wise defined (as different policies may be optimal in different subregions), and DIRECT's space-partitioning strategy is well suited to discover such subregions.
Moreover, when a PAC approximation function $\ApproxFunOfProperty{f}$ has already been synthesized for analysis or threshold checking, it can also be reused as the objective evaluated by DIRECT, corresponding to the PAC-approximation option in the second bullet of the DIRECT iterations above.
In this case, the DIRECT algorithm operates on the cheap-to-evaluate function $\ApproxFunOfProperty{f}$ instead of repeatedly invoking the exact evaluator for $f_{\lsf}$, and the PAC approximation error can be combined with a DIRECT partition-diameter bound under the additional conditions discussed below.

\subsection{Theoretical Guarantees}
\label{ssec:theoreticalGuarantees}

We now establish theoretical guarantees for the DIRECT-based optimization of parametric Markov models.
The bounds below are a posteriori bounds: they are meaningful once a Lipschitz constant or a certified upper bound on it is available, and they depend on the geometry of the hyperrectangles generated by the optimization run.
The main results relate the gap between the true optimal value $f_{\lsf}(\parameters^{*})$ and the best sampled value $f_{\lsf}(\hat{\parameters})$ found by the DIRECT algorithm, using properties of the Lipschitz constant and the partition geometry.

We first recall the notion of Lipschitz continuity in our setting.
\begin{definition}
\label{def:lipschitz}
    A function $f \colon \parametersDomain \to \reals$ is \emph{Lipschitz continuous} with constant $\lipcons \geq 0$ on $\parametersDomain$ if, for all $\parameters, \parameters' \in \parametersDomain$,
    \[
        \abs{f(\parameters) - f(\parameters')} \leq \lipcons \cdot \norm[2]{\parameters - \parameters'}.
    \]
\end{definition}
For reachability and expected-reward properties on pMCs, the induced value function is rational in the parameters whenever the corresponding instantiated linear systems are nonsingular.
Hence it is Lipschitz continuous on any compact parameter domain $\parametersDomain$ on which the rational representation has no singularity after simplification and, for reward properties, the expected rewards are finite.
For pMDPs, the optimal value is the pointwise maximum or minimum of the finitely many policy-induced value functions obtained from deterministic memoryless policies.
Therefore, if each policy-induced reachability or expected-reward function is well defined and Lipschitz continuous on $\parametersDomain$ with a common finite bound, then $f_{\lsf}$ is Lipschitz continuous on $\parametersDomain$ as well.

We now present our main result, which bounds the optimality gap after running the DIRECT algorithm.

\begin{theorem}
\label{thm:DIRECToptimalityGap}
    Let $\pmdp$ be a pMDP with parameter domain $\parametersDomain$ and $\lsf = \lQ(\lpf)$ be a PRCTL state formula.
    Suppose that $f_{\lsf} \colon \parametersDomain \to \posreals$ is Lipschitz continuous with constant $\lipcons$ on $\parametersDomain$.
    After the DIRECT algorithm has generated hyperrectangles $\hypercube_{\Numdivide} = \setnocond{\hyperrectangle_{1}, \dots, \hyperrectangle_{\Numdivide}}$ that cover $\parametersDomain$, with centroids $c_{1}, \dots, c_{\Numdivide}$, the best sampled value $\hat{\parameters} = \arg\opt_{k} f_{\lsf}(c_{k})$ satisfies
    \begin{equation}
    \label{eq:optimalityGap}
        \abs{f_{\lsf}(\parameters^{*}) - f_{\lsf}(\hat{\parameters})} \leq \lipcons \cdot \max_{1 \leq k \leq \Numdivide} \distance_{k},
    \end{equation}
    where $\parameters^{*} = \arg\opt_{\parameters \in \parametersDomain} f_{\lsf}(\parameters)$ are the true optimal parameters and $\distance_{k} = \frac{1}{2} \norm[2]{\lenside_{k}}$ is the Euclidean distance from the centroid of $\hyperrectangle_{k}$ to its vertices.
\end{theorem}
\begin{proof}
    In the proof below, we assume that $\opt = \max$, i.e., that the PRCTL state formula $\lsf = \lQ(\lpf)$ is actually the maximization formula $\lsf = \lQ[\max](\lpf)$; 
    the case with $\opt = \min$ is symmetric, with all inequalities reversed and the final absolute value unchanged.
    
    Let $\parameters^{*} \in \parametersDomain$ be a point at which the maximum $f_{\lsf}(\parameters^{*})$ of $f_{\lsf}$ is attained.
    Since the hyperrectangles $\setnocond{\hyperrectangle_{1}, \dots, \hyperrectangle_{\Numdivide}}$ cover $\parametersDomain$, there exists an index $k^{*} \in \setnocond{1, \dots, \Numdivide}$ such that $\parameters^{*} \in \hyperrectangle_{k^{*}}$.
    Since $c_{k^{*}}$ is the centroid of $\hyperrectangle_{k^{*}}$ and $\parameters^{*}$ lies in $\hyperrectangle_{k^{*}}$, we have $\norm[2]{\parameters^{*} - c_{k^{*}}} \leq \distance_{k^{*}}$.
    By the Lipschitz condition on $f_{\lsf}$, we get
    \[
        f_{\lsf}(\parameters^{*}) - f_{\lsf}(c_{k^{*}}) \leq \abs{f_{\lsf}(\parameters^{*}) - f_{\lsf}(c_{k^{*}})} \leq \lipcons \cdot \norm[2]{\parameters^{*} - c_{k^{*}}} \leq \lipcons \cdot \distance_{k^{*}}.
    \]
    Since $\hat{\parameters} = \arg\max_{k} f_{\lsf}(c_{k})$, we have $f_{\lsf}(c_{k^{*}}) \leq f_{\lsf}(\hat{\parameters})$, which implies
    \[
        f_{\lsf}(\parameters^{*}) - f_{\lsf}(\hat{\parameters}) \leq f_{\lsf}(\parameters^{*}) - f_{\lsf}(c_{k^{*}}) \leq \lipcons \cdot \distance_{k^{*}} \leq \lipcons \cdot \max_{1 \leq k \leq \Numdivide} \distance_{k}.
        \qedhere
    \]
\end{proof}

Theorem~\ref{thm:DIRECToptimalityGap} shows that the optimality gap is controlled by the product of the Lipschitz constant $\lipcons$ and the maximum distance $\max_{k} \distance_{k}$ of the hyperrectangle family considered in the bound.
This is an a posteriori certificate rather than an unconditional convergence statement: if some large hyperrectangles remain unrefined, the maximum distance may stay large even though the best sampled value is already good.
The following corollary therefore states the convergence consequence under the explicit refinement condition needed by the bound:
\begin{corollary}
\label{cor:DIRECTconvergence}
    Under the assumptions of Theorem~\ref{thm:DIRECToptimalityGap}, suppose that the generated hyperrectangles used for the certificate satisfy $\max_{k} \distance_{k} \to 0$ as $\Numdivide \to \infty$. 
    Consequently, $f_{\lsf}(\hat{\parameters}) \to f_{\lsf}(\parameters^{*})$ as $\Numdivide \to \infty$.
\end{corollary}
\begin{proof}
    The convergence $f_{\lsf}(\hat{\parameters}) \to f_{\lsf}(\parameters^{*})$ follows immediately by applying Theorem~\ref{thm:DIRECToptimalityGap} and taking the limit under the assumed condition $\max_{k} \distance_{k} \to 0$.
\end{proof}

We now show how to combine the optimality gap bound with the PAC approximation guarantees.
When the DIRECT algorithm uses the PAC approximation $\ApproxFunOfProperty{f}$ instead of the exact function $f_{\lsf}$, we need to account for the additional approximation error introduced by $\ApproxFunOfProperty{f}$.

\begin{theorem}
\label{thm:DIRECTPACoptimalityGap}
    Under the same assumptions as Theorem~\ref{thm:DIRECToptimalityGap}, let $\ApproxFunOfProperty{f}$ be a $\margin$-PAC approximation of $f_{\lsf}$ with $(\errorRate, \significanceLevel)$-guarantee.
    Assume that $\ApproxFunOfProperty{f}$ is Lipschitz continuous on $\parametersDomain$ with constant $\lipcons_{\ApproxFunOfProperty{f}}$.
    Let
    \[
        \hat{\parameters}_{\mathrm{PAC}}
        =
        \arg\opt_{k} \ApproxFunOfProperty{f}(c_{k})
    \]
    be the best choice of parameters found by DIRECT on $\ApproxFunOfProperty{f}$.
    If both $\parameters^{*}$ and $\hat{\parameters}_{\mathrm{PAC}}$ belong to the PAC-good set 
    $\parametersDomain_{\mathrm{good}}
    =
    \setcond{\parameters \in \parametersDomain}{\abs{f_{\lsf}(\parameters) - \ApproxFunOfProperty{f}(\parameters)} \leq \margin}$,
    then
    \begin{equation}
    \label{eq:PACoptimalityGap}
        \abs{f_{\lsf}(\parameters^{*}) - f_{\lsf}(\hat{\parameters}_{\mathrm{PAC}})} \leq \lipcons_{\ApproxFunOfProperty{f}} \cdot \max_{1 \leq k \leq \Numdivide} \distance_{k} + 2\margin.
    \end{equation}
    Moreover, with confidence $1-\significanceLevel$, the PAC guarantee ensures $P(\parametersDomain_{\mathrm{good}}) \geq 1-\errorRate$.
\end{theorem}
\begin{proof}
    In the proof below, we assume that $\opt = \max$, i.e., that the PRCTL state formula $\lsf = \lQ(\lpf)$ is actually the maximization formula $\lsf = \lQ[\max](\lpf)$;
    the case with $\opt = \min$ is symmetric, with the inequality direction reversed and the same absolute-value conclusion.

    Since both $\parameters^{*}$ and $\hat{\parameters}_{\mathrm{PAC}}$ lie in $\parametersDomain_{\mathrm{good}}$ by assumption, we have
        $f_{\lsf}(\parameters^{*}) \leq \ApproxFunOfProperty{f}(\parameters^{*}) + \margin$
    and
        $f_{\lsf}(\hat{\parameters}_{\mathrm{PAC}}) \geq \ApproxFunOfProperty{f}(\hat{\parameters}_{\mathrm{PAC}}) - \margin$.
    This implies that 
    \begin{align*}
        f_{\lsf}(\parameters^{*}) - f_{\lsf}(\hat{\parameters}_{\mathrm{PAC}})
        & \leq \left(\ApproxFunOfProperty{f}(\parameters^{*}) + \margin\right) - \left(\ApproxFunOfProperty{f}(\hat{\parameters}_{\mathrm{PAC}}) - \margin\right) \\
        & = \left(\ApproxFunOfProperty{f}(\parameters^{*}) - \ApproxFunOfProperty{f}(\hat{\parameters}_{\mathrm{PAC}})\right) + 2\margin.
    \end{align*}
    Now, let $k^{*}$ be such that $\parameters^{*} \in \hyperrectangle_{k^{*}}$.
    Since $\hat{\parameters}_{\mathrm{PAC}} = \arg\max_{k} \ApproxFunOfProperty{f}(c_{k})$, we have $\ApproxFunOfProperty{f}(c_{k^{*}}) \leq \ApproxFunOfProperty{f}(\hat{\parameters}_{\mathrm{PAC}})$, and since $\ApproxFunOfProperty{f}$ is Lipschitz continuous on $\parametersDomain$ with constant $\lipcons_{\ApproxFunOfProperty{f}}$, it follows that 
    \[
        \ApproxFunOfProperty{f}(\parameters^{*}) - \ApproxFunOfProperty{f}(\hat{\parameters}_{\mathrm{PAC}}) \leq \ApproxFunOfProperty{f}(\parameters^{*}) - \ApproxFunOfProperty{f}(c_{k^{*}}) \leq \lipcons_{\ApproxFunOfProperty{f}} \cdot \distance_{k^{*}}.
    \]
    Hence, in the maximization case we obtain the bound
    \[
        f_{\lsf}(\parameters^{*}) - f_{\lsf}(\hat{\parameters}_{\mathrm{PAC}}) \leq \lipcons_{\ApproxFunOfProperty{f}} \cdot \max_{k} \distance_{k} + 2\margin,
    \]
    which proves the claimed conditional bound.
    For $\opt = \min$, the same argument applied to $-f_{\lsf}$ and $-\ApproxFunOfProperty{f}$ gives
    \[
        f_{\lsf}(\hat{\parameters}_{\mathrm{PAC}}) - f_{\lsf}(\parameters^{*}) \leq \lipcons_{\ApproxFunOfProperty{f}} \cdot \max_{k} \distance_{k} + 2\margin,
    \]
    which is exactly the absolute-value statement in~\eqref{eq:PACoptimalityGap}.
\end{proof}
\begin{remark}
    The statement of Theorem~\ref{thm:DIRECTPACoptimalityGap} does not by itself imply that an arbitrary fixed optimizer $\parameters^{*}$ lies in $\parametersDomain_{\mathrm{good}}$:
    with confidence $1 - \significanceLevel$, the PAC guarantee gives $P(\parametersDomain_{\mathrm{good}}) \geq 1-\errorRate$.
    This means that a parameter point drawn independently according to the sampling measure lies in $\parametersDomain_{\mathrm{good}}$ with probability at least $1-\errorRate$, but it does not imply that a point chosen specifically to satisfy some constraint, such as the true optimizer $\parameters^{*}$ of $f_{\lsf}$, is in $\parametersDomain_{\mathrm{good}}$.
    In general, knowing that $\parameters^{*} \in \parametersDomain_{\mathrm{good}}$ requires additional information: 
    for example, a uniform error certificate for $\ApproxFunOfProperty{f}$, an analytical verification of the region containing $\parameters^{*}$, or an exact/validated evaluation showing that the approximation error at the relevant optimizer candidate is within $\margin$.
    Without such additional information, Theorem~\ref{thm:DIRECTPACoptimalityGap} should be read as a conditional certificate explaining how the PAC approximation error and the DIRECT discretization error combine once the relevant optimizer points are known to be in the PAC-good set.
\end{remark}

Theorem~\ref{thm:DIRECTPACoptimalityGap} provides a useful conditional decomposition of the total error into two terms, once the exact optimizer and the returned PAC optimizer are known to lie in the PAC-good set:
\begin{itemize}
\item 
    the \emph{discretization error} $\lipcons_{\ApproxFunOfProperty{f}} \cdot \max_{k} \distance_{k}$, which depends on how finely the DIRECT algorithm has partitioned the parameter space and on the Lipschitz constant of the approximation $\ApproxFunOfProperty{f}$ optimized by DIRECT.
    This term decreases as more iterations are performed; 
    and
\item 
    the \emph{approximation error} $2\margin$, which depends on how close the PAC approximation $\ApproxFunOfProperty{f}$ is to the true function $f_{\lsf}$.
    This term can be reduced by increasing the degree $d$ of the approximation polynomial, which provides more degrees of freedom for the fit.
\end{itemize}

\subsubsection*{Estimation of the Lipschitz constant.}
\label{ssec:estimationLipschitz}
Although the DIRECT algorithm does not require a Lipschitz constant to operate, a certified upper bound on the Lipschitz constant can enhance the algorithm's efficiency by eliminating hyperrectangles that cannot contain the optimal solution, thereby reducing the search time.
Consider a hyperrectangle $\hyperrectangle_{k}$ with centroid $c_{k}$. 
If the following condition holds:
\begin{equation}
\label{eq:removeRectangle}
    f(c_{k}) + \bar{\lipcons} \cdot \distance_{k} \leq f_{\opt},
\end{equation}
where $\bar{\lipcons}$ is a certified upper bound such that $\bar{\lipcons}\geq \lipcons$, then $\hyperrectangle_{k}$ cannot contain a point with function value larger than $f_{\opt}$, and it can be removed from the active set of hyperrectangles.
When $\lipcons$ is not known exactly, the sampled function values can still provide a diagnostic estimate of its magnitude.
A practical sample-based estimator over the sampled centroids is:
\begin{equation}
\label{eq:lipschitzEstimator}
    \hat{\lipcons} = \max_{j \neq k} \frac{\abs{f(c_{j}) - f(c_{k})}}{\norm[2]{c_{j} - c_{k}}}.
\end{equation}
Since $\hat{\lipcons}$ is computed from finitely many sampled pairs, it generally satisfies $\hat{\lipcons}\leq \lipcons$ and should be interpreted as an empirical lower estimate rather than a safe upper bound.
Therefore, replacing $\bar{\lipcons}$ in~\eqref{eq:removeRectangle} by $\hat{\lipcons}$ is not sound unless $\hat{\lipcons}$ is independently inflated or certified so that the resulting value is at least $\lipcons$.
The estimator in~\eqref{eq:lipschitzEstimator} is thus useful for diagnostics and adaptive parameter tuning, whereas certified pruning requires an analytical or verified numerical upper bound $\bar{\lipcons}$.
In practice, $\hat{\lipcons}$ is useful in three ways even when it is not by itself a sound upper bound.
First, it gives a scale estimate for the objective landscape, which helps interpret the magnitude of $\max_{k} \distance_{k}$ in the a posteriori gap bound.
Second, it can guide engineering choices such as termination tolerances, rescaling of parameters, and prioritization of regions whose sampled slopes are large.
Third, it can be used as the starting point for a conservative bound, for example by inflating it with a model-dependent safety factor or by checking the largest observed slopes with interval arithmetic; only after such an independent certification can the resulting value be used as $\bar{\lipcons}$ in the pruning rule~\eqref{eq:removeRectangle}.

For parametric Markov chains with reachability properties, the Lipschitz constant can sometimes be bounded analytically.
The following result provides such a bound for a specific class of models:
\begin{lemma}
\label{lem:lipschitzBoundReachability}
    Let $\pmc = (\mstates, \minit, \mtransitions, \mlabelling)$ be a pMC with parameters $\parameters \in \parametersDomain$, and let $\lsf = \lPmc{\mathord{=}?}(\lF T)$ be a reachability property for a target set $T \subseteq \mstates$.
    If the transition expressions $\mtransitions(s, s')$ are affine functions of the parameters and Assumption~\ref{asmt:wellPosedDomain} holds, then $f_{\lsf}$ is Lipschitz continuous on $\parametersDomain$ with Lipschitz constant $\lipcons$ bounded by
    \begin{equation}
    \label{eq:lipschitzBound}
        \lipcons \leq \abs{\mstates} \cdot \max_{s, s'} \norm[\infty]{\nabla_{\parameters} \mtransitions(s, s')(\parameters)},
    \end{equation}
    where the maximum is taken over all state pairs $(s, s')$ with $\mtransitions(s, s')(\parameters) > 0$ for some $\parameters \in \parametersDomain$, and $\nabla_{\parameters} \mtransitions(s, s')$ denotes the gradient of the transition probability function with respect to $\parameters$.
\end{lemma}
\begin{proof}
    Since the transition probabilities are affine in $\parameters$, their gradients $\nabla_{\parameters} \mtransitions(s, s')$ are constant vectors.
    The reachability probability $f_{\lsf}(\parameters) = g(\minit, T, \parameters)$ is obtained as the solution of the linear system where $g(s, T, \parameters) = 1$ for $s \in T$, $g(s, T, \parameters) = 0$ for absorbing non-target states, and $g(s, T, \parameters) = \sum_{s'} \mtransitions(s, s')(\parameters) \cdot g(s', T, \parameters)$ otherwise.
    By the chain rule and the fact that $g(s', T, \parameters) \in [0, 1]$ for all $s'$ and $\parameters$, the Jacobian of $f_{\lsf}$ can be bounded component-wise.
    Applying the mean value theorem on each component and summing over the at most $\abs{\mstates}$ terms in the path decomposition, we obtain the stated bound.
\end{proof}

The bound in~\eqref{eq:lipschitzBound} depends on $\max_{s,s'}\norm[\infty]{\nabla_{\parameters} \mtransitions(s,s')}$, which is a constant that must be computed from the concrete model.
When every transition probability is an affine function with unit coefficients, e.g., $\mtransitions(s,s') = p_i$ or $\mtransitions(s,s') = 1 - p_i$ for a single parameter $p_i$, each partial derivative is $\pm 1$ and the bound reduces to $\lipcons \leq \abs{\mstates}$.
However, even for affine transition probabilities, coefficients larger than $1$ are possible (e.g., $\mtransitions(s,s')(\parameters) = 3p - 0.5$ on a suitable domain), giving $\norm[\infty]{\nabla_{\parameters} \mtransitions(s,s')} = 3$.
Thus, the gradient norm must be evaluated explicitly for the model at hand rather than assumed to be at most~$1$.

\begin{remark}
    Lemma~\ref{lem:lipschitzBoundReachability} requires the transition probabilities to be \emph{affine} in~$\parameters$.
    In parametric Markov models where the transition probabilities are rational functions of the parameters (e.g., $\mtransitions(s,s')(\parameters) = p/(p+q)$), the gradient is no longer constant and can be unbounded near poles of the denominator.
    For such models, the sample-based estimator in~\eqref{eq:lipschitzEstimator} provides only an empirical lower estimate of the Lipschitz constant, while a sound upper bound must be derived analytically or by a verified numerical method on a case-by-case basis.
    Any certified upper bound, even if loose for a specific model, can be used for the pruning rule~\eqref{eq:removeRectangle} and for the optimality gap bound in Theorem~\ref{thm:DIRECToptimalityGap}.
\end{remark}

\section{Experimental Evaluation}
\label{sec:experiments}

We have implemented the PAC-based analysis approach proposed in Section~\ref{sec:synthesisPACfunctions} in our tool \pacpma\footnote{\url{https://github.com/iscas-tis/PacPMA/}}, the PAC-based Parametric Model Analyzer~\cite{DBLP:conf/atva/LiuTHXZ23}.
In this journal version, we focus our experimental evaluation on the new DIRECT-based optimization layer introduced in Section~\ref{sec:DIRECT}.
The benchmark suite contains 2997 instances: 
723 pMDP probability properties, 526 pMDP reward properties, 1040 pMC probability properties, and 708 pMC reward properties.
As benchmarks, we use the MCs and MDPs from the \prism benchmark suite~\cite{DBLP:conf/qest/KwiatkowskaNP12}, where we replace probabilistic choices by parameters, following the construction already used in~\cite{DBLP:conf/atva/LiuTHXZ23,DBLP:conf/birthday/ChiLT0J25}.
We performed the experiments on a desktop machine with an i5-12400F CPU and 16 GB of memory running Ubuntu Server 24.04.3 LTS;
we used \benchexec~\cite{DBLP:journals/sttt/BeyerLW19} to trace and constrain the executions, allowing each benchmark to use 15 GB of memory and imposing a time limit of 10 minutes of wall-clock time.

In our previous work about scenario approximation for parametric Markov models~\cite{DBLP:conf/atva/LiuTHXZ23,DBLP:conf/birthday/ChiLT0J25}, we already presented the experimental comparison of the PAC approximation with exact rational-function generation by \prism~\cite{DBLP:conf/cav/KwiatkowskaNP11}, the use of \storm~\cite{DBLP:journals/sttt/HenselJKQV22} as ordinary model-checking back-end, and the SMC-backed PAC experiments.
Those results justify the experimental scope adopted here: 
exact rational-function generation and model-checker back-end comparisons are important baselines for the approximation framework; repeating them here on a larger set of benchmarks would not provide new insights on the use of the scenario approach to approximate exact rational functions. 
Thus, we focus our experiments on practically evaluating the use of DIRECT on optimizing the parametric Markov models, i.e., our new contribution of this paper.
We leave a fresh \pacpma vs.\@ \prism{}/\storm comparison and SMC experiments to a future tool-oriented work.
The remainder of this section focuses on whether DIRECT, when integrated with \pacpma, gives optimized values that agree with the scenario-based optimizer on their common successful benchmarks, and on how the DIRECT variants compare in solving time.


In Section~\ref{sec:DIRECT}, we presented how the DIRECT algorithm can be applied to optimize the satisfaction value $f_{\lsf}$ over the parameter domain $\parametersDomain$ of a parametric Markov model.
We now evaluate the performance of DIRECT and its variants on the benchmark suite comprising both pMDPs and pMCs, and compare the results with the PAC scenario approach.
For each benchmark, DIRECT searches for the parameter configuration that optimizes $f_{\lsf}$ by iteratively partitioning the parameter space; 
we use the NLopt implementation of DIRECT~\cite{DIRECT} integrated into \pacpma, with a stopping tolerance of $10^{-8}$.
We test eight variants of the DIRECT algorithm:
the standard DIRECT, its local refinement DIRECT-L, their no-scaling variants (DIRECT-noscal and DIRECT-L-noscal), randomized local variants (DIRECT-L-rand and DIRECT-L-rand-noscal), and the original Fortran implementations (Orig-DIRECT and Orig-DIRECT-L).

\subsection{Overall comparison}

\begin{table}[t]
    \centering
    \caption{Overview of DIRECT optimization outcomes across all benchmarks}
    \label{tab:experimentsDIRECTOverall}
    \begin{tabular}{l r r r r r}
        \hline
        Algorithm & \#Computed & \#Timeout & \#Memout & \#Failure & Total \\
        \hline
        \pacpma & 1979 & 191 & 827 & 0 & 2997 \\
        DIRECT & 1380 & 790 & 827 & 0 & 2997 \\
        DIRECT-L & 1508 & 662 & 827 & 0 & 2997 \\
        DIRECT-noscal & 1509 & 661 & 827 & 0 & 2997 \\
        DIRECT-L-noscal & 1509 & 661 & 827 & 0 & 2997 \\
        DIRECT-L-rand & 1516 & 654 & 827 & 0 & 2997 \\
        DIRECT-L-rand-noscal & 1509 & 661 & 827 & 0 & 2997 \\
        Orig-DIRECT & 1499 & 665 & 820 & 13 & 2997 \\
        Orig-DIRECT-L & 1165 & 1003 & 821 & 8 & 2997 \\
        \hline
    \end{tabular}
\end{table}

Table~\ref{tab:experimentsDIRECTOverall} summarizes the outcomes across all 2997 benchmarks (723 pMDP probability, 526 pMDP reward, 1040 pMC probability, and 708 pMC reward).
The PAC scenario approach (\pacpma), run with degree $d = 0$, successfully computes a result for 1979 benchmarks, giving the largest number of completed runs in this comparison.
Among the DIRECT variants, DIRECT-L-rand achieves the highest success rate with 1516 benchmarks solved, followed by DIRECT-noscal, DIRECT-L-noscal, and DIRECT-L-rand-noscal with 1509 solved benchmarks each, DIRECT-L with 1508 solved benchmarks, Orig-DIRECT with 1499 solved benchmarks, and standard DIRECT with 1380 solved benchmarks.
The randomized local DIRECT variant therefore gives the best completion count on this benchmark suite, while the no-scaling variants remain close behind.
Orig-DIRECT is competitive, but it no longer outperforms the strongest NLopt-based variants on the updated data.
Disabling the internal scaling generally improves performance slightly over standard DIRECT, likely because the parameter domains in our benchmarks are already normalized.

\subsection{Optimality gap: Scenario versus DIRECT}

\begin{figure}[tb]
    \centering
    \resizebox{0.72\linewidth}{!}{
        \includegraphics{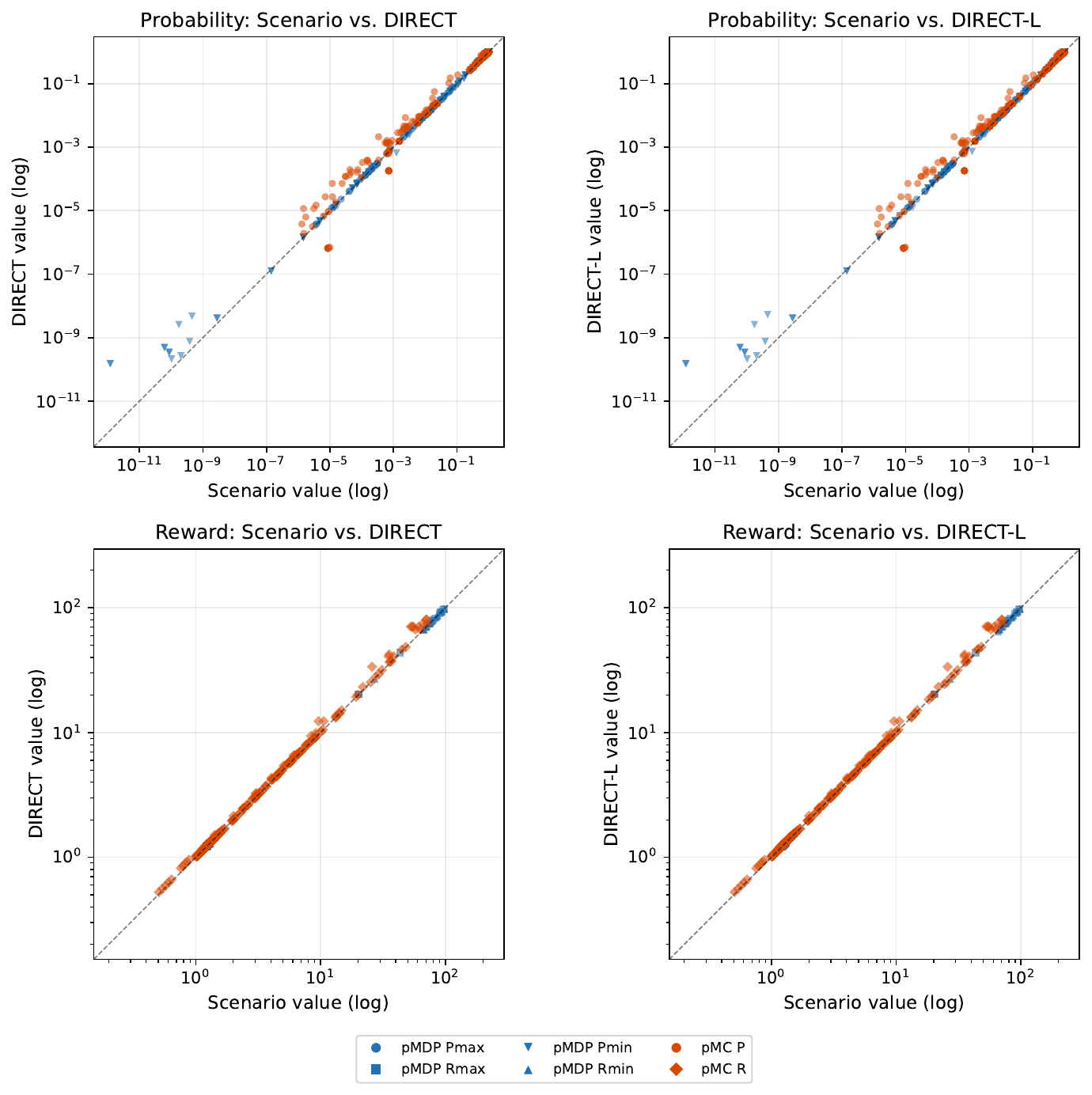}
    }
    \caption{Scenario values ($x$-axis) versus DIRECT and DIRECT-L values ($y$-axis) on common successful benchmarks, with logarithmic axes and zero-valued points omitted. Top plots: probability properties (995/1116 benchmarks for DIRECT/DIRECT-L, respectively); bottom plots: reward properties (344/351)}
    \label{fig:directScatter}
\end{figure}

Figure~\ref{fig:directScatter} compares the optimal values found by the PAC scenario approach and by DIRECT/DIRECT-L on the benchmarks where both methods succeed, split by property type.
The two columns compare Scenario against DIRECT and DIRECT-L, respectively, while the two rows separate probability and reward properties so that the two value scales are not visually mixed.
Each point represents a benchmark, with pMDP instances in blue and pMC instances in orange.
The legend separates six benchmark classes and is arranged with probability entries in the first row and reward entries in the second row; 
its columns align pMDP max, pMDP min, and pMC properties.
For the pMC cases, all properties in the current benchmark set are optimized by maximizing over the parameter domain; 
hence the pMC markers in Figure~\ref{fig:directScatter} should be read as parameter-optimization results, not as scheduler max/min choices.
For the pMDP benchmarks, the optimization direction used in DIRECT agrees with the one of the PRCTL property: $\min$ for $\lQ[\min]$ and $\max$ for $\lQ[\max]$.
Both axes use a logarithmic scale to better visualize the distribution across the wide range of positive values; 
benchmarks whose reported value is $0$ are not included in this log-scale comparison.

The points cluster tightly around the $y = x$ dashed line, showing that DIRECT finds values very close to those obtained via the PAC scenario approach on their common successful instances.
For instance, among the common successful pMDP probability-minimization instances, 63 have both the scenario value and the DIRECT/DIRECT-L value at most $10^{-3}$.
These points are visually compressed by the logarithmic scale, but they still lie close to the diagonal, indicating agreement on near-zero minima.
The clusters of blue triangles in the bottom-left corner of the upper plots correspond to pMDP probability-minimization benchmarks for which both optimizers find very small positive probabilities, below the threshold of $10^{-8}$ we set as termination criterion for DIRECT.
The comparison given in Figure~\ref{fig:directScatter} is empirical: 
it compares two computed procedures rather than the unknown true optimum, and therefore should not be read as a direct verification of Theorem~\ref{thm:DIRECToptimalityGap}.

Quantitatively, for probability properties, the mean absolute difference $\abs{f_{\mathrm{scenario}} - f_{\mathrm{DIRECT}}}$ is $0.0046$ across the 995 common benchmarks for DIRECT and $0.0042$ across the 1116 common benchmarks for DIRECT-L.
For reward properties, the corresponding mean absolute differences are $0.451$ across the 344 common benchmarks for DIRECT and $0.449$ across the 351 common benchmarks for DIRECT-L.
Since rewards are not a priori bounded as the probability properties, the absolute differences should indeed be interpreted together with the scale of the returned reward values:
the mean relative difference with respect to the scenario value is $1.93\%$ for DIRECT and $1.88\%$ for DIRECT-L, while the corresponding median relative differences are $0.20\%$ and $0.21\%$.
When we compare the optimality gap with the margin $\margin$ reported by the PAC scenario (which, by Theorem~\ref{thm:DIRECTPACoptimalityGap}, bounds the additional approximation error), we find that $86.5\%$ of the DIRECT values and $87.7\%$ of the DIRECT-L values fall within $\margin$ of the scenario value.
This empirical behavior is compatible with the form of the bound in Theorem~\ref{thm:DIRECTPACoptimalityGap}, but the experiment does not estimate the Lipschitz constant $\lipcons_{\ApproxFunOfProperty{f}}$ or the true optimum $f_{\lsf}(\parameters^{*})$.
Thus the result should be interpreted as evidence that the two optimization routes agree closely in practice, not as a numerical proof that the theoretical upper bound is tight.

We also compare which procedure returns the better objective value on the common successful instances, interpreting larger values as better for MAX properties and smaller values as better for MIN properties.
Standard DIRECT returns a strictly better value than the scenario optimizer on 1127 of its 1339 common successful benchmarks, while the scenario optimizer is better on 59 and the remaining 153 are tied up to numerical tolerance.
For DIRECT-L, the corresponding numbers are 1253 better values for DIRECT-L, 61 better values for the scenario optimizer, and 153 ties among 1467 common successful benchmarks.
These numbers indicate that DIRECT is not merely reproducing the scenario optimizer: 
on the benchmarks where both complete, it more often improves the returned objective value.
At the same time, the improvement is usually small relative to the PAC certificate, as only 74 DIRECT-favoring comparisons and 21 scenario-favoring comparisons exceed the reported margin $\margin$.
Thus DIRECT should be viewed as a complementary optimization route rather than as a replacement for the scenario approach: 
it solves fewer instances overall, but it can cheaply refine the objective value on many instances where it succeeds.

\subsection{Solving time comparison}

\begin{figure}[tb]
    \centering
    \resizebox{0.68\linewidth}{!}{
        \includegraphics{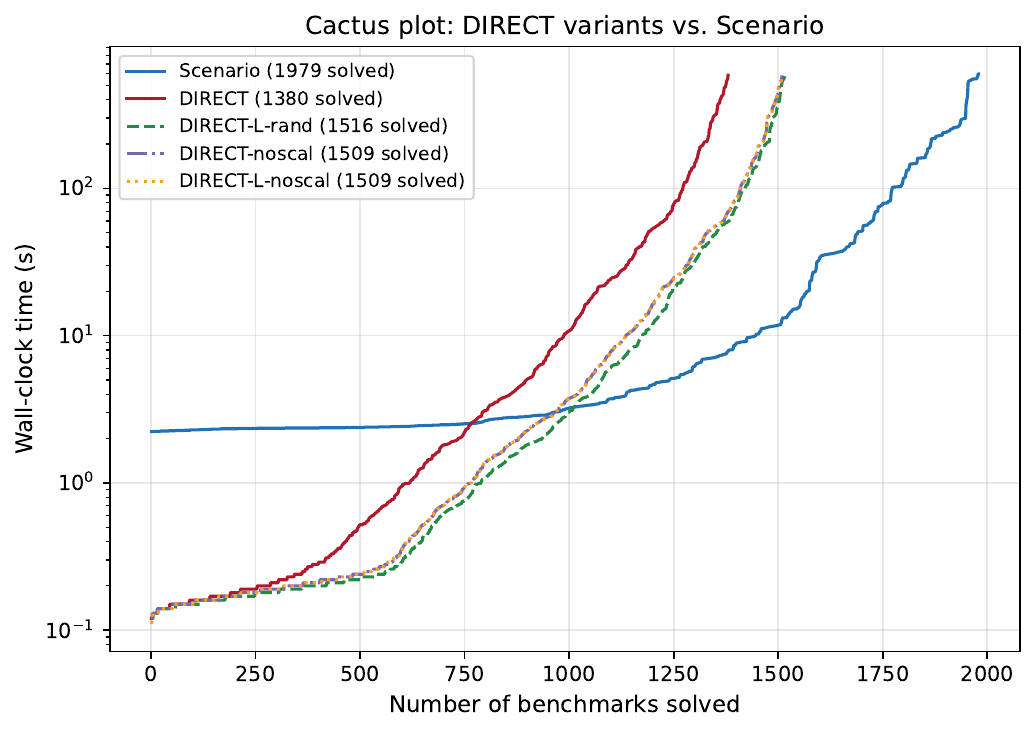}
    }
    \caption{Cactus plot comparing the solving time of \pacpma (Scenario), DIRECT, DIRECT-L-rand, DIRECT-noscal, and DIRECT-L-noscal on all 2997 benchmarks}
    \label{fig:directCactus}
\end{figure}

Figure~\ref{fig:directCactus} shows the cactus plot for the scenario approach, standard DIRECT, and the three best-performing DIRECT variants.
While \pacpma solves the most benchmarks (1979), the DIRECT variants are often faster on the benchmarks they can solve:
the majority of solved instances for the plotted DIRECT variants complete in under 10 seconds, whereas \pacpma shows a wider spread of solving times.
About two seconds of the total running time of \pacpma comes from the start-up time of \matlab when invoked to solve the LP problem after the sampled points have been evaluated.
Among the DIRECT variants, DIRECT-L-rand solves the most benchmarks among the plotted variants, with 1516 solved instances.
DIRECT-noscal and DIRECT-L-noscal solve the same 1509 benchmarks and their cactus curves almost overlap:
on these benchmarks they return identical objective values, and their median absolute running-time difference is only $0.01$ seconds.
Both variants therefore achieve essentially the same balance between the number of solved instances and the solving time.
The cactus plot also suggests that the DIRECT variants with local refinement or without scaling can be faster than standard DIRECT on the subset of benchmarks they solve.
The long tail of the DIRECT curves is expected: 
a difficult instance may require many hyperrectangle refinements before the stopping tolerance is reached, and each refinement triggers one or more objective evaluations.
Moreover, some benchmark families contain instantiated models whose individual evaluations are already expensive, so even a moderate number of DIRECT iterations can lead to long running times.
Together with the value-direction comparison above, this shows the intended role of DIRECT in the tool chain: 
the scenario approach remains the more robust default when the goal is to obtain a certified PAC approximation, whereas DIRECT is useful when the goal is to optimize the property value quickly once an evaluable objective is available.

\section{Discussion and Conclusion}
\label{sec:conclusion}

In this paper, we presented how the PAC-based approximation framework developed for parametric discrete-time Markov chains in~\cite{DBLP:conf/atva/LiuTHXZ23} can be extended to study several properties of parametric Markov decision processes.
We showed how it can be combined with statistical model checking, allowing us to analyze properties of black-box parametric Markov models and synthesize the output profiles with the desired PAC-guarantees.
Furthermore, we integrated the DIRECT algorithm for derivative-free global optimization over the parameter space and established conditional bounds on the optimality gap in terms of the Lipschitz constant and the partition diameter, both for exact and PAC-approximated objective functions.
The experimental results presented in this paper focus on the DIRECT integration: on 2997 benchmarks, the DIRECT variants solve fewer instances than the scenario optimizer, but they are often faster on the instances they solve and, on common successful benchmarks, more often return a better objective value.
The fact that more than $86\%$ of the DIRECT and DIRECT-L computed values remain within the PAC margin of the corresponding scenario value is consistent with the expectation that both optimization routes approximate the same underlying optimum; we report this as an empirical consistency check, showing that the observed DIRECT improvements are usually small relative to the PAC certificate rather than as an independent accuracy guarantee.

The framework developed in this paper has several limitations that are important for interpreting the results.
First, the PAC guarantees are distributional with respect to the chosen sampling measure over the parameter domain; 
they are not uniform guarantees over all parameter valuations unless an additional exhaustive or certified verification step is performed.
Second, the DIRECT optimality bounds require a Lipschitz constant, or a certified upper bound on it, and the PAC version of the bound applies conditionally to points that belong to the PAC-good set.
Third, the current experiments quantify agreement between scenario-based optimization and DIRECT on common successful instances, but they do not compute the unknown true optimum for every benchmark.
To compute such an optimum, an analytical solution of the exact symbolic rational-function would be required, incurring in the lengthy process of generating and analyzing it.

The framework developed in this paper can be extended in multiple directions.

\subsubsection*{Other approximation templates.}
When formulating the LP problem~\eqref{eq:PACLPpolynomialApproximation} for the PAC approximation, we used a (low-degree) polynomial template to get the approximation.
As such, a template allows us to easily compute properties of the resulting approximation function, such as finding its minimum and maximum values or plotting it.
We can easily change the formulation of the LP problem~\eqref{eq:PACLPpolynomialApproximation} to use non-polynomial templates, involving e.g.\@ trigonometric and transcendental operators:
in fact, these functions would not appear explicitly in the constraints of the LP problem~\eqref{eq:PACLPpolynomialApproximation}, since only their value on the sampled parameters is used.
Recall from the example just below the LP problem~\eqref{eq:PACLPpolynomialApproximation} that for the cloud-service pMDP shown in Figure~\ref{fig:pmdp}, the quadratic polynomial template, and the sampled point $(p=0.7, q=0.4)$, the corresponding constraint
\[
    -\margin \leq f_{\lsf}(\parameters_{i}) - \coefficients \cdot (1, \parameters_{i}, \dotsc, \parameters^{d}_{i})^{T} \leq \margin
\]
is instantiated as
\[
    -\margin \leq 0.397 - (c_{0} + 0.7 \cdot c_{11} + 0.4 \cdot c_{12} + 0.49 \cdot c_{21} + 0.28 \cdot c_{22} + 0.16 \cdot c_{23}) \leq \margin.
\]
If we would use as template $c_{0} \cdot q \cdot e^{p} + c_{1} \cdot \sin (p \cdot \pi)$, then the corresponding constraint would be instantiated as
\[
    -\margin \leq 0.397 - c_{0} \cdot 0.4 \cdot e^{0.7} + c_{1} \cdot \sin (0.7 \cdot \pi) \leq \margin,
\]
which is a linear constraint as intended.

\subsubsection*{Synthesizing multiple functions.}
As we have seen in Section~\ref{ssed:pMDRMspecificPRCTLanalysis}, when dealing with pMDPs the domain of the parameters can be partitioned into multiple subspaces, each one associated with a different optimizing policy $\policy$ and function $f_{\lsf}^{\policy}$.
The current formulation of the approximation problem~\eqref{eq:PACLPpolynomialApproximation} does not take into consideration this event, since it aims to synthesize a single (polynomial) function corresponding to the optimal values for satisfying $\lsf$ on the whole domain of the parameters where samples are taken.
Clearly, by synthesizing a single function, the computed margin $\margin$ might be larger than by synthesizing multiple functions, one for each subregion.

Extending our approach to work with multiple regions and functions is rather easy: 
we could just apply our framework and synthesize a PAC-guaranteed approximation function for each subregion independently.
The theoretical framework would remain essentially unchanged, except for the fact that the PAC-guarantees would be given for each subregion instead of for the whole domain of the parameters;
new results could be provided, to combine together the PAC-guarantees obtained for the subregions and infer appropriate PAC-guarantees for the whole domain of the parameters.
The implementation could take advantage of different strategies:
the parameters can be sampled, the corresponding instantiated MDPs are evaluated, and, according to the corresponding computed policy, the subregions are determined;
or
the subregions are first determined by e.g.\@ a bisection procedure and then the current implementation is applied to each subregion.
In both cases, the overall number of samples needs to be increased, so that in each subregion there are enough samples to ensure the statistical guarantees according to Theorem~\ref{thm:PACnumberOfSamples};
the parameters used to obtain the boundaries of the subregions should not contribute to reach $\nsamples$ (since they are not randomly sampled), but they can still be used in the problem~\eqref{eq:PACLPpolynomialApproximation} to provide additional constraints.

\subsubsection*{Artifact availability.}
The artifact accompanying this paper, including the instructions for running the tools, the benchmarks, and the scripts used to generate the plots, is available on Zenodo at \url{https://doi.org/10.5281/zenodo.21223130}.

\subsubsection*{Acknowledgements.}
We thank Jianting Yang (CNRS@CREATE, 1 Create Way, \#08-01 CREATE Tower, Singapore 138602) for improving a  proof of the paper. 
Work supported in part by 
the Beijing Natural Science Foundation Project No.\@ IS26039,
the CAS Project for Young Scientists in Basic Research under grant No.\@ YSBR-040,
NSFC under grant No.\@ 61836005,
the CAS Pioneer Hundred Talents Program,
and
the ISCAS New Cultivation Project ISCAS-PYFX-202201.
D.N.\@ Jansen is supported by Beijing Natural Science Foundation Project No.\@ IS25071.
\newline\protect\includegraphics[height=8pt]{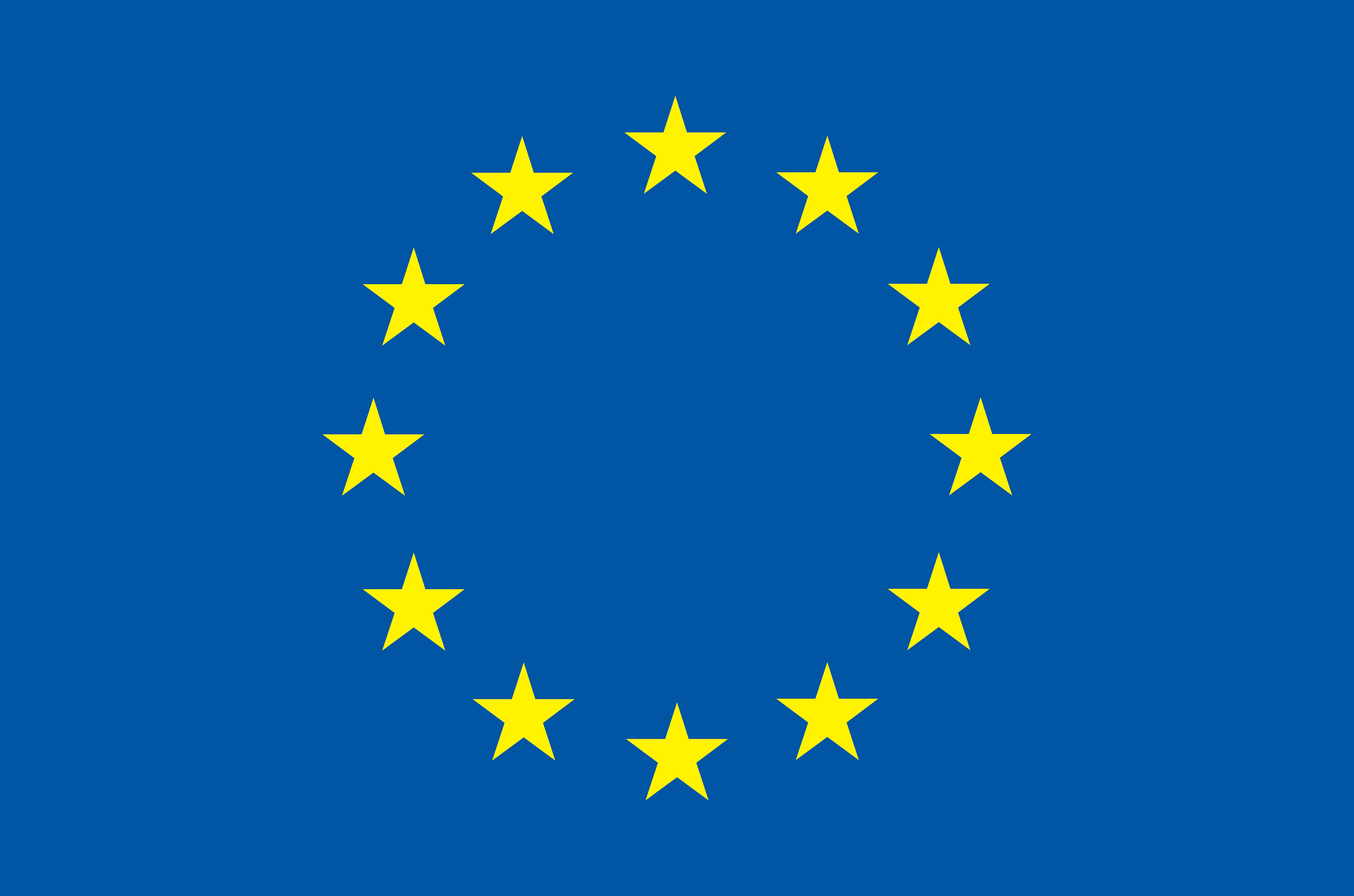} This project is part of the European Union's Horizon 2020 research and innovation programme under the Marie Sk\l{}odowska-Curie grant no.\@ 101008233.

\bibliographystyle{elsarticle-num}
\bibliography{biblio}

\end{document}